\documentclass[11pt]{article}
\usepackage[utf8]{inputenc}
\PassOptionsToPackage{table}{xcolor}
\usepackage{tikz,amsmath,amssymb,pgfplots,enumerate,authblk,amsfonts,amsthm}
\usepackage{mathtools}
\usepackage[ruled,linesnumbered]{algorithm2e}
\usepackage{subcaption}
\usepackage[authoryear,round,sort&compress]{natbib}
\definecolor{linkcol}{rgb}{0,0,0.55}
\definecolor{citecol}{rgb}{0,0.45,0}
\definecolor{urlcol}{rgb}{0.55,0,0}
\usepackage[colorlinks=true,linkcolor=linkcol,citecolor=citecol,urlcolor=urlcol]{hyperref}
\usepackage{graphicx}
\usepackage[margin=1in]{geometry}
\usepackage{libertinus}
\usepackage[libertine]{newtxmath}
\usepackage[T1]{fontenc}
\usepackage{booktabs}  
\usepackage{multirow}

\pgfplotsset{compat=1.18}
\usetikzlibrary{decorations.pathreplacing}
\usetikzlibrary{shapes.geometric, arrows, arrows.meta}
\usetikzlibrary{arrows,decorations.pathmorphing,backgrounds,calc,math}
\usetikzlibrary{positioning,fit}

\tikzstyle{map}=[->,semithick]
\tikzstyle{arc}=[bend left,->,semithick]
\tikzstyle{rinclusion}=[right hook->,semithick]
\tikzstyle{linclusion}=[left hook->,semithick]
\tikzstyle{step1} = [rectangle, rounded corners, minimum width=5cm, minimum height=0.7cm,
text centered, draw=black, fill=red!30]
\tikzstyle{step2} = [rectangle, rounded corners, minimum width=5cm, minimum height=0.7cm,
text centered, draw=black, fill=yellow!30]
\tikzstyle{step3} = [rectangle, rounded corners, minimum width=5cm, minimum height=0.7cm,
text centered, draw=black, fill=blue!30]
\tikzstyle{step4} = [rectangle, rounded corners, minimum width=7cm, minimum height=0.7cm,
text centered, draw=black, fill=orange!30]
\tikzstyle{step5} = [rectangle, rounded corners, minimum width=5cm, minimum height=0.7cm,
text centered, draw=black, fill=green!30]
\tikzstyle{arrow} = [thick,->,>=stealth]

\newcommand{\xmark}{\ensuremath{\times}}

\def\E{{\mathbb E}}

\newcommand{\ignore}[1]{ }

\newcommand{\D}[1]{\mathcal{D}_{{#1}}}
\newcommand{\db}{d_{\mathcal{B}}}

\theoremstyle{plain}
\newtheorem{theorem}{Theorem}[section]
\newtheorem{remark}{Remark}[section]
\newtheorem{proposition}{Proposition}[section]
\newtheorem{corollary}{Corollary}[section]
\newtheorem{definition}{Definition}[section]
\newtheorem{lemma}{Lemma}[section]

\def\RR{{\mathbb R}}

\def\GG{{\mathbb G}}

\def\m{\sigma}

\numberwithin{equation}{section}

\def\eps{{\varepsilon}}

\title{A Closed-Form Persistence-Landmark Pipeline for Certified Point-Cloud and Graph Classification}

\author[1]{Sushovan Majhi}
\affil[1]{\small Data Science, George Washington University, USA (s.majhi@gwu.edu)}
\author[2]{Atish Mitra}
\affil[2]{\small Department of Mathematical Sciences, Montana Technological University, USA (amitra@mtech.edu)}
\author[3]{\v{Z}iga Virk}
\affil[3]{\small Faculty of Computer and Information Science, University of Ljubljana, Institute IMFM, Slovenia (ziga.virk@fri.uni-lj.si)}
\author[4]{Pramita Bagchi}
\affil[4]{\small Biostatistics and Bioinformatics, George Washington University, USA (pramita.bagchi@gwu.edu)}

\date{}

\begin{document}

\maketitle

\begin{abstract}
We introduce \textbf{PLACE} (\textbf{P}ersistence-\textbf{L}andmark
\textbf{A}nalytic \textbf{C}lassification \textbf{E}ngine),
a closed-form pipeline for classifying point clouds and graphs
through their persistent-homology signatures. Three quantitative
guarantees---a margin-based excess-risk rate, a closed-form
descriptor-selection rule, and a per-prediction certificate---are
derived from training labels alone, with no learned weights or
held-out calibration. The embedding sums Mitra--Virk single-point
coordinate functions over a sparse landmark grid; the closed-form
weight rule $w_k^2 \propto (d_{k+1}^2 - d_k^2)/R_k^2$ is the
maximizer of the distortion slope $\lambda(\nu)$ in Mitra--Virk's
affine certificate
$\|\Phi(A) - \Phi(B)\|_{\ell^2} \geq \lambda(\nu)\,(\db(A,B) - R_1)$
under $\nu$-coherence. The guarantees take the following form.
\emph{(i)}~An $O(kR/(\Delta\sqrt{m_{\min}}))$ margin bound, driven by
class-mean separation $\Delta$ and embedding radius $R$, matched in
the sample-starved regime $m \lesssim R/\Delta$ by a Le~Cam
minimax lower bound.
\emph{(ii)}~A closed-form descriptor-selection rule.  The Mahalanobis
margin under Ledoit--Wolf-shrunk covariance is the empirically
strongest closed-form ranker on a heterogeneous $64$-descriptor
chemical-graph pool (mean Spearman $\rho = +0.56$ across $11$
benchmarks, positive on $10$ of $11$); the isotropic surrogate
$\Delta/\sqrt\ell$ admits a closed-form selection-consistency rate
on the homogeneous ($14$--$15$ descriptor) protein/social pools.
\emph{(iii)}~A training-time-decided certificate, with no
per-prediction overhead, in three concrete radii (non-asymptotic
Pinelis, asymptotic Gaussian plug-in, and a variance-aware
Pinelis--Bernstein form), with firing-rate diagnostics across
the $12$ benchmarks.
Empirically, PLACE is the strongest diagram-based method on
Orbit5k and matches the strongest topology-based baseline within
statistical noise on MUTAG and COX2.  The remaining gaps fall
into two diagnosable regimes: descriptor blindness on
NCI1/NCI109, and pool-coverage limits elsewhere.  The
variance-aware Pinelis--Bernstein radius fires the certificate
on $8$ of the $12$ benchmarks (full $\geq 99\%$ firing on the
six chemical datasets where the empirical class signal-to-noise
$\hat\Delta_c / \|\hat\Sigma_c\|_{\mathrm{op}}^{1/2} \gtrsim 1$);
the four holdouts (COX2, PTC, IMDB-M, Orbit5k) sit in a
structurally distinct regime where any sample-mean
concentration argument fails at the available sample sizes.
The certificate's mechanism validates empirically: on MUTAG
the empirical and population nearest-centroid rules agree on
every one of $940$ held-out test predictions
(Clopper--Pearson $95\%$ lower bound on test-time agreement
$\geq 0.984$), consistent with the Pinelis--Bernstein
certificate firing.
\end{abstract}

\noindent\textbf{Keywords:}\quad
persistent homology;
topological data analysis;
landmark embedding;
kernel methods;
classification certificates;
minimax lower bound;
descriptor selection;
closed-form learning.

\section{Introduction}

Persistent homology produces a canonical topological signature of
structured data---graphs, point clouds, shapes---called the
\emph{persistence diagram}: a finite multiset of points in the
half-plane above the diagonal, augmented by a formal diagonal
point $\ast$ (Figure~\ref{fig:intro_pd}).
Stability under perturbation is well-understood
\citep{Cohen-Steiner2007,ChazalCohenSteiner2009,ChazalDeSilva2016},
but the varying cardinality and non-Hilbertian geometry of
diagrams make them incompatible with standard machine learning.
Existing vectorizations---persistence images~\citep{Adams2017},
landscapes~\citep{Bubenik15},
kernels~\citep{Kusano2016,Carriere2017}, and learned
weights~\citep{yusu_metric_learning}---all offer Lipschitz
\emph{upper} bounds on embedding distortion but no \emph{lower}
bound with explicit constants, so there is no guarantee that
bottleneck-separated diagrams remain separated after vectorization.
Each method further carries hyperparameters---kernel bandwidth,
image resolution, landscape level count, learned weight
function---whose selection requires held-out data, so any
downstream accuracy claim inherits the dependence on a validation
split.
Despite a decade of work, there is no way to inspect a trained
persistence-diagram classifier and certify, before seeing test
data, whether its predictions will be correct.

A second gap concerns how the input $X$ is turned into a
persistence diagram in the first place.
For graphs, this means choosing a \emph{descriptor function}
$f : X \to \RR$---degree, centrality measures, curvature, or
heat-kernel signatures of various scales~\citep{Ollivier2009}---%
whose sublevel sets define the filtration.
For point clouds, the analogous choice is among filtration
constructions parameterized by a radius (e.g., Vietoris--Rips
or alpha complex).
Different choices produce different diagrams and different
downstream accuracy, with swings of $5$--$15$ percentage points
across our $12$ benchmarks.
\citet{yusu_metric_learning} highlight the effective use of
multiple descriptor functions as a key open problem;
in practice, the choice is made by trial-and-error against
held-out labels, embedding a label-consuming hyperparameter
selection into every reported accuracy number.
There is no closed-form rule that ranks descriptors directly
from training data.

This paper introduces \textbf{PLACE}, a persistence-based
classifier with provable accuracy guarantees and per-prediction
correctness certificates.
Our starting point is the \emph{persistence landmark embedding}
of \citet{Mitra2024}: the only \emph{explicit} coarse embedding
of $\D{n}$ into a finite-dimensional Euclidean space with known
distortion constants.  The earlier work \citet{Mitra2021}
established existence via an asymptotic-dimension argument
($\mathrm{asdim}(\D{n}, \db) = 2n$, linear in $n$); the 2024
construction makes that existence concrete, placing $M$
landmarks on a lattice in the birth-death plane at $N$
geometrically spaced scales, assigning each landmark a compactly supported
hat function as a coordinate, and assembling an $n$-fold
composition for $n$-point diagrams.  The composition's
$n$-fold structure is what makes the lower distortion bound
$\rho_-$ \emph{unconditional} on $\{\db \geq R_1\}$: every
cross-pair gets a dedicated coordinate, so cancellation between
hat-function contributions cannot occur.
The explicit composition has dimension $M^n$ per scale, growing
exponentially in the diagram cardinality $n$.
We replace it with a summation that evaluates the single-point
coordinate at every point of the diagram and adds, dropping the
embedding dimension to $\ell = O(MN)$ overall---linear in the
grid size and the number of scales.  The summation pooling is
the source of computational tractability and of a structural
trade-off: by collapsing the $M^n$ cross-pair coordinates of
the $n$-fold form into $M$ summed coordinates, summation enables
a cancellation construction
(Remark~\ref{rem:noninterference_why}) that the $n$-fold form
rules out automatically.  The lower bound is therefore
\emph{conditional} on $\nu$-coherence---a matching-free
per-scale block-norm floor on the embedding difference
(Proposition~\ref{prop:n_point_lower}; holding on $\geq 99.7\%$
of pairs in the Section~\ref{sec:experiments} audit).  Linear
complexity is bought at the price of a conditional lower bound,
with the conditional close to universal in practice.
Summation pooling additionally keeps the embedding
\emph{linear} in the empirical diagram measure---a property that
max-pooled or order-statistic alternatives lack.
The distortion constant depends on the per-scale terms
$w_k^2 R_k^2$, whose optimization
(Section~\ref{sec:filt_select_theory}) yields a closed-form
weight rule
(equation~\eqref{eq:closed_form_weights}) in one step.
For downstream classification, the per-pair distortion is
replaced by a data-dependent class-mean separation $\Delta$,
which drives the theory in Section~\ref{sec:metric}.

Three results establish the theory.
First, with $R$ denoting the embedding radius, the excess risk
of a linear SVM on the embedded features is
$O\bigl(kR/(\Delta\sqrt{m_{\min}})\bigr)$ in the per-class training-set size
(Theorem~\ref{thm:fisher_bound}; the factor $k-1$ comes from the
one-vs-one reduction used in the proof); a Le~Cam two-point
argument shows that in the sample-starved regime
$m \lesssim R/\Delta$ no classifier achieves better than
constant excess risk
(Theorem~\ref{thm:lower_bound}); the bound depends on $\Delta$,
not on the worst-case bottleneck separation.
Second, descriptor selection is itself closed-form: the
Mahalanobis margin\footnote{Prasanta Chandra Mahalanobis
(1893--1972) was an Indian statistician who founded the Indian
Statistical Institute in 1931 and pioneered the use of large-scale
sample surveys in national planning; his 1936 paper
\emph{On the Generalised Distance in Statistics} introduced the
covariance-aware distance that bears his name, the LDA
Bayes-margin form of the Fisher discriminant ratio, which
Section~\ref{sec:mah_selector} of this paper adopts as the
descriptor-selection rule on heterogeneous pools. We dedicate
this work to his memory.}
$\hat\rho_\mathrm{Mah}$ between empirical
class means under a Ledoit--Wolf-shrunk pooled covariance---the
LDA Bayes-margin form of the Fisher discriminant ratio---is
computable from training labels without any held-out validation
or learned weights, and rank-correlates with linear-SVM accuracy
across $11$ benchmarks with mean Spearman $\rho = +0.56$
(range $-0.24$ to $+0.89$, positive on $10$ of $11$;
Section~\ref{sec:filtration_selection}).
Its simpler isotropic surrogate $\Delta/\sqrt\ell$ is
ranking-consistent under a separation gap
(Proposition~\ref{prop:selection_consistency},
Corollary~\ref{cor:data_driven_rate}) and provides an
upper-bound interpretation tied directly to
Theorem~\ref{thm:fisher_bound}, with Mahalanobis as the
appropriate selector when structural homogeneity
fails---the regime in which heterogeneous descriptor pools
live (Remark~\ref{rem:fisher_ratio}).
Third, a scalar check $r_m < \tfrac{1}{2}\Delta$ certifies
agreement between the empirical and population nearest-centroid
classifiers on every input, with probability $\geq 1-\alpha$
(Theorem~\ref{thm:confidence_containment}).
The bound provides three complementary radii: a non-asymptotic
Pinelis form (dimension-free but $L^{2}$-bounded), an asymptotic
Gaussian plug-in form (dimension penalty $\sqrt{\chi^2_\ell}$),
and a non-asymptotic variance-aware Pinelis--Bernstein form that
combines the dimension-freeness of the first with the
operator-norm refinement of the second.  The Pinelis--Bernstein
radius fires the certificate on $8$ of the $12$ benchmarks
(Table~\ref{tab:cert_firing}); the four holdouts have
$\|\hat\Sigma_c\|_{\mathrm{op}} \gtrsim \hat\Delta_c^{2}/4$,
i.e.\ they are not nearest-centroid-separable at the population
level, and we read this as a structural diagnostic rather than
a slack-of-bound issue.
The check is performed once from training statistics, has no
per-prediction overhead, requires no calibration
split---unlike conformal methods~\citep{VovkEtAl2005}---and
certifies individual point predictions rather than sets.
Figure~\ref{fig:pipeline} gives the end-to-end view: raw graph
or point cloud enters on the left, a certified label exits on
the right, and every ingredient along the way is fixed
analytically from the compact support size $L$ and the training
labels.
\begin{figure}[t]
\centering
\begin{tikzpicture}[scale=1.15, font=\footnotesize,
  pt/.style={circle,fill=blue!70!black,inner sep=0pt,minimum size=3.2pt}]

\begin{scope}[xshift=0cm]
  \node[anchor=south,font=\footnotesize\bfseries] at (1.5,3.8) {(a) point cloud};
  \foreach \x/\y in {
    2.70/1.90, 2.52/2.94, 1.86/3.06, 1.30/2.96,
    0.42/1.93, 0.65/0.94, 1.34/0.76, 2.08/0.73,
    2.68/1.08, 1.55/3.10, 0.58/2.48, 1.50/0.68}{
    \node[pt] at (\x,\y) {};
  }
  \node[gray,anchor=north] at (1.5,0.62) {\scriptsize $n$ points (noisy circle)};
\end{scope}

\begin{scope}[xshift=3.5cm]
  \node[anchor=south,font=\footnotesize\bfseries] at (1.7,3.8) {(b) filtration};
  \node[gray,anchor=south] at (0.55,3.1) {\scriptsize $r{=}r_1$};
  \draw[gray!70,thin] (1.04,1.90)--(1.03,1.64);
  \foreach \x/\y in {
    1.04/1.90, 0.97/2.29, 0.72/2.33, 0.50/2.28,
    0.16/1.93, 0.25/1.54, 0.51/1.47, 0.80/1.45,
    1.03/1.64, 0.60/2.34, 0.22/2.09, 0.58/1.43}{
    \node[pt] at (\x,\y) {};
  }
  \begin{scope}[xshift=1.15cm]
    \node[gray,anchor=south] at (0.55,3.1) {\scriptsize $r{=}r_2$};
    \draw[blue!60,thin]
      (1.04,1.90)--(0.97,2.29)--(0.72,2.33)--(0.60,2.34)--(0.50,2.28)--
      (0.22,2.09)--(0.16,1.93)--(0.25,1.54)--(0.51,1.47)--(0.58,1.43)--
      (0.80,1.45)--(1.03,1.64)--(1.04,1.90);
    \foreach \x/\y in {
      1.04/1.90, 0.97/2.29, 0.72/2.33, 0.60/2.34, 0.50/2.28,
      0.22/2.09, 0.16/1.93, 0.25/1.54, 0.51/1.47, 0.58/1.43,
      0.80/1.45, 1.03/1.64}{
      \node[pt] at (\x,\y) {};
    }
    \node[blue!80!black,font=\scriptsize,anchor=south] at (0.55,2.4) {loop born};
  \end{scope}
  \begin{scope}[xshift=2.3cm]
    \node[gray,anchor=south] at (0.55,3.1) {\scriptsize $r{=}r_3$};
    \fill[red!20,opacity=0.7]
      (1.04,1.90)--(0.97,2.29)--(0.72,2.33)--(0.60,2.34)--(0.50,2.28)--
      (0.22,2.09)--(0.16,1.93)--(0.25,1.54)--(0.51,1.47)--(0.58,1.43)--
      (0.80,1.45)--(1.03,1.64)--cycle;
    \draw[blue!30,thin]
      (1.04,1.90)--(0.97,2.29)--(0.72,2.33)--(0.60,2.34)--(0.50,2.28)--
      (0.22,2.09)--(0.16,1.93)--(0.25,1.54)--(0.51,1.47)--(0.58,1.43)--
      (0.80,1.45)--(1.03,1.64)--(1.04,1.90);
    \foreach \x/\y in {
      1.04/1.90, 0.97/2.29, 0.72/2.33, 0.60/2.34, 0.50/2.28,
      0.22/2.09, 0.16/1.93, 0.25/1.54, 0.51/1.47, 0.58/1.43,
      0.80/1.45, 1.03/1.64}{
      \node[pt] at (\x,\y) {};
    }
    \node[red!70!black,font=\scriptsize,anchor=south] at (0.55,2.4) {loop dies};
  \end{scope}
  \node[gray,anchor=north] at (1.7,0.62) {\scriptsize $r_1 < r_2 < r_3$};
\end{scope}

\begin{scope}[xshift=7.4cm]
  \node[anchor=south,font=\footnotesize\bfseries] at (1.4,3.8) {(c) barcode};
  \draw[-{Stealth}] (0,0.8) -- (2.9,0.8) node[right] {$r$};
  \draw[blue!50,line width=2pt] (0.05,3.35)--(0.22,3.35);
  \draw[blue!50,line width=2pt] (0.05,3.20)--(0.28,3.20);
  \draw[blue!50,line width=2pt] (0.05,3.05)--(0.18,3.05);
  \draw[blue!50,line width=2pt] (0.05,2.90)--(0.31,2.90);
  \draw[blue!50,line width=2pt] (0.05,2.75)--(0.24,2.75);
  \draw[blue!50,line width=2pt] (0.05,2.60)--(0.26,2.60);
  \draw[blue!50,line width=2pt] (0.05,2.45)--(0.20,2.45);
  \draw[blue!50,line width=2pt] (0.05,2.30)--(0.29,2.30);
  \draw[blue!50,line width=2pt] (0.05,2.15)--(0.25,2.15);
  \draw[blue!50,line width=2pt] (0.05,2.00)--(0.27,2.00);
  \draw[blue!50,line width=2pt] (0.05,1.85)--(0.23,1.85);
  \draw[blue!80!black,line width=2.5pt] (0.05,1.65)--(2.7,1.65);
  \draw[red!70!black,line width=2.5pt] (0.9,1.3)--(2.5,1.3)
    node[anchor=west,font=\scriptsize] {$H_1$};
  \node[blue!70!black,anchor=east,font=\scriptsize] at (0.03,2.6) {$H_0$};
  \node[gray,anchor=north] at (1.4,0.62) {\scriptsize bar length $=$ lifetime};
\end{scope}

\begin{scope}[xshift=11.0cm]
  \node[anchor=south,font=\footnotesize\bfseries] at (2,3.8) {(d) persistence diagram};
  \draw[-{Stealth}] (0,0.8)--(2.95,0.8) node[right] {$b$};
  \draw[-{Stealth}] (0,0.8)--(0,3.65)   node[above] {$d$};
  \draw[gray,dashed] (0,0.8)--(2.7,3.5);
  \node[gray,font=\scriptsize,anchor=west] at (2.35,3.22) {$d{=}b$};
  \draw[red!40,dashed,thin] (0.9,0.8)--(0.9,2.9);
  \draw[red!40,dashed,thin] (0,2.9)--(0.9,2.9);
  \node[red!60,font=\scriptsize,anchor=south east] at (0.9,0.78) {$b$};
  \node[red!60,font=\scriptsize,anchor=east]  at (0,2.9)    {$d$};
  \draw[<->,>=stealth,red!60,thick] (0.9,1.7)--(0.9,2.9) node[midway,anchor=west,red!70!black,font=\scriptsize] {$\tau$};
  \fill[blue!50!black] (0.12,1.28) circle (1.8pt);
  \fill[blue!50!black] (0.12,1.22) circle (1.8pt);
  \fill[blue!50!black] (0.12,1.16) circle (1.8pt);
  \fill[blue!50!black] (0.12,1.10) circle (1.8pt);
  \fill[blue!50!black] (0.12,1.04) circle (1.8pt);
  \fill[blue!50!black] (0.15,1.26) circle (1.8pt);
  \fill[blue!50!black] (0.15,1.20) circle (1.8pt);
  \fill[blue!50!black] (0.15,1.14) circle (1.8pt);
  \fill[blue!50!black] (0.15,1.08) circle (1.8pt);
  \fill[blue!50!black] (0.18,1.24) circle (1.8pt);
  \fill[blue!50!black] (0.18,1.18) circle (1.8pt);
  \fill[blue!80!black] (0.12,3.40) circle (2.2pt);
  \fill[red!70!black]  (0.90,2.90) circle (3pt);
  \node[gray,anchor=north] at (1.4,0.62) {\scriptsize $d{-}b =$ persistence};
\end{scope}

\end{tikzpicture}
\caption{%
  \textbf{From point cloud to persistence diagram.}
  \textbf{(a)}~Noisy sample from a circle.
  \textbf{(b)}~Vietoris--Rips filtration at radii $r_1 < r_2 < r_3$:
  the $1$-cycle is born at $r_2$ and dies at $r_3$.
  \textbf{(c)}~Barcode; bar length equals feature lifetime.
  \textbf{(d)}~Persistence diagram; each feature becomes a point
  $(b,d)$, with distance $\tau = d - b$ to the diagonal measuring
  topological significance.
  We overlaid the $0$-dim (\textcolor{blue}{blue}) and $1$-dim (\textcolor{red}{red}) diagrams together.
}
\label{fig:intro_pd}
\end{figure}
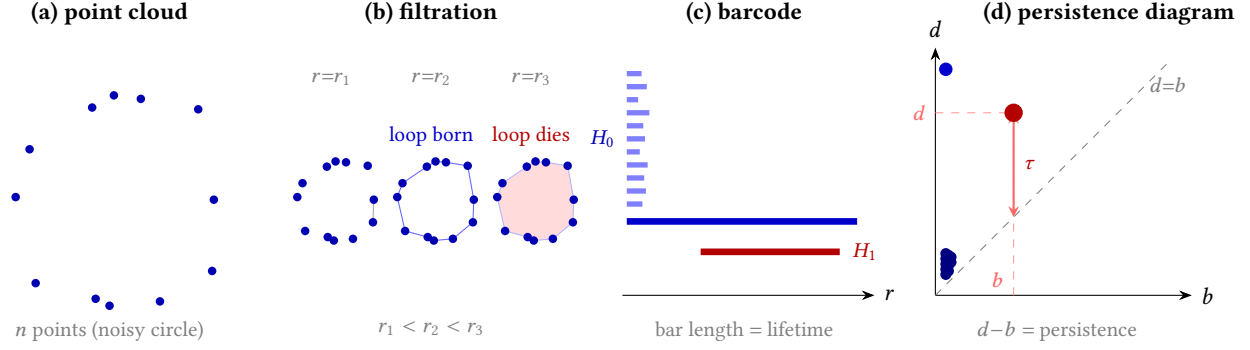

The shift from worst-case distortion to class-mean separation
also resolves a puzzle in the empirical literature: descriptors
with vanishing pairwise separation (e.g., degree) can match the
accuracy of descriptors with large separation (e.g., Ricci
curvature), because $\Delta$ captures the aggregate
distributional signal the classifier actually uses.
Descriptor choice---not mass tuning or scale optimization---is
the primary accuracy driver, and two closed-form selectors
identify the right one without held-out data: the Mahalanobis
margin $\hat\rho_{\mathrm{Mah}}$ on heterogeneous pools (the
LDA Bayes-margin form of the Fisher discriminant ratio), and
its isotropic surrogate $\hat\Delta/\sqrt{\ell}$ on homogeneous
pools, computable directly from the raw input without a
separate diagram-level analysis.
This addresses an open question raised by
\citet{yusu_metric_learning}, who identified the effective
use of multiple descriptor functions as a key challenge;
the Mahalanobis-plus-surrogate pair provides a principled,
closed-form answer to the descriptor-selection part of that
challenge (Section \ref{sec:filt_select_theory}).
Across 12 benchmarks (Section~\ref{sec:experiments}), PLACE is
the strongest diagram-based method on Orbit5k, matches the
strongest topology-based baseline within statistical noise on
MUTAG and COX2, and exhibits quantitative gaps on the remaining
graph datasets, all without any tuning.
The method's principled failures---e.g., on NCI1/NCI109, where
classes are distinguished by discrete node labels that our
continuous descriptors cannot access---are diagnosed by the same
statistic, suggesting the descriptor pool, rather than the
embedding machinery, is the bottleneck.

\begin{figure}[t]
\centering
\begin{tikzpicture}[
  font=\small,
  stage/.style={draw=gray!45, rounded corners=3pt, fill=gray!3,
                line width=0.5pt,
                minimum width=2.9cm, minimum height=3.3cm,
                inner sep=4pt},
  slabel/.style={font=\small\bfseries, text=black!80},
  sbelow/.style={font=\footnotesize, text=gray!55!black, align=center},
  cfband/.style={draw=blue!40, fill=blue!4, rounded corners=4pt,
                 inner sep=8pt, font=\footnotesize, align=center,
                 line width=0.6pt},
  flow/.style={-{Stealth[length=6pt,width=5pt]}, line width=1.1pt, gray!55},
]

\def\sxa{0}
\def\sxb{3.45}
\def\sxc{6.9}
\def\sxd{10.35}
\def\sxe{13.8}

\node[stage] (input) at (\sxa,0) {};
\node[stage] (filt)  at (\sxb,0) {};
\node[stage] (pd)    at (\sxc,0) {};
\node[stage] (embed) at (\sxd,0) {};
\node[stage] (clf)   at (\sxe,0) {};

\begin{scope}[shift={(input)}, scale=1.4]
  \coordinate (g1) at (-0.62, 0.85);
  \coordinate (g2) at (-0.05, 0.98);
  \coordinate (g3) at ( 0.58, 0.75);
  \coordinate (g4) at ( 0.30, 0.38);
  \coordinate (g5) at (-0.40, 0.32);
  \draw[gray!65, line width=0.5pt] (g1)--(g2) (g2)--(g3) (g3)--(g4)
                                    (g4)--(g5) (g5)--(g1) (g1)--(g4);
  \foreach \p in {g1,g2,g3,g4,g5}
    {\fill[blue!70!black] (\p) circle (1.4pt);}
  \draw[gray!35, dotted, line width=0.5pt] (-0.95,0.02)--(-0.15,0.02);
  \draw[gray!35, dotted, line width=0.5pt] ( 0.15,0.02)--( 0.95,0.02);
  \node[font=\tiny, gray!60!black] at (0,0.02) {\emph{or}};
  \foreach \x/\y in {-0.55/-0.25, -0.15/-0.15, 0.30/-0.20, 0.60/-0.50,
                     0.62/-0.85, 0.25/-1.05, -0.25/-1.05, -0.60/-0.80,
                     -0.68/-0.50, -0.05/-0.65}
    {\fill[blue!70!black] (\x,\y) circle (1pt);}
\end{scope}

\begin{scope}[shift={(filt)}, scale=1.4]
  \coordinate (q1) at (-0.70, 0.60);
  \coordinate (q2) at (-0.05, 0.90);
  \coordinate (q3) at ( 0.65, 0.50);
  \coordinate (q4) at ( 0.70,-0.25);
  \coordinate (q5) at ( 0.05,-0.65);
  \coordinate (q6) at (-0.70,-0.35);
  \fill[blue!20, fill opacity=0.55] (q1) -- (q2) -- (q6) -- cycle;
  \draw[blue!55!black, line width=0.7pt]
    (q1)--(q2) (q2)--(q3) (q3)--(q4) (q4)--(q5) (q5)--(q6) (q6)--(q1);
  \draw[blue!55!black, line width=0.7pt] (q2)--(q6);
  \draw[gray!40, fill=gray!8, fill opacity=0.4, line width=0.35pt]
    (q3) circle (0.32);
  \draw[gray!40, fill=gray!8, fill opacity=0.4, line width=0.35pt]
    (q5) circle (0.32);
  \foreach \p in {q1,q2,q3,q4,q5,q6}
    {\fill[blue!70!black] (\p) circle (1.5pt);}
  \draw[gray!55, -{Stealth[length=3pt,width=2.5pt]}, line width=0.4pt]
    (-0.70,-1.00) -- (0.70,-1.00);
  \node[font=\tiny, gray!60!black, anchor=west] at (0.72,-1.00) {$r$};
\end{scope}

\begin{scope}[shift={(pd)}, scale=1.2]
  \draw[-{Stealth[length=3pt,width=2.5pt]}, gray!55, line width=0.4pt]
    (-0.90,-0.85)--(0.92,-0.85);
  \draw[-{Stealth[length=3pt,width=2.5pt]}, gray!55, line width=0.4pt]
    (-0.90,-0.85)--(-0.90, 0.90);
  \node[font=\tiny, gray!55!black, anchor=north] at ( 0.92,-0.88) {$b$};
  \node[font=\tiny, gray!55!black, anchor=east]  at (-0.92, 0.88) {$d$};
  \draw[gray!55, dashed, line width=0.4pt] (-0.90,-0.85)--(0.80,0.85);
  \foreach \x/\y in {-0.45/0.55, -0.10/0.68, -0.55/0.28}
    {\fill[red!75!black] (\x,\y) circle (1.5pt);}
  \foreach \x/\y in {-0.78/-0.58, -0.72/-0.52, -0.75/-0.65, -0.68/-0.48, -0.80/-0.68}
    {\fill[blue!70!black] (\x,\y) circle (1pt);}
\end{scope}

\begin{scope}[shift={(embed)}, scale=1.2]
  \draw[-{Stealth[length=3pt,width=2.5pt]}, gray!55, line width=0.4pt]
    (-0.95,-0.90)--(0.95,-0.90);
  \draw[-{Stealth[length=3pt,width=2.5pt]}, gray!55, line width=0.4pt]
    (-0.95,-0.90)--(-0.95, 0.95);
  \node[font=\tiny, gray!55!black, anchor=north] at ( 0.95,-0.92) {$b$};
  \node[font=\tiny, gray!55!black, anchor=east]  at (-0.97, 0.95) {$d$};
  \draw[gray!55, dashed, line width=0.4pt] (-0.95,-0.90)--(0.85,0.90);
  \foreach \x/\y in {-0.65/-0.20, -0.65/0.20, -0.65/0.60,
                     -0.25/0.20, -0.25/0.60,
                      0.15/0.60,  0.15/0.20,
                      0.55/0.60}
    {\fill[blue!65!black] (\x,\y) circle (1.2pt);}
  \coordinate (La) at (-0.25, 0.20);
  \coordinate (Lb) at (-0.65, 0.20);
  \coordinate (Lc) at (-0.25, 0.60);
  \coordinate (Ld) at ( 0.15, 0.60);
  \foreach \p in {La,Lb,Lc,Ld}
    \draw[orange!55, fill=orange!15, fill opacity=0.45, line width=0.35pt]
      (\p) ++(-0.20,-0.20) rectangle ++(0.40, 0.40);
  \foreach \p in {La,Lb,Lc,Ld}
    {\fill[orange!80!black] (\p) circle (1.5pt);}
  \fill[red!75!black] (-0.40, 0.25) circle (1.7pt);
  \fill[red!75!black] ( 0.00, 0.55) circle (1.7pt);
  \draw[red!45, dotted, line width=0.5pt] (-0.40, 0.25) -- (La);
  \draw[red!45, dotted, line width=0.5pt] (-0.40, 0.25) -- (Lb);
  \draw[red!45, dotted, line width=0.5pt] ( 0.00, 0.55) -- (Lc);
  \draw[red!45, dotted, line width=0.5pt] ( 0.00, 0.55) -- (Ld);
  \node[font=\tiny, gray!55!black] at (0.15, -0.72)
    {$\textstyle\Phi_p=\sum_{a\in A}\varphi_{R,p}(a)$};
\end{scope}

\begin{scope}[shift={(clf)}, scale=1.4]
  \fill[blue!7]  (-0.95,-0.80) rectangle (0.05, 0.85);
  \fill[red!6]   (0.05,-0.80)  rectangle (0.95, 0.85);
  \draw[gray!55, dashed, line width=0.4pt] (0.05,-0.80)--(0.05, 0.85);
  \fill[blue!70!black] (-0.50, 0.20) circle (1.7pt);
  \fill[red!75!black]  ( 0.55, 0.05) circle (1.7pt);
  \node[font=\tiny, blue!60!black, anchor=south east] at (-0.55, 0.55) {$\mu_c$};
  \node[font=\tiny, red!65!black, anchor=west]  at ( 0.60, 0.35) {$\mu_{c'}$};
  \draw[blue!60, line width=0.6pt] (-0.50, 0.20) circle (0.30);
  \node[font=\tiny, blue!55!black, anchor=south east] at (-0.23, 0.18) {$r_m$};
  \fill[gray!80] (-0.12, -0.50) circle (1.1pt);
  \draw[gray!65, -{Stealth[length=3pt,width=2.5pt]}, line width=0.5pt]
    (-0.12, -0.50) -- (-0.40, 0.10);
  \node[font=\tiny, gray!60!black] at (-0.12, -0.72) {$\Phi(A_{\mathrm{test}})$};
\end{scope}

\node[slabel, above=3pt of input] {\textsc{input}};
\node[slabel, above=3pt of filt]  {\textsc{filtration}};
\node[slabel, above=3pt of pd]    {\textsc{persistence}};
\node[slabel, above=3pt of embed] {\textsc{embedding}};
\node[slabel, above=3pt of clf]   {\textsc{classify}};

\node[sbelow, below=3pt of input.south] {graph or\\point cloud};
\node[sbelow, below=3pt of filt.south]  {complex $K_r$,\\$r$ grows};
\node[sbelow, below=3pt of pd.south]    {diagram\\$A=\{(b_i,d_i)\}$};
\node[sbelow, below=3pt of embed.south] {$\Phi(A)\in\RR^\ell$};
\node[sbelow, below=3pt of clf.south]   {$\widehat{y}$ with\\certificate};

\draw[flow] (input.east) -- (filt.west);
\draw[flow] (filt.east)  -- (pd.west);
\draw[flow] (pd.east)    -- (embed.west);
\draw[flow] (embed.east) -- (clf.west);

\end{tikzpicture}
\caption{\textbf{The PLACE pipeline.}
A point cloud or graph (left) is converted to a persistence diagram
through a filtration---a growing sequence of simplicial complexes---%
then embedded to $\RR^\ell$ by summing hat-function coordinates over
a landmark grid: each diagram point (red) contributes to the
coordinates indexed by the landmarks (orange) whose $d_\mathcal{B}$-cover
squares it falls within, via $\Phi_p(A)=\sum_{a \in A}\varphi_{R,p}(a)$.
The embedded vector is then classified by a linear rule.
Every ingredient---the descriptor choice, the grid scales $R_k$,
the weights $w_k$, and the certificate threshold---is fixed
analytically from training labels alone, with no held-out
calibration or cross-validation.}
\label{fig:pipeline}
\end{figure}
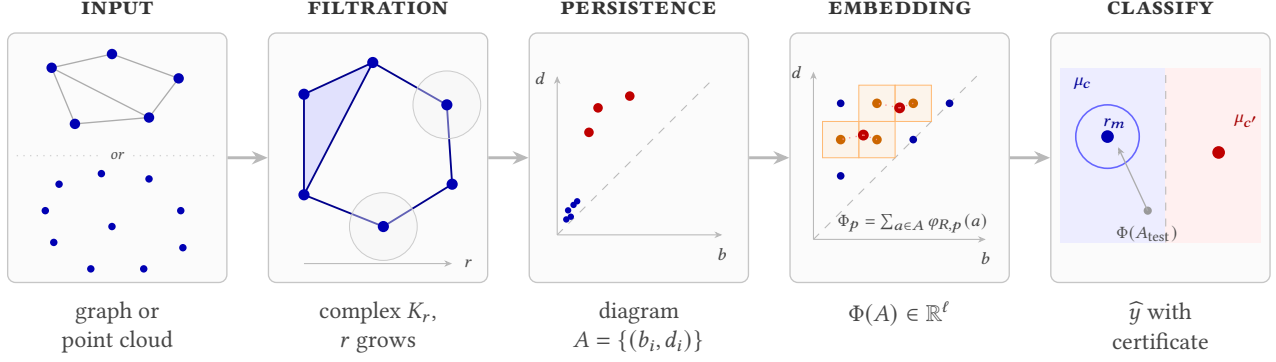

\subsection{Our Contribution and Organization}
\label{sec:contributions}

We make four contributions; all are closed-form, computationally
efficient, and validated on $12$ benchmarks:
\emph{(i)}~A summation embedding of dimension $\ell = O(MN)$
specializing the $n$-fold construction of \citet{Mitra2024} to
linear-in-grid complexity, with an explicit constant-floor lower
bound on the bottleneck metric $\D{n}$ holding under
$\nu$-coherence (Proposition~\ref{prop:n_point_lower}).
The summation specialization trades the unconditional but
exponential ($M^n$) lower bound of \citet{Mitra2024} for a conditional but
linear ($MN$) one; the matching-free $\nu$-coherence hypothesis
holds on $\geq 99.7\%$ of cross-class pairs in the
Section~\ref{sec:experiments} audit.  The closed-form scale weights
$w_k^2 \propto (d_{k+1}^2 - d_k^2)/R_k^2$
(equation~\eqref{eq:closed_form_weights}, Section~\ref{sec:preli})
are the unique maximizer of the distortion slope $\lambda(\nu)$
in Mitra--Virk's affine certificate of
Corollary~\ref{cor:lipschitz}. Beyond bi-Lipschitz stability,
$\lambda(\nu)$ bridges geometry and statistics: it controls how
the class-mean separation $\Delta$ in (ii) inherits from
diagram-level separation (Proposition~\ref{prop:lambda_sep}).
$\nu$-coherence is the proof's actual mechanism (a per-scale
block-norm floor on $\Phi(A) - \Phi(B)$); empirically it holds on
$\geq 99.7\%$ of pairs in the Section~\ref{sec:experiments} audit
($100\%$ on three of four benchmarks).  The empirical rate of
(ii) flows through $\Delta > 0$ directly via
Theorem~\ref{thm:fisher_bound}, independent of any pairwise
condition.
\emph{(ii)}~A margin-based excess-risk rate
$O\bigl((k-1)R/(\Delta\sqrt{m_{\min}})\bigr)$, driven by class-mean
separation $\Delta$ alone and independent of the cross-pair
hypothesis of (i), with a matching Le~Cam sample-starved lower
bound (constant excess risk for $m \lesssim R/\Delta$); the upper
rate uses no tunable parameters beyond the closed-form pipeline
(Theorems~\ref{thm:fisher_bound}--\ref{thm:lower_bound},
Section~\ref{sec:metric}).
\emph{(iii)}~A closed-form descriptor-selection rule given by
the Mahalanobis margin under Ledoit--Wolf shrinkage---the LDA
Bayes-margin form of the Fisher discriminant ratio---empirically
the strongest closed-form ranker on a heterogeneous
$64$-descriptor chemical-graph stress test (mean Spearman
$\rho = +0.56$ across $11$ benchmarks, range $-0.24$ to
$+0.89$, positive on $10$ of $11$), with the
isotropic surrogate $\Delta/\sqrt{\ell}$ admitting a
closed-form selection-consistency rate
(Proposition~\ref{prop:selection_consistency},
Corollary~\ref{cor:data_driven_rate},
Remark~\ref{rem:fisher_ratio},
Section~\ref{sec:filt_select_theory}).
\emph{(iv)}~A certificate for individual point predictions,
decided once at training time from $\hat\Delta$ and $r_m$,
with no per-prediction overhead, in three complementary
forms---a non-asymptotic Pinelis radius, an asymptotic Gaussian
plug-in radius, and a non-asymptotic variance-aware
Pinelis--Bernstein radius---with per-dataset firing-rate
diagnostics across all $12$ benchmarks
(Theorem~\ref{thm:confidence_containment},
Table~\ref{tab:cert_firing},
Section~\ref{sec:certified});
unlike conformal prediction this requires neither a
calibration split nor an inflation factor, only $\hat\Delta$
and (for the variance-aware forms) the empirical covariances
$\hat\Sigma_c$.
Empirically (Section~\ref{sec:experiments}), PLACE is the
strongest diagram-based method on Orbit5k, matches the strongest
topology-based baseline within statistical noise on MUTAG and
COX2, and exhibits quantitative gaps on the remaining graph
datasets that fall into two diagnosable regimes
(descriptor-blindness or pool-coverage limits;
Section~\ref{sec:filtration_selection}).

The remaining degree of freedom not analytically pinned by
contribution (i)---the \emph{positions} of the landmarks, which
change the grid combinatorially---is left to future work.

\subsection{Related Work}\label{sec:related}

We situate PLACE at the intersection of three lines of work:
certified machine learning, persistence diagram vectorizations
and their neural extensions, and the Mitra--Virk landmark
embedding~\citep{Mitra2021,Mitra2024}.

\paragraph{Certified machine learning.}
Three families of methods attach correctness guarantees to
classifier predictions.
\emph{Conformal prediction}
\citep{VovkEtAl2005,Lei2018,Vovk2013,Angelopoulos2023} constructs
prediction \emph{sets} $\hat C(x)$ with marginal coverage
$\mathbb{P}(y \in \hat C(x)) \geq 1-\alpha$ using a held-out
calibration set; the guarantee is distribution-free but applies
to the set, not to any single label, and requires data splitting.
\emph{Selective classification} \citep{Chow1970,Geifman2017,Geifman2019} and
learning with rejection \citep{Bartlett2008,Cortes2016}
let the classifier abstain on low-confidence inputs but provide
no probabilistic correctness guarantee for accepted predictions.
PLACE differs from both families: its certificate applies to
individual point predictions, requires no calibration data, and
is decided once at training time
(Section~\ref{sec:certified}).

\paragraph{Persistence diagram vectorizations.}
Persistence landscapes \citep{Bubenik15} embed a diagram
as a sequence of piecewise linear functions, stable under the
bottleneck distance but potentially of high dimension for
peaked diagrams.
Persistence images \citep{Adams2017} discretize a
weighted Gaussian mixture supported on the diagram onto a fixed
grid; choice of weighting function and bandwidth are free
parameters.
Sliced Wasserstein and persistence scale-space kernels
\citep{Kusano2016,Carriere2017} avoid an explicit
finite-dimensional feature map in favor of a positive-definite
kernel over diagrams.
Weighted kernels (WKPI) of \citet{yusu_metric_learning} learn the
Gaussian weight function via metric learning on held-out labels.
Neural extensions---PersLay \citep{Carriere2020} and
Persformer \citep{Reinauer2021}---learn the vectorization
end-to-end.
A common feature of all the above is a Lipschitz \emph{upper}
bound on embedding distortion under bottleneck perturbation and
the absence of any \emph{lower} bound: bottleneck-distant
diagrams may collapse to identical features.
None of the above constructions carries a per-prediction
correctness certificate.
See the recent survey of \citet{Ali2023-ht} for a
taxonomy of these vectorizations.
Table~\ref{tab:related-vec} summarizes the comparison.

\paragraph{Topology with neural networks.}
A parallel line of work couples persistence with deep learning
more tightly than a frozen vectorization.
\citet{Hofer2017} first embedded persistence diagrams
through a learnable layer trained end-to-end with a downstream
network.
\citet{Gabrielsson2020} expose filtration and
persistence as a differentiable topology layer over general
simplicial complexes, permitting task-driven filtrations.
On graphs specifically, topological
GNNs~\citep{Horn2022} inject persistence features as channels into
message-passing architectures, reusing the inductive bias of
standard GNN backbones~\citep{GIN2019,RetGK2018} while
augmenting them with global topological summaries; the Euler
characteristic transform of \citet{ECP2024} is a lightweight alternative
that captures shape structure at GNN-competitive cost.
These neural/GNN approaches typically reach higher empirical
accuracy on label-rich datasets (NCI1, NCI109) by learning
filtration and vectorization jointly on held-out data; the
distinguishing contribution of PLACE is orthogonal---an
analytically fixed embedding that admits a closed-form
descriptor-selection criterion and a per-prediction
correctness certificate, neither of which has been demonstrated
for the learned families above.

\begin{table}[ht]
\centering
\caption{Persistence diagram vectorizations at a glance.
  \textbf{Lipschitz:} upper stability under the bottleneck
  distance.
  \textbf{Lower dist.:} an explicit lower bound on
  $\|\Phi(a)-\Phi(a')\|$ (multiplicative or constant-floor),
  on the indicated subspace of diagrams.
  \textbf{Poly.\ dim:} embedding dimension polynomial in the
  grid size $M$.
  \textbf{Tuning-free:} all embedding parameters (scales,
  weights, kernel width, grid resolution) are fixed
  analytically---no held-out validation, no learned weights.
  \textbf{Cert.:} a correctness certificate of some form.
  Mitra--Virk's lower modulus $\rho_-$ certifies metric
  \emph{distortion} (topological differences survive embedding);
  PLACE's $\lambda(\nu)$ certifies the same and additionally
  drives a \emph{classification} certificate (the empirical
  prediction matches the population prediction with probability
  $\geq 1-\alpha$, Theorem~\ref{thm:confidence_containment}).
  PLACE is the only construction ticking every column.}
\label{tab:related-vec}
\small
\setlength{\tabcolsep}{4pt}
\resizebox{\textwidth}{!}{%
\begin{tabular}{lccccc}
\toprule
\textbf{Method} & \textbf{Lipschitz} & \textbf{Lower dist.}
  & \textbf{Poly.\ dim} & \textbf{Tuning-free}
  & \textbf{Cert.} \\
\midrule
Landscapes \citeyearpar{Bubenik15}
  & \checkmark & ---
  & \checkmark
  & \xmark~{\scriptsize(levels)}
  & \xmark \\
Persistence images \citeyearpar{Adams2017}
  & \checkmark & ---
  & \checkmark
  & \xmark~{\scriptsize($\sigma$, grid, weight)}
  & \xmark \\
SW\,/\,PSS kernels \citeyearpar{Kusano2016,Carriere2017}
  & \checkmark & ---
  & {\scriptsize implicit}
  & \xmark~{\scriptsize(bandwidth)}
  & \xmark \\
WKPI \citeyearpar{yusu_metric_learning}
  & \checkmark & ---
  & \checkmark
  & \xmark~{\scriptsize(learned $w$)}
  & \xmark \\
PersLay\,/\,Persformer \citeyearpar{Carriere2020,Reinauer2021}
  & {\scriptsize learned} & ---
  & \checkmark
  & \xmark~{\scriptsize(end-to-end)}
  & \xmark \\
Mitra--Virk asdim \citeyearpar{Mitra2021}
  & \checkmark
  & {\scriptsize existential ($\mathrm{asdim} = 2n$)}
  & ---
  & ---
  & {\scriptsize metric only} \\
Mitra--Virk $n$-fold \citeyearpar{Mitra2024}
  & \checkmark
  & \checkmark~{\scriptsize $\rho_-$ on $\{\db \geq R_1\}$, unconditional}
  & \xmark~{\scriptsize($NM^n$)}
  & \checkmark
  & {\scriptsize metric only} \\
\rowcolor{blue!6}
\textbf{PLACE (ours)}
  & \checkmark
  & \checkmark~{\scriptsize $\rho_-$ on $\D{n}$, conditional ($\nu$-coherent, $\geq 99.7\%$)}
  & \checkmark~{\scriptsize ($MN$)}
  & \checkmark
  & {\scriptsize classification} \\
\bottomrule
\end{tabular}%
}
\end{table}

\paragraph{Mitra--Virk landmark embedding and why we build on it.}
\citet{Mitra2021} establish the coarse embedding
$\D{n} \hookrightarrow \ell^2$ existentially via proving
$\mathrm{asdim}(\D{n}, \db) = 2n$, without making the coarse embedding explicit.  \citet{Mitra2024} give the first explicit such
coarse embedding with computable distortion constants $\rho_-$
on $\{\db \geq R_1\}$; their construction is unconditional---no
matching or coherence hypothesis is required---at the cost of
$M^n$ coordinates per scale, exponential in the diagram
cardinality.  Their work is purely metric-theoretic and does
not address classification or downstream learning tasks.
Three geometric features of their construction produce that
bound, and PLACE is designed to retain them.
\emph{(i) Compactly supported ramp coordinates.}
Each $\varphi_{R,p}$ is a $1$-Lipschitz hat function with fixed
peak $3R/2$ and bounded support in bottleneck distance; pointwise
changes translate into bounded, traceable changes in the embedded
vector.
Gaussian kernels used by persistence images spread mass across
all landmarks, so no single coordinate is responsible for a local
displacement and the constants in any lower bound degrade with
the bandwidth.
\emph{(ii) Cover structure of the grid.}
The landmark lattice is designed so every diagram point lies
within $3R/2$ of some landmark; this is what lifts the pointwise
Lipschitz property into a bi-Lipschitz embedding
(Proposition~\ref{prop:n_point_lower}).
Order-statistic constructions such as persistence
landscapes~\citep{Bubenik15} are nonlinear in the empirical diagram
measure, and kernel-based vectorizations (SW, PSS) work
implicitly and admit no finite-dimensional grid with this property.
\emph{(iii) Analytically optimal weights and scales.}
The distortion slope $\lambda(\nu)$ of
Corollary~\ref{cor:lipschitz} admits a unique weight-vector
argmax in closed form
(equation~\eqref{eq:closed_form_weights} in
Section~\ref{sec:filt_select_theory}), eliminating the
hyperparameter search (bandwidth, resolution, learned weighting)
that PI and WKPI require.

PLACE retains these three ingredients but replaces the
$n$-fold composition---$M^n$ coordinates per scale,
unconditional lower bound---with a summation over single-point
evaluations (Section~\ref{sec:preli}), reducing the per-scale
dimension to $M$ and the total to $\ell = O(MN)$.  Summation
pooling collapses the $M^n$ cross-pair coordinates into $M$,
which is what enables the cancellation construction of
Remark~\ref{rem:noninterference_why}; the resulting lower bound
is conditional on $\nu$-coherence rather than unconditional.
The trade is \emph{exponential-unconditional} \citep{Mitra2024} for
\emph{linear-conditional} (PLACE), with empirical near-universality
of $\nu$-coherence ($\geq 99.7\%$;
Section~\ref{sec:experiments}) making the conditional close to
operational.  At the same time, PLACE trades Mitra and Virk's individual-pair
$\rho_-$ guarantee for a data-dependent class-mean separation
$\Delta$ that drives the classification theory of
Section~\ref{sec:metric}.
The closed-form specification is what lets PLACE deliver a
classifier, a descriptor ranking, and a per-prediction
certificate from training labels alone; stripped of learned
weights, WKPI~\citep{yusu_metric_learning} reduces to an ordinary
persistence image, and the WKPI numbers in
Table~\ref{tab:exp1} are achievable only with the learned weight
function, not with the underlying vectorization.

\section{Persistence Landmark Embedding}\label{sec:preli}

A \emph{persistence diagram} $A = \{a_1, \ldots, a_n\}$ is a
finite multiset of $n$ points $(b_i, d_i)$ with $d_i > b_i \geq 0$
(Figure~\ref{fig:intro_pd}); we write $\mathcal{D} = \bigcup_n \D{n}$
for the space of all such finite diagrams.
The \emph{bottleneck distance}
$\db(A,B) = \min_{\sigma} \max_i d_\infty(a_i, b_{\sigma(i)})$ is the
optimal matching cost~\citep{Cohen-Steiner2007}, where the matching
$\sigma$ pairs points of $A$ with points of $B$ and matches any
unmatched points of $A$ or $B$ to the formal diagonal $\ast$, with
matching cost $d_\infty(a, \ast) = (d-b)/2$ for $a = (b,d)$.
On single-point diagrams $\D{1}$,
$\db(p,p') = \min\{d_\infty(p,p'),\, \max\{d_\infty(p,\ast), d_\infty(p',\ast)\}\}$.

We embed diagrams into $\RR^\ell$ by specializing the landmark
construction of \citet{Mitra2024}; all diagrams are
assumed to be supported in the compact region
$\mathcal{T}_L := \{(b,d) : 0 \leq b < d \leq L\}$ and to have
cardinality bounded by some $N_{\max} < \infty$ (the specific
value used in experiments is given in
Section~\ref{sec:experiments}).  The cardinality bound is
structurally necessary, not merely a practical convenience:
\begin{itemize}
\item \citet{Carriere2019-nx} show that even on bounded-cardinality
diagrams $\D{n}$ (with $n$ fixed), no bi-Lipschitz embedding into
Hilbert space exists.
\item \citet{Mitra2021} and \citet{BubenikWagner2020} show that even
on the union $\mathcal{D} = \bigcup_n \D{n}$ of all finite diagrams,
no coarse embedding into Hilbert space exists.
\item \citet{Zava2025} extends the impossibility to the
unbounded-cardinality Gromov--Hausdorff space (whose $1$D
Euclidean--Hausdorff specialization contains $(\D{},\db)$ as a
subspace), ruling out coarse embeddings into any uniformly convex
Banach space and hence into any Hilbert space.
\end{itemize}

The matching positive direction in the bounded regime is
established by~\citet{Mitra2021}, who prove
$\mathrm{asdim}(\D{n}, \db) = 2n$ existentially.  \citet{Mitra2024} make the
existence concrete: their $n$-fold landmark composition gives
the first explicit coarse embedding of $\D{n}$ with computable
distortion constants $\rho_-$ on $\{\db \geq R_1\}$, at $M^n$
coordinates per scale.  The construction
\eqref{eq:multiscale_embed} below specializes that $n$-fold
form to a summation, reducing complexity to $\ell = O(MN)$ at
the cost of a $\nu$-coherence conditional on the lower bound
(Proposition~\ref{prop:n_point_lower}).
For the analysis below, each diagram is padded with diagonal
points $*$ to cardinality exactly $N_{\max}$, so all diagrams
lie in $\D{N_{\max}}$; this leaves $\db$ unchanged and adds zero
contribution to every non-diagonal landmark coordinate, so the
embedding $\Phi$ of \eqref{eq:multiscale_embed} below is
unaffected on the original points.
Fix a scale $R > 0$.
The \emph{landmark grid} $\GG_R$ is the finite set of
single-point diagrams in $\D{1}$ whose single point, called a
\emph{landmark}, lies in the lattice
\begin{equation}\label{eq:grid}
  \GG_R \;=\; \bigl\{\,(mR, nR) \,:\,
    m \in \{1,3,5,\ldots\},\;
    n \in \{4,6,8,\ldots\},\;
    n \geq m+3\,\bigr\}
  \,\cap\, [0,L]^2 .
\end{equation}
We write $\GG_R^+ := \GG_R \cup \{\ast\}$, adjoining the formal
diagonal landmark $\ast$. The parity condition makes the $d_\mathcal{B}$-balls of radius
$\tfrac{3R}{2}$ centered at the points of $\GG_R^+$ cover $\D{1}$
with multiplicity at most four~\citep[Lemma~3.5]{Mitra2024}
(Figure~\ref{fig:landmark_grid}).
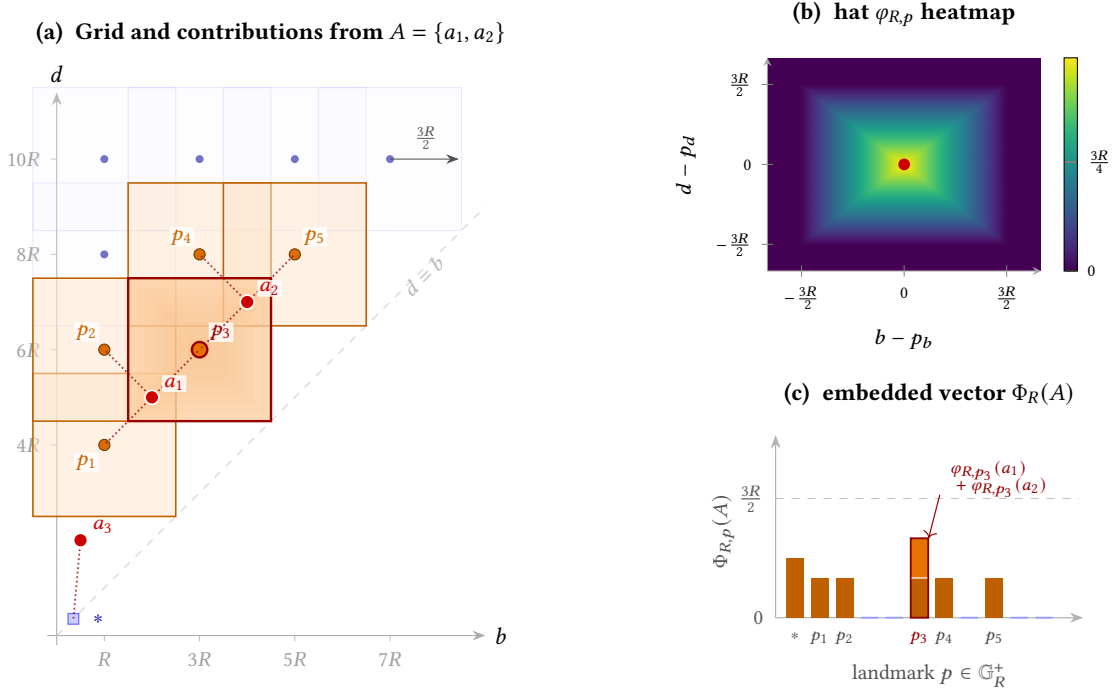
\begin{figure}[t]
\centering
\begin{minipage}[c][9.5cm][c]{0.54\textwidth}
\centering
\begin{tikzpicture}[scale=0.63, font=\footnotesize,
  lm/.style={circle,fill=blue!60!black,opacity=0.55,inner sep=0pt,minimum size=3.0pt},
  lmA/.style={circle,fill=orange!85!black,draw=orange!40!black,line width=0.3pt,inner sep=0pt,minimum size=4.2pt},
  lmS/.style={circle,fill=orange!90!black,draw=red!60!black,line width=0.9pt,inner sep=0pt,minimum size=5.8pt},
  pt/.style={circle,fill=red!80!black,draw=white,line width=0.6pt,inner sep=0pt,minimum size=5.2pt},
  ballBg/.style={draw=blue!25,fill=blue!3,fill opacity=0.18,draw opacity=0.35,line width=0.25pt},
  ballAct/.style={draw=orange!75!black,fill=orange!20,fill opacity=0.55,line width=0.6pt},
  ballShared/.style={draw=red!60!black,fill=orange!35,fill opacity=0.75,line width=0.9pt},
]
  \draw[-{Stealth}, gray!65, line width=0.4pt](-0.2,0)--(9.0,0)node[right,black]{$b$};
  \draw[-{Stealth}, gray!65, line width=0.4pt](0,-0.2)--(0,11.4)node[above,black]{$d$};
  \draw[gray!40,dashed,line width=0.3pt](0,0)--(9.0,9.0);
  \node[gray!55,font=\scriptsize,anchor=west,rotate=45] at (7.2,7.05) {$d=b$};
  \foreach \t/\lab in {1/$R$, 3/$3R$, 5/$5R$, 7/$7R$}{
    \draw[gray!55,very thin](\t,0)--(\t,-0.12);
    \node[gray!75,font=\scriptsize,anchor=north]at(\t,-0.14){\lab};
  }
  \foreach \t/\lab in {4/$4R$, 6/$6R$, 8/$8R$, 10/$10R$}{
    \draw[gray!55,very thin](0,\t)--(-0.12,\t);
    \node[gray!75,font=\scriptsize,anchor=east]at(-0.14,\t){\lab};
  }
  \foreach \m/\n in {1/8, 1/10, 3/10, 5/10, 7/10}{
    \draw[ballBg](\m-1.5,\n-1.5)rectangle(\m+1.5,\n+1.5);
  }
  \foreach \m/\n in {1/4, 1/6, 3/8, 5/8}{
    \draw[ballAct](\m-1.5,\n-1.5)rectangle(\m+1.5,\n+1.5);
  }
  \foreach \i in {0,1,...,7}{
    \pgfmathsetmacro\r{1.5 - \i*0.18}
    \fill[orange!85!black, opacity=0.09] (3-\r, 6-\r) rectangle (3+\r, 6+\r);
  }
  \draw[ballShared](3-1.5,6-1.5)rectangle(3+1.5,6+1.5);
  \foreach \m/\n in {1/8, 1/10, 3/10, 5/10, 7/10}{
    \node[lm]at(\m,\n){};
  }
  \foreach \m/\n in {1/4, 1/6, 3/8, 5/8}{
    \node[lmA]at(\m,\n){};
  }
  \node[lmS] at (3,6) {};
  \tikzset{lmlab/.style={fill=white,fill opacity=0.82,text opacity=1,inner sep=0.8pt,rounded corners=0.5pt}}
  \node[orange!80!black,font=\scriptsize,anchor=north east,lmlab]at(1-0.1,4-0.15){$p_1$};
  \node[orange!80!black,font=\scriptsize,anchor=south east,lmlab]at(1-0.1,6+0.15){$p_2$};
  \node[red!65!black,font=\scriptsize\bfseries,anchor=south west,lmlab]at(3+0.18,6+0.15){$p_3$};
  \node[orange!80!black,font=\scriptsize,anchor=south east,lmlab]at(3-0.1,8+0.15){$p_4$};
  \node[orange!80!black,font=\scriptsize,anchor=south west,lmlab]at(5+0.15,8+0.15){$p_5$};
  \node[draw=blue!55,fill=blue!20,line width=0.5pt,rectangle,inner sep=0pt,minimum size=4pt]
    at (0.35,0.35) {};
  \node[blue!70!black,font=\scriptsize,anchor=west]at(0.55,0.35){$\ast$};
  \draw[red!60!black,densely dotted,line width=0.7pt,opacity=0.75] (2,5)--(1,4);
  \draw[red!60!black,densely dotted,line width=0.7pt,opacity=0.75] (2,5)--(1,6);
  \draw[red!60!black,densely dotted,line width=0.7pt,opacity=0.75] (2,5)--(3,6);
  \draw[red!60!black,densely dotted,line width=0.7pt,opacity=0.75] (4,7)--(3,6);
  \draw[red!60!black,densely dotted,line width=0.7pt,opacity=0.75] (4,7)--(3,8);
  \draw[red!60!black,densely dotted,line width=0.7pt,opacity=0.75] (4,7)--(5,8);
  \node[pt] at (2,5) {};
  \node[red!80!black,font=\scriptsize\bfseries,anchor=south west,lmlab]at(2.2,5.1){$a_1$};
  \node[pt] at (4,7) {};
  \node[red!80!black,font=\scriptsize\bfseries,anchor=south west,lmlab]at(4.2,7.1){$a_2$};
  \draw[red!60!black,densely dotted,line width=0.7pt,opacity=0.75] (0.5,2)--(0.35,0.35);
  \node[pt] at (0.5,2) {};
  \node[red!80!black,font=\scriptsize\bfseries,anchor=south west,lmlab]at(0.7,2.1){$a_3$};
  \draw[-{Stealth},thin,gray!65!black,shorten >=1pt]
    (7,10) -- node[midway,above,sloped,font=\tiny,gray!70!black]{$\tfrac{3R}{2}$} (8.5,10);
  \node[anchor=south,font=\footnotesize\bfseries]at(4.5,12.2){(a)\; Grid and contributions from $A=\{a_1,a_2\}$};
\end{tikzpicture}
\end{minipage}%
\hfill
\begin{minipage}[c][9.5cm][c]{0.44\textwidth}
\centering
\begin{tikzpicture}
\begin{axis}[
  width=5.2cm, height=4.4cm,
  view={0}{90},
  xlabel={$b - p_b$}, ylabel={$d - p_d$},
  xlabel style={font=\scriptsize},
  ylabel style={font=\scriptsize},
  xtick={-1.5, 0, 1.5},
  xticklabels={$-\tfrac{3R}{2}$, $0$, $\tfrac{3R}{2}$},
  ytick={-1.5, 0, 1.5},
  yticklabels={$-\tfrac{3R}{2}$, $0$, $\tfrac{3R}{2}$},
  tick label style={font=\tiny},
  xmin=-2, xmax=2, ymin=-2, ymax=2,
  axis lines=left,
  axis line style={thin, gray!70},
  enlargelimits=false,
  colormap/viridis,
  colorbar,
  colorbar style={
    width=5pt,
    ytick={0, 0.75, 1.5},
    yticklabels={$0$, $\tfrac{3R}{4}$, $\tfrac{3R}{2}$},
    tick label style={font=\tiny},
  },
  title={\footnotesize\bfseries (b)\; hat $\varphi_{R,p}$ heatmap},
  title style={yshift=2pt},
]
\addplot3[surf, shader=interp, samples=60, domain=-2:2, y domain=-2:2, z buffer=sort]
  { max(0, 1.5 - max(abs(x), abs(y))) };
\addplot3[only marks, mark=*, mark size=2pt, color=red!80!black]
  coordinates {(0, 0, 1.6)};
\end{axis}
\end{tikzpicture}

\vspace{0.4em}

\begin{tikzpicture}[font=\footnotesize, x=0.55cm, y=1.05cm]
  \draw[-{Stealth}, gray!60, line width=0.45pt] (0, 0) -- (7.4, 0);
  \draw[-{Stealth}, gray!60, line width=0.45pt] (0, 0) -- (0, 2.3);
  \draw[gray!50, dashed, line width=0.35pt] (0, 1.5) -- (7.3, 1.5);
  \node[gray!70!black, font=\scriptsize, anchor=east] at (-0.05, 1.5) {$\tfrac{3R}{2}$};
  \node[gray!70!black, font=\scriptsize, anchor=east] at (-0.05, 0) {$0$};
  \node[gray!65!black, font=\scriptsize, rotate=90, anchor=south] at (-0.75, 1.1) {$\Phi_{R,p}(A)$};
  \foreach \x in {3, 4, 7, 9, 10}{
    \pgfmathsetmacro\xpos{0.25 + \x*0.6}
    \draw[blue!40, line width=0.7pt] (\xpos, 0) -- ++(0.42, 0);
  }
  \pgfmathsetmacro\xast{0.25 + 0*0.6}
  \fill[orange!75!black] (\xast, 0) rectangle ++(0.42, 0.75);
  \node[font=\tiny, gray!60!black, anchor=north] at (\xast + 0.21, -0.05) {$\ast$};
  \foreach \x/\lab in {1/$p_1$, 2/$p_2$, 6/$p_4$, 8/$p_5$}{
    \pgfmathsetmacro\xpos{0.25 + \x*0.6}
    \fill[orange!75!black] (\xpos, 0) rectangle ++(0.42, 0.5);
    \node[font=\tiny, gray!60!black, anchor=north] at (\xpos + 0.21, -0.05) {\lab};
  }
  \pgfmathsetmacro\xsh{0.25 + 5*0.6}
  \fill[orange!75!black] (\xsh, 0) rectangle ++(0.42, 0.5);
  \fill[orange!90!black] (\xsh, 0.5) rectangle ++(0.42, 0.5);
  \draw[red!55!black, line width=0.7pt] (\xsh, 0) rectangle (\xsh + 0.42, 1.0);
  \draw[white, line width=0.45pt] (\xsh+0.02, 0.5) -- ++(0.38, 0);
  \node[font=\tiny, red!60!black, anchor=north] at (\xsh + 0.21, -0.05) {$p_3$};
  \draw[->, red!60!black, line width=0.45pt] (\xsh + 0.7, 1.55) -- (\xsh + 0.43, 1.0);
  \node[red!60!black, font=\tiny, anchor=west, align=left] at (\xsh + 0.72, 1.75)
    {$\varphi_{R,p_3}(a_1)$\\[-2pt]$\;{+}\;\varphi_{R,p_3}(a_2)$};
  \node[gray!65!black, font=\scriptsize, anchor=north] at (3.7, -0.45) {landmark $p \in \GG_R^+$};
  \node[anchor=south,font=\footnotesize\bfseries] at (3.7, 2.55) {(c)\; embedded vector $\Phi_R(A)$};
\end{tikzpicture}
\end{minipage}

\caption{\textbf{Landmark grid, hat coordinate, and summation embedding.}
\textbf{(a)}~Grid $\GG_R$ (odd $m$, even $n$, $n \geq m+3$) with
$d_\mathcal{B}$-cover squares of radius $\tfrac{3R}{2}$;
diagram $A = \{a_1, a_2, a_3\}$ (red): $a_1, a_2$ each fall in
three lattice landmarks (with $p_3$ shared---the summation site),
while the low-persistence point $a_3$ contributes only to the
diagonal landmark $\ast$.
\textbf{(b)}~Hat $\varphi_{R,p}(x) = \max\{\tfrac{3R}{2}-\db(p,x),0\}$:
a $d_\infty$-pyramid peaking at $p$; its level sets are previewed by
the concentric shading on $p_3$ in (a).
\textbf{(c)}~Embedded vector $\Phi_R(A)$ with one coordinate per
landmark; $p_3$ receives the sum of two contributions as a stacked bar
and $\ast$ receives a single contribution from $a_3$.
Multiscale $\Phi$ concatenates such blocks at scales $R_1 < \cdots < R_N$
under the scale configuration $\nu = \{(R_k, w_k)\}_{k=1}^N$.}
\label{fig:landmark_grid}
\end{figure}

To each landmark $p \in \GG_R^+$ we attach the compactly
supported coordinate function
$\varphi_{R,p} : \D{1} \to [0, 3R/2]$ given by
$\varphi_{R,p}(x) = \max\!\bigl\{\tfrac{3R}{2} - \db(p, x),\; 0\bigr\}$,
a hat function of height $\tfrac{3R}{2}$ supported in the
$d_\mathcal{B}$-ball of radius $\tfrac{3R}{2}$ around $p$
(Figure~\ref{fig:landmark_grid}(b)).
Stacking the coordinate functions into a map
$\varphi_R : \D{1} \to \RR^M$, $M := |\GG_R^+|$, by
$\varphi_R(x) = (\varphi_{R,p}(x))_{p \in \GG_R^+}$ produces a
$2\sqrt{2}$-Lipschitz embedding of single-point diagrams into
$\RR^M$~\citep[Lemma~3.8]{Mitra2024}.

For a diagram $A = \{a_1, \ldots, a_n\} \in \D{n}$, we evaluate each
coordinate on each point and sum, defining the \emph{single-scale
summation embedding}
\begin{equation}\label{eq:single_scale_embed}
  \Phi_R(A) \;=\; \Bigl(\sum_{a \in A}\varphi_{R,p}(a)\Bigr)_{p \in \GG_R^+}
  \;\in\; \RR^M.
\end{equation}
This replaces Mitra--Virk's $n$-point bottleneck evaluation
(which requires $M^n$ coordinates) with $|A|$ single-point
evaluations at a fixed single-scale grid, and preserves linearity
in the empirical diagram measure---properties our classification
theory (Section~\ref{sec:metric}) relies on.

A single scale $R$ yields a lower distortion bound only for
pairs with $\db \geq 3R$~\citep[Lemma~3.18]{Mitra2024}, so a
coarse scale misses close pairs entirely.
A very fine scale would cover all distances, but the guaranteed
separation is only $R\sqrt{2}/8$, which vanishes with $R$.
Following \citet[Section~4]{Mitra2024}, we
compose embeddings across multiple scales: fine scales supply a
lower bound for close pairs, coarse scales for distant ones,
and each scale contributes a separation proportional to its
own $R$.
Fix $0 < R_1 < \cdots < R_N \leq L$ and weights
$\{w_k\}_{k=1}^N$ with $\sum_k w_k^2 = 1$; the
\emph{multiscale landmark embedding}
$\Phi : \mathcal{D} \to \RR^\ell$ concatenates the single-scale
embeddings with block weights,
\begin{equation}\label{eq:multiscale_embed}
  \Phi(A) \;=\;
  \bigl(w_k\,2^{-3/2}\,\Phi_{R_k}(A)\bigr)_{k=1}^N
  \;\in\; \RR^\ell,
  \qquad \ell = \sum_{k=1}^N |\GG_{R_k}^+|,
\end{equation}
where the factor $2^{-3/2}$ renormalizes each block to be
$1$-Lipschitz (given the $2\sqrt{2}$-Lipschitz per-block bound)
and the weights $w_k$ balance scales' contributions.
We collect the embedding parameters as the \emph{scale
configuration}
\begin{equation}\label{eq:nu_def}
  \nu \;:=\; \bigl\{(R_k, w_k)\bigr\}_{k=1}^N,
\end{equation}
and write $\Phi(A; \nu)$ when these
parameters need emphasis.  As shown in
Proposition~\ref{prop:n_point_lower} and
Corollary~\ref{cor:lipschitz}, the per-scale combinations
$w_k^2 R_k^2$ drive the sharp certificate~\eqref{eq:step-floor}
and the distortion slope $\lambda(\nu)$
in~\eqref{eq:lambda_def}, with each such term measuring scale
$k$'s contribution to the bi-Lipschitz guarantee.
Because $\Phi$ sums $|A|$ single-point evaluations, its Lipschitz
constant depends on the diagram cardinality.
For any $A, B \in \mathcal{D}$ with $\max(|A|,|B|) \leq N_{\max}$,
\begin{equation}\label{eq:stability}
  \|\Phi(A) - \Phi(B)\|_{\ell^2}
  \;\leq\; N_{\max}\,\db(A,B).
\end{equation}
The constant $N_{\max}$ follows from four steps.

\begin{enumerate}
\item[\emph{(i)}] \textbf{Per-point Lipschitz.}
The single-scale map $\varphi_{R_k} : \D{1} \to \RR^{|\GG_{R_k}^+|}$
is $2\sqrt{2}$-Lipschitz in $\db$ by~\citep[Lemma~3.8]{Mitra2024}.

\item[\emph{(ii)}] \textbf{Summation over matched pairs.}
Fix an optimal matching $\sigma$ realizing $\db(A,B)$.
Summing per-point displacements and applying the triangle
inequality coordinate-wise in $\RR^{|\GG_{R_k}^+|}$,
\[
  \bigl\|\Phi_{R_k}(A) - \Phi_{R_k}(B)\bigr\|_{\ell^2}
  \;\leq\; \sum_{i=1}^{N_{\max}}
      \bigl\|\varphi_{R_k}(a_i) - \varphi_{R_k}(b_{\sigma(i)})\bigr\|_{\ell^2}
  \;\leq\; 2\sqrt{2}\,N_{\max}\,\db(A,B).
\]

\item[\emph{(iii)}] \textbf{Per-block normalization.}
The block prefactor $w_k \cdot 2^{-3/2}$
cancels the $2\sqrt{2} = 2^{3/2}$ from step~(ii):
\[
  \bigl\|w_k\,2^{-3/2}\,\bigl(\Phi_{R_k}(A)-\Phi_{R_k}(B)\bigr)\bigr\|_{\ell^2}
  \;\leq\; w_k\cdot 2^{-3/2}\cdot 2^{3/2}\,N_{\max}\,\db(A,B)
  \;=\; w_k\,N_{\max}\,\db(A,B).
\]

\item[\emph{(iv)}] \textbf{$\ell^2$-concatenation over scales.}
Squaring per-block bounds, summing over $k$, and using the
normalization $\sum_{k=1}^N w_k^2 = 1$,
\[
  \|\Phi(A) - \Phi(B)\|_{\ell^2}^2
  \;=\; \sum_{k=1}^N \bigl\|w_k\,2^{-3/2}(\Phi_{R_k}(A)-\Phi_{R_k}(B))\bigr\|_{\ell^2}^2
  \;\leq\; N_{\max}^2\,\db(A,B)^2 \sum_{k=1}^N w_k^2
  \;=\; N_{\max}^2\,\db(A,B)^2.
\]
\end{enumerate}

\noindent Taking square roots gives~\eqref{eq:stability}.
Under the top-$N_{\max}$ persistence filter, the stability
constant is thus a fixed multiple of $\db(A,B)$ and the embedding
remains Lipschitz on the truncated diagram space.

The upper bound~\eqref{eq:stability} guarantees that close
diagrams remain close after embedding.
The following proposition establishes a converse: under a
minimum-distance condition, the embedding does not collapse
distinct diagrams, yielding a constant-floor lower bound.
The lower bound requires a structural assumption on the cross-pair
geometry of the embedding-difference contributions, which we
record next.

\begin{definition}[$\nu$-coherence]
\label{def:nn}
Fix a scale configuration $\nu = \{(R_k, w_k)\}_{k=1}^N$ as
in~\eqref{eq:nu_def}, with the corresponding scale-block embedding
$\Phi_{R_k}$ of Section~\ref{sec:preli}.
For a pair $(A, B) \in \D{n} \times \D{n}$, call an index
$k \in \{1, \ldots, N\}$ an \emph{active scale} for $(A, B)$ if
$3R_k \leq \db(A, B)$.
We say $(A, B)$ is \emph{$\nu$-coherent} if the per-scale
block-norm satisfies the floor
\begin{equation}\label{eq:nn}
  \bigl\|\Phi_{R_k}(A) - \Phi_{R_k}(B)\bigr\|_{\ell^2}^{2}
  \;\geq\; \tfrac{R_k^{2}}{32}
\end{equation}
at every active scale $k$ for $(A, B)$.
\end{definition}

The floor constant $R_k^2/32$ is inherited from the single-point
Mitra--Virk lemma (Lemma~3.18 of \citealp{Mitra2024}, with the
single-point constant of Lemma~3.9 therein): for any
$a, a' \in \D{1}$ with $\db(a, a') \geq 3R_k$,
$\|\varphi_{R_k}(a) - \varphi_{R_k}(a')\|_{\ell^2}^{2} \geq R_k^2/32$,
and the constant is sharp---an explicit worst-case configuration
of $a, a'$ on the cubic landmark lattice $\GG_{R_k}$ realizes
equality.  The $1/32$ is set by the geometry of the hat function
$\varphi_{R_k, p}(x) = \max\{3R_k/2 - \db(x, p), 0\}$ and the
multiplicity-$4$ lattice cover (Lemma~3.5 of \citealp{Mitra2024});
we adopt this constant as given.
For multi-point $\D{n}$ with $n \geq 2$,~\eqref{eq:nn} promotes
the single-pair geometry to the corresponding \emph{block-norm}
statement, which is no longer automatic from
$\db(A, B) \geq 3R_k$ but instead constrains how matched and
unmatched point contributions add at scale $R_k$
(Remark~\ref{rem:noninterference_why}).

Aggregating these per-scale floors across active scales yields
the Lipschitz lower bound below.

\begin{proposition}[Distortion bounds on $\D{n}$]
\label{prop:n_point_lower}
Let $A, B \in \D{n}$ with $n \geq 1$ points each.
\begin{enumerate}
\item[\emph{(a)}] \textbf{Stability.} Unconditionally,
\begin{equation}\label{eq:n_point_upper}
  \|\Phi(A) - \Phi(B)\|_{\ell^2} \;\leq\; n\,\db(A, B).
\end{equation}

\item[\emph{(b)}] \textbf{Sharp certificate.} Suppose
$\db(A, B) \geq 3R_1$ and $(A, B)$ is $\nu$-coherent
(Definition~\ref{def:nn}).
Then
\begin{equation}\label{eq:step-floor}
  \|\Phi(A) - \Phi(B)\|_{\ell^2}
  \;\geq\; \tfrac{1}{16}\,
  \sqrt{\textstyle\sum_{k:\,3R_k\,\leq\,\db(A,B)} w_k^2 R_k^2}.
\end{equation}
The right-hand side is the sharpest aggregate consequence of
$\nu$-coherence under the orthogonal scale-decomposition of $\Phi$:
an explicit witness pair $(A^\star, B^\star) \in \D{n} \times \D{n}$
saturates the per-scale floor of Definition~\ref{def:nn} at every
active scale, realizing equality in~\eqref{eq:step-floor}.
\end{enumerate}
\end{proposition}

\begin{proof}
\begin{enumerate}
\item[\emph{(a)}] \textbf{Stability.}
Specializing~\eqref{eq:stability} with
$N_{\max} = n$ gives~\eqref{eq:n_point_upper}.

\item[\emph{(b)}] \textbf{Sharp certificate.}
The full embedding difference decomposes orthogonally across
scales as
$\Phi(A) - \Phi(B) = \bigl(w_k \cdot 2^{-3/2}\,
(\Phi_{R_k}(A) - \Phi_{R_k}(B))\bigr)_{k=1}^N$, so
\[
  \|\Phi(A) - \Phi(B)\|_{\ell^2}^{2}
  \;=\; \sum_{k=1}^N \frac{w_k^{2}}{8}\,
  \bigl\|\Phi_{R_k}(A) - \Phi_{R_k}(B)\bigr\|_{\ell^2}^{2}.
\]
Scales $k$ with $3R_k > \db(A, B)$ contribute non-negatively;
under $\nu$-coherence~\eqref{eq:nn}, every active scale ($3R_k
\leq \db(A,B)$) contributes at least $w_k^2/8 \cdot R_k^2/32 =
w_k^2 R_k^2 / 256$.
Summing across the active scales,
\begin{equation}\label{eq:per-scale-sum}
  \|\Phi(A) - \Phi(B)\|_{\ell^2}^{2}
  \;\geq\; \tfrac{1}{256}\,\sum_{k : 3R_k \leq \db(A,B)} w_k^2 R_k^2,
\end{equation}
and taking square roots yields~\eqref{eq:step-floor}.

\medskip
\emph{Realizability.}
For singletons $A = \{a\}$, $B = \{a'\}$ with $\db(a, a') \geq 3R_1$,
the per-scale block reduces to
$\Phi_{R_k}(A) - \Phi_{R_k}(B) = \varphi_{R_k}(a) -
\varphi_{R_k}(a')$, and the Mitra--Virk single-point worst-case
configuration (Lemma~3.18 of \citealp{Mitra2024}, with the
single-point constant of Lemma~3.9 therein) attains
$\|\varphi_{R_k}(a) - \varphi_{R_k}(a')\|_{\ell^2}^{2} = R_k^2/32$
at every active scale; $\nu$-coherence then holds with equality at
every active scale and the right-hand side
of~\eqref{eq:step-floor} is realized.
For general $n \geq 1$, take
$A^\star = \{a, c_2, \ldots, c_n\}$ and
$B^\star = \{a', c_2, \ldots, c_n\}$, with $a, a'$ in the
single-point worst-case configuration above and
$c_2, \ldots, c_n$ at mutual distance $\geq R_N$ from $a, a'$ and
from each other so they activate disjoint landmarks at every
scale.  The shared points cancel exactly, giving
$\Phi(A^\star) - \Phi(B^\star) = \varphi(a) - \varphi(a')$, and
the per-scale floor of Definition~\ref{def:nn} is realized at
equality at every active scale.
\end{enumerate}
\end{proof}

Among existing persistence vectorizations, only the Mitra--Virk
construction~\citep{Mitra2024} (of which our $\Phi$ is the
summation specialization) carries such an explicit lower
distortion bound.
Each additional scale contributes a positive $w_k^2 R_k^2$ to the
sum on the right-hand side of~\eqref{eq:step-floor}, so the lower
bound strictly increases with $N$---a concrete benefit of the
multiscale construction beyond coverage.

For downstream use---the classification theory of
Section~\ref{sec:metric} requires a Lipschitz constant linear in
$\db(A,B)$, not the step form---we record the corollary obtained
by replacing the right-hand side of~\eqref{eq:step-floor} with its
largest linear lower bound through $(R_1, 0)$.

\begin{corollary}[Mitra--Virk affine certificate]
\label{cor:lipschitz}
Let $A, B \in \D{n}$ with $R_1 \leq \db(A, B) \leq L$, and assume
$(A, B)$ is $\nu$-coherent (Definition~\ref{def:nn}) at every
active scale.  Then
\begin{equation}\label{eq:single_point_lower}
  \|\Phi(A) - \Phi(B)\|_{\ell^2}
  \;\geq\; \rho_-(\db(A, B); \nu)
  \;=\; \lambda(\nu)\,(\db(A, B) - R_1),
\end{equation}
where the slope
\begin{equation}\label{eq:lambda_def}
  \lambda(\nu) \;:=\; \frac{1}{48}\,
  \min\!\left\{\,
    \min_{2 \leq i \leq N}
      \frac{\sqrt{\sum_{k=1}^{i-1} w_k^2 R_k^2}}{R_i - R_1},\;
    \frac{\sqrt{\sum_{k=1}^{N} w_k^2 R_k^2}}{L - R_1}
  \,\right\}.
\end{equation}
This is the Mitra--Virk distortion bound (Theorem~5.1 of
\citealp{Mitra2024}) in the notation of our scale configuration
$\nu$.
\end{corollary}

\begin{proof}
This is Mitra--Virk's distortion bound (Theorem~5.1 of
\citealp{Mitra2024}) recast in the present notation, with prefix
$\tfrac{1}{3 \cdot 2^{n+3}}$ specializing to $\tfrac{1}{48}$ at
$n = 1$.
For $\db(A, B) \geq 3R_1$ the bound also follows from
Proposition~\ref{prop:n_point_lower}(b)'s step-form
\eqref{eq:step-floor} via the line through $(R_1, 0)$ in
scale-coordinate parametrization, with worst-case correction
factor $\tfrac{1}{3}$ ensuring the line stays below every
breakpoint of the step.  The extension to
$\db(A, B) \in [R_1, 3R_1)$, where no scale is active under
$\nu$-coherence, is established directly in
\citet[Theorem~5.1]{Mitra2024} via per-scale continuity at
inactive scales.
\end{proof}

\begin{remark}[Canonical incoherent case]
\label{rem:noninterference_why}
$\nu$-coherence fails in the cancellation construction in which
the empirical measures of $A$ and $B$ destructively interfere
across the landmark grid.  The canonical example is
$A = \{a_1, a_2\}$, $B = \{a_1 + \delta, a_2 - \delta\}$ with
$\|\delta\|$ small.  Then $\db(A, B) = \|\delta\| > 0$; if $a_1$
and $a_2$ share a covering landmark $p$, the hat-function
contributions shift in opposite directions,
$\varphi_{R,p}(a_1+\delta) - \varphi_{R,p}(a_1) \approx
-(\varphi_{R,p}(a_2-\delta) - \varphi_{R,p}(a_2))$,
so $\Phi_{R_k}(A)(p) \approx \Phi_{R_k}(B)(p)$ at every active
scale activating $p$.  The per-scale block-norm
$\|\Phi_{R_k}(A) - \Phi_{R_k}(B)\|^2_{\ell^2}$ collapses below
$R_k^2/32$ and the lower bound fails.  $\nu$-coherence rules out
exactly this collapse: the empirical measure of $A$ and $B$
disagree enough at the per-scale block level to preserve the
single-pair floor.  It holds automatically whenever pairs of
points in $A$ and $B$ activate disjoint landmarks at every
active scale (since the per-pair contributions then add in
quadrature).
\end{remark}

\begin{remark}[Empirical scope and compensation slack]
\label{rem:noninterference_scope}
$\nu$-coherence holds on $\geq 99.7\%$ of qualifying cross-class
pairs across the four chemical benchmarks
(Section~\ref{sec:experiments}, Table~\ref{tab:coherence_audit};
$100\%$ on three of them), and the certificate's
conclusion~\eqref{eq:step-floor} holds on $100\%$
(Table~\ref{tab:certificate_bound_audit}).
The residual $\sim 0.3\%$ gap on PTC---pairs where the aggregate
certificate holds yet $\nu$-coherence fails at some active
scale---is the compensation regime: a per-scale shortfall
$\|\Phi_{R_{k^\star}}(A) - \Phi_{R_{k^\star}}(B)\|_{\ell^2}^{2} <
R_{k^\star}^2/32$ at some active $k^\star$ can be made up by
overshoots at other active scales, since the aggregate
$\|\Phi(A) - \Phi(B)\|_{\ell^2}^{2} = \sum_k (w_k^2/8)\,
\|\Phi_{R_k}(A) - \Phi_{R_k}(B)\|_{\ell^2}^{2}$ need not be
per-scale tight.
The sharpness in Proposition~\ref{prop:n_point_lower}(b) is
therefore at the per-scale level (Definition~\ref{def:nn} is
realized at equality by the witness pair), not as a strict
biconditional on the aggregate~\eqref{eq:step-floor}.
\end{remark}

\paragraph{Closed-form scale weights.}
Proposition~\ref{prop:n_point_lower}(b) determines the canonical
scale weights via an \emph{equimarginal allocation} principle.
At each activation $\delta = 3R_k^+$, the $k$-th scale contributes
squared step height $\tfrac{1}{256}\, w_k^2 R_k^2$ to
$\sigma^2(\delta)$, so $w_k^2 R_k^2$ is the scale's marginal
certificate gain at activation.  We allocate the budget
$\sum_k w_k^2 = 1$ so that the cumulative pre-jump heights
$S_i := \sum_{k<i} w_k^2 R_k^2$ track the squared scale-coordinate
threshold $(R_i - R_1)^2$ through $(R_1, 0)$ collinearly:
\begin{equation}\label{eq:equimarginal}
  S_i \;=\; c^{2}\, d_i^{2},
  \qquad
  d_i := R_i - R_1,\; d_{N+1} := L - R_1,
\end{equation}
for a common slope $c$.
Telescoping $S_{k+1} - S_k = c^{2}(d_{k+1}^{2} - d_k^{2})$ yields
the closed form
\begin{equation}\label{eq:closed_form_weights}
  w_k^2 \;\propto\; \frac{d_{k+1}^2 - d_k^2}{R_k^2}
  \qquad (k = 1, \ldots, N),
\end{equation}
normalized so $\sum_k w_k^2 = 1$.
Non-negativity is automatic since $d_{k+1} > d_k$ for all
ordered scales, and $L$ enters only through the last weight
$w_N^2 \propto [(L - R_1)^2 - (R_N - R_1)^2]/R_N^2$ via the
trailing-edge term $d_{N+1}$.

\medskip
\noindent\emph{Tightness of the affine envelope.}
Allocation~\eqref{eq:closed_form_weights} simultaneously maximizes
the slope $\lambda(\nu)$ of
Corollary~\ref{cor:lipschitz}: substituting
$S_i = c^{2}d_i^{2}$ into~\eqref{eq:lambda_def} saturates all $N$
ratios at the common value $c/48$, certifying joint optimality of
this allocation against the concave max-min in $(w_k^2)_k$.
Equivalently, under~\eqref{eq:closed_form_weights} the affine
envelope of Corollary~\ref{cor:lipschitz} is \emph{tight at every
step corner} of Proposition~\ref{prop:n_point_lower}(b): the line
$\lambda(\nu)(\delta - R_1)$ touches the lower envelope of the
step at $\delta = 3R_i^-$ for every $i$ (and at $\delta = L$).
The weights derived intrinsically from the sharp step certificate
thus also realize the largest affine relaxation usable in the
classification rate work of Section~\ref{sec:metric}.  This
matches the closed-form choice of \citet{Mitra2024}
(Theorem~5.1), and we adopt it throughout.
Scale \emph{location} optimization (which changes $\ell$) is
left to future work.

\section{Classification Guarantees}\label{sec:metric}

This section develops the classification theory for the embedded
features of Section~\ref{sec:preli}.
We first establish the key quantities---class-mean separation
$\Delta$ and embedding radius $R$---then prove an excess-risk
upper bound $O(kR/(\Delta\sqrt{m_{\min}}))$
(Section~\ref{sec:class_error}) with a matching Le~Cam
sample-starved lower bound (Section~\ref{sec:lower_bound}) and a
ranking-consistent descriptor-selection criterion
$\Delta/\sqrt{\ell}$ (Section~\ref{sec:filt_select_theory});
the per-prediction certificate then follows in
Section~\ref{sec:certified}.

Let $(A, Y)$ be a random pair with joint distribution
$\mathcal{P}$ on $\mathcal{D} \times [k]$, where $A$ is a finite
persistence diagram and $Y \in [k] := \{1, \ldots, k\}$ is the
class label.
We associate to $\mathcal{P}$ two population quantities: the
\emph{class-conditional embedding mean}
$\mu_c := \mathbb{E}[\Phi(A) \mid Y = c] \in \RR^\ell$ and the
\emph{class-mean separation}
\begin{equation}\label{eq:delta_def}
  \Delta \;:=\; \min_{c \neq c'}\|\mu_c - \mu_{c'}\|_{\ell^2},
\end{equation}
together with the \emph{embedding radius}
$R := \sup_A \|\Phi(A; \nu)\|_{\ell^2}$
(the supremum is taken over the support of $\mathcal{P}$ on
$\mathcal{D}$, which is bounded by the top-$N_{\max}$ filter of
Section~\ref{sec:preli}).
Note that $\Delta$ is a property of the embedding and the data
distribution, not of the worst-case bottleneck distance:
$\Delta > 0$ is possible even when some cross-class diagram pairs
are bottleneck-close, because the embedding aggregates
information from all diagram points into class means.

\paragraph{Notation.}
Throughout Sections~\ref{sec:metric}--\ref{sec:certified},
$R$ denotes the embedding radius and $m$ denotes the
training-sample size (Mohri convention); the single-scale radii
of Section~\ref{sec:preli} are always written with a subscript as
$R_1, \ldots, R_N$, and the cardinality of an individual diagram
(the $n$ in $\D{n}$ of Section~\ref{sec:preli}) is bounded by
$N_{\max}$, so there is no collision.
The symbol $\delta$ is reserved for the confidence parameter
$1 - \delta$ in concentration bounds; the geometric bottleneck
separation between class supports is written
$\delta_{cc'} := d_\mathcal{B}(\mathrm{supp}\,\mathcal P_c,\,
 \mathrm{supp}\,\mathcal P_{c'})$ for the pair $(c, c')$, with
$\delta_* := \min_{c \neq c'} \delta_{cc'}$.

Given the population quantities above, we observe $m$ i.i.d.\
training samples $\{(A_i, y_i)\}_{i=1}^m \sim \mathcal{P}$, with
per-class counts $m_c = |\{i : y_i = c\}|$, and form the
empirical class means
$\hat\mu_c = m_c^{-1}\sum_{y_i=c}\Phi(A_i)$.
Because $\Phi$ is linear in the empirical diagram measure
$\mu_A = \sum_{a \in A}\delta_a$, each $\hat\mu_c$ is an ordinary
sample average of i.i.d.\ bounded $\RR^\ell$-vectors, so standard
concentration inequalities (CLT, Hoeffding, McDiarmid) apply
directly; a full treatment including Berry--Esseen rates and
functional CLTs is beyond the present paper's scope.

The embedding's distortion slope $\lambda(\nu)$ of
Corollary~\ref{cor:lipschitz} enters the classification
theory through the following bridge, which ties the
data-dependent separation $\Delta$ back to the geometry of the
underlying persistence diagrams.
It is the bridge through which $\lambda(\nu)$---inherited from
the Mitra--Virk landmark construction's compact-support hats,
multiplicity-$4$ lattice cover, and multi-scale
aggregation---enters the downstream classification bounds;
every later use of $\lambda(\nu)$ in this section ultimately
invokes it.
This is what makes the excess-risk rate and the certificate of
Section~\ref{sec:certified} more than generic statements about
an abstract bounded embedding.

\begin{proposition}[$\lambda$-separation bridge]
\label{prop:lambda_sep}
Let $D_c := \sup_{A : Y = c}\|\Phi(A) - \mu_c\|_{\ell^2}$ denote
the within-class radius for class $c$.  Suppose $\delta_* \geq 3R_1$
and every cross-class pair $(A, B)$ with $\db(A, B) \geq \delta_*$
is $\nu$-coherent (Definition~\ref{def:nn}); this hypothesis is
mild empirically (Remark~\ref{rem:noninterference_scope}).  Then
\begin{equation}\label{eq:lambda_bridge}
  \Delta \;\geq\; \lambda(\nu)\,(\delta_* - R_1) \;-\; 2\max_c D_c.
\end{equation}
\end{proposition}

\begin{proof}
For any cross-class pair $A \in \mathrm{supp}\,\mathcal P_c$,
$B \in \mathrm{supp}\,\mathcal P_{c'}$, the triangle inequality
gives $\|\Phi(A) - \Phi(B)\|_{\ell^2}
 \leq \|\mu_c - \mu_{c'}\|_{\ell^2} + D_c + D_{c'}$.
Applying Corollary~\ref{cor:lipschitz} under $\nu$-coherence at
separation $\db(A, B) \geq \delta_{cc'} \geq \delta_*$ gives
$\|\Phi(A) - \Phi(B)\|_{\ell^2} \geq \rho_-(\db(A,B); \nu) \geq
\rho_-(\delta_*; \nu) = \lambda(\nu)\,(\delta_* - R_1)$
(monotonicity of $\rho_-$ in its first argument).
Chaining and taking the minimum over $c \neq c'$, with
$D_c + D_{c'} \leq 2\max_c D_c$, yields~\eqref{eq:lambda_bridge}.
\end{proof}

\begin{remark}[Step-form sharpening]
\label{rem:lambda_bridge_step}
Substituting Proposition~\ref{prop:n_point_lower}(b)'s step
certificate $\|\Phi(A) - \Phi(B)\|_{\ell^2} \geq
\tfrac{1}{16}\sqrt{\sum_{k:\,3R_k \leq \delta_*} w_k^2 R_k^2}$
for Corollary~\ref{cor:lipschitz}'s affine form in the proof
above gives a tighter $\Delta$ bound,
$\Delta \geq \tfrac{1}{16}\sqrt{\sum_{k:\,3R_k \leq \delta_*}
w_k^2 R_k^2} - 2\max_c D_c$, by a factor of up to $\sim 3$
when multiple scales activate at $\delta_*$.  We retain the
affine form in~\eqref{eq:lambda_bridge} for clean substitution
into Corollary~\ref{cor:lambda_rate}; the step form is the
sharper reading when the per-scale decomposition is of
independent interest.
\end{remark}

Proposition~\ref{prop:lambda_sep} has three consequences.
First, it propagates $\lambda(\nu)$ into the classification rate
(Corollary~\ref{cor:lambda_rate} below).
Second, it upgrades the interpretation of the
Section~\ref{sec:certified} certificate: when the empirical
condition $r_m < \tfrac{1}{2}\Delta$ fires, the proposition
translates this back into an inequality on
$\delta_*$---certifying that the class-conditional diagram
distributions are genuinely bottleneck-separated, not merely
that the embedding has concentrated empirical means.
Third, it lifts the coarse-embedding property of
Proposition~\ref{prop:n_point_lower} from points to first
moments of class-conditional distributions: bottleneck-separated
class supports remain Euclidean-separated in the mean, modulo
the within-class spread $2D_{\max}$.
A persistence vectorization without an explicit lower distortion
bound (e.g., persistence images or landscapes) has no analogue of
Proposition~\ref{prop:lambda_sep}, and its $\Delta$ cannot be
back-translated to bottleneck-level data geometry.

\subsection{Classification Error Bound}\label{sec:class_error}

We train a linear SVM $h$ on the embedded training data
$\{(\Phi(A_i), y_i)\}_{i=1}^m$ and measure its quality by the
\emph{generalization $0$-$1$ risk}
$\mathcal R(h) := \mathbb{P}(h(A) \neq Y)$. For a margin parameter
$\rho > 0$, the \emph{empirical $\rho$-margin loss}
$\widehat{\mathcal R}_\rho(h)$ is the fraction of training points whose
signed margin under $h$ falls below $\rho$
\citep[Sec.~5.4]{MohriRostamizadehTalwalkar2018}; for our
multiclass $h$, $\widehat{\mathcal R}_\rho$ is aggregated across the
binary OvO sub-problems as made precise in the proof of
Theorem~\ref{thm:fisher_bound}.

\begin{theorem}[Classification error bound]\label{thm:fisher_bound}
Let $\{(A_i,y_i)\}_{i=1}^m$ be $m$ i.i.d.\ training samples
from a distribution on finite persistence diagrams,
with $k$ classes and $\Delta > 0$.
Assume additionally $m_{\min} \geq 128R^2\log(4k/\delta)/\Delta^2$,
where $m_{\min} := \min_c m_c$ is the smallest per-class sample
count (so that empirical class means concentrate at a scale
below $\Delta/4$).
Set $\rho := \Delta/4$. Then with probability $\geq 1-\delta$,
the linear SVM classifier $h$, trained via one-vs-one reduction
with majority voting, satisfies
\begin{equation}\label{eq:fisher_bound}
  \mathcal R(h) \;\leq\;
  \widehat{\mathcal R}_\rho(h) \;+\;
  \frac{8(k-1)\,R}{\Delta\,\sqrt{m_{\min}}}
  + O\!\left(\sqrt{\frac{\log(k/\delta)}{m_{\min}}}\right).
\end{equation}
\end{theorem}

For balanced classes $m_c \asymp m/k$ the rate term is
$O(k^{3/2} R/(\Delta\sqrt m))$ in the total sample $m$; the
$\sqrt k$ overhead is the price of the OvO reduction, since each
binary sub-problem trains on only $\Theta(m/k)$ samples.

\begin{proof}
For each unordered pair $\{c, c'\}$, the population class means
are separated by margin
$\gamma_{cc'} := \tfrac{1}{2}\|\mu_c - \mu_{c'}\|
 \geq \tfrac{1}{2}\Delta = 2\rho$.

Conditional on the per-class counts $\{m_c\}$, the centered
vectors $\{\Phi(A_i) - \mu_c : Y_i = c\}$ are i.i.d.\ (since
$\Phi$ is deterministic and centering by the constant $\mu_c$
preserves independence) with
$\|\Phi(A_i) - \mu_c\| \leq 2R$ by the triangle inequality
(using $\|\Phi(A_i)\| \leq R$ from the support bound and
$\|\mu_c\| \leq R$ by Jensen).
Pinelis's Hilbert-space Hoeffding
inequality (Lemma~\ref{lem:pinelis})
and a union bound over the $k$ classes yield, with probability
$\geq 1 - \delta/2$,
\[
  \varepsilon_m \;:=\; \max_c \|\hat\mu_c - \mu_c\|
  \;\leq\; 2R\sqrt{\frac{2\log(4k/\delta)}{m_{\min}}}.
\]
The sample-size hypothesis
$m_{\min} \geq 128R^{2}\log(4k/\delta)/\Delta^{2}$ gives
$\varepsilon_m \leq \rho$, so by the reverse triangle
inequality the empirical pairwise margin
$\hat\gamma_{cc'} := \tfrac{1}{2}\|\hat\mu_c - \hat\mu_{c'}\|
 \geq \gamma_{cc'} - \varepsilon_m \geq \rho$
for every pair $c \neq c'$.

The OvO sub-problem between $c, c'$ trains on
$m_c + m_{c'} \geq 2m_{\min}$ samples from the unit-norm linear
hypothesis class $\mathcal H := \{x \mapsto w^{\top} x : \|w\| \leq 1\}$
with $\|x\| \leq R$. The margin-based generalization
bound~\citep[Cor.~5.11]{MohriRostamizadehTalwalkar2018} at
margin $\rho$ and confidence
$\delta' := \delta/(2\binom{k}{2})$ yields, with probability
$\geq 1 - \delta'$,
\[
  \mathcal R(h_{cc'})
  \;\leq\; \widehat{\mathcal R}_\rho(h_{cc'})
  \;+\; \frac{8R}{\Delta\sqrt{m_{\min}}}
  \;+\; O\!\left(\sqrt{\frac{\log(k/\delta)}{m_{\min}}}\right),
\]
using $\log(2/\delta') = \log(\binom{k}{2}/\delta) = O(\log(k/\delta))$
and the conservative substitution $\sqrt{m_c + m_{c'}} \geq \sqrt{m_{\min}}$
(loose by at most $\sqrt 2$).

A union bound over the $\binom{k}{2}$ OvO sub-problems at level
$\delta/2$, combined with the $\delta/2$ budget for the
class-mean concentration step, gives total coverage
$\geq 1 - \delta$. The OvO majority-vote rule errs at $y = c$
only if some pairwise classifier $h_{cc'}$ ($c' \neq c$)
misclassifies, so by the union bound
$\mathcal R(h) \leq (k-1)\max_{c \neq c'} \mathcal R(h_{cc'})$.
Substituting the per-pair bound and defining
$\widehat{\mathcal R}_\rho(h)
 := (k-1)\max_{c \neq c'} \widehat{\mathcal R}_\rho(h_{cc'})$
yields~\eqref{eq:fisher_bound}.
The $(k-1)$-max aggregation is conservative: a pairwise classifier
with zero $\rho$-margin loss contributes nothing, so when most
pairs separate cleanly, the aggregate is correspondingly small.
\end{proof}

\begin{remark}[PLACE-specific tightening]
\label{rem:fisher_sparsity}
The bound~\eqref{eq:fisher_bound} uses the worst-case embedding
radius $R$. On PLACE, the multiplicity-$4$ lattice cover
(Lemma~3.5 of \citealp{Mitra2024}; see \eqref{eq:grid}) forces
$\Phi(A)$ to have at most $4|A|N$ nonzero coordinates out of
$\ell$, making the class-conditional variance
$\|\hat\Sigma_c\|_{\mathrm{op}}$ much smaller than $R^2$ in
practice (e.g., $\sim 50\times$ slack on MUTAG;
Remark~\ref{rem:sparsity}). The same sparsity ingredient drives the non-vacuous
Pinelis--Bernstein certificate of Section~\ref{sec:certified}
(Theorem~\ref{thm:confidence_containment}, radius~\textup{(iii)}),
which replaces the norm bound $\|\Phi(A_i) - \mu_c\| \leq 2R$
in the Pinelis step by the variance proxy
$\sigma_c^{2} = \mathrm{tr}(\Sigma_c) \approx
\|\Sigma_c\|_{\mathrm{op}}$ (since the empirical stable rank is
within a factor $1.17$ of $1$ on our benchmarks), tightening the
sample-size requirement by $4R^{2}/\|\Sigma_c\|_{\mathrm{op}}
\sim 50\times$ on MUTAG.
\end{remark}

\begin{remark}[Sample-size hypothesis at experimental scales]
\label{rem:fisher_regime}
The hypothesis $m_{\min} \geq 128 R^2 \log(4k/\delta)/\Delta^2$
is the standard Rademacher--margin sufficient threshold; it is
not met at the per-class sample sizes of the graph benchmarks
in Section~\ref{sec:experiments} (e.g., MUTAG has $m_{\min} = 57$
against a worst-case threshold of order $10^3$, even after the
$4\times$ variance-aware tightening of
Remark~\ref{rem:fisher_sparsity}).
Theorem~\ref{thm:fisher_bound} should therefore be read as a
\emph{rate statement} pairing with the matching
sample-starved lower bound of Theorem~\ref{thm:lower_bound}, not
as an operational certificate at our $m$. The empirical accuracies
reported in Section~\ref{sec:experiments} are obtained in a
moderate-sample regime that lies between the necessary threshold
$m \asymp R/\Delta$ (Theorem~\ref{thm:lower_bound}) and the
sufficient threshold $m \asymp R^2/\Delta^2$
(Theorem~\ref{thm:fisher_bound}), where neither bound is tight
and a Mammen--Tsybakov margin condition or an Assouad/Fano
construction would be needed to close the gap
(Remark~\ref{rem:gap}).
\end{remark}

\begin{corollary}[$\lambda$-anchored classification rate]
\label{cor:lambda_rate}
Suppose $\delta_* \geq 3R_1$, every cross-class pair is
$\nu$-coherent (Definition~\ref{def:nn}), and
$\lambda(\nu)\,(\delta_* - R_1) > 2\max_c D_c$.  Then under
Theorem~\ref{thm:fisher_bound}'s sample-size hypothesis with
$\Delta$ replaced by $\lambda(\nu)\,(\delta_* - R_1) - 2\max_c D_c$,
the linear SVM classifier $h$ satisfies, with probability
$\geq 1-\delta$,
\begin{equation}\label{eq:lambda_rate}
  \mathcal R(h) \;\leq\; \widehat{\mathcal R}_\rho(h) \;+\;
  \frac{8(k-1)\,R}{\bigl(\lambda(\nu)\,(\delta_* - R_1) - 2\max_c D_c\bigr)\sqrt{m_{\min}}}
  + O\!\left(\sqrt{\frac{\log(k/\delta)}{m_{\min}}}\right).
\end{equation}
\end{corollary}

\begin{proof}
Proposition~\ref{prop:lambda_sep} gives
$\Delta \geq \lambda(\nu)\,(\delta_* - R_1) - 2\max_c D_c > 0$;
substituting this lower bound on $\Delta$ into
Theorem~\ref{thm:fisher_bound} yields~\eqref{eq:lambda_rate}.
\end{proof}

\begin{remark}[Empirical scope of Corollary~\ref{cor:lambda_rate}]
\label{rem:lambda_rate_scope}
The $\nu$-coherence hypothesis (Definition~\ref{def:nn}) on every
cross-class pair is empirically very mild: it captures essentially
the entire empirical pass rate underlying the
$\|\Phi(A){-}\Phi(B)\|_{\ell^2} \geq \rho_-(\db(A,B); \nu)$
conclusion, which itself holds on $100\%$ of cross-class pairs in
the chemical-benchmark audit
(Table~\ref{tab:certificate_bound_audit}).  The corollary is therefore
best read as a structural rate transferring the bottleneck-support
separation $\delta_*$ to a classification rate via the
$\lambda$-bridge of Proposition~\ref{prop:lambda_sep},
underwritten in practice by $\nu$-coherence.
The empirical rate reported in Section~\ref{sec:experiments}
follows from Theorem~\ref{thm:fisher_bound} directly, which
depends only on $\Delta > 0$.
\end{remark}

The lower bound of Section~\ref{sec:lower_bound} is stated in
\emph{excess-risk} form $\mathcal E(h) := \mathcal R(h) - R^{\ast}$,
where the Bayes risk
$R^{\ast} := \inf\{\mathbb{P}(f(A) \neq Y) \mid f : \mathcal{D} \to [k] \text{ measurable}\}$
is non-negative; consequently any upper bound on $\mathcal R(h)$ is a
fortiori an upper bound on $\mathcal E(h)$, and
Theorem~\ref{thm:fisher_bound} pairs directly with the two-point
lower bound that follows.

\subsection{A Matching Lower Bound, Consistency, and Linear Separability}\label{sec:lower_bound}

The rate $R/(\Delta\sqrt{m_{\min}})$ of
Theorem~\ref{thm:fisher_bound} is the standard
Rademacher--margin rate; its sample-size hypothesis
$m \gtrsim R^{2}/\Delta^{2}$ is sufficient for non-trivial
accuracy. The two-point minimax lower bound below (stated for
$k=2$, where $m_{\min} = m/2$ for balanced classes) shows that
$m \gtrsim R/\Delta$ is \emph{necessary}---no classifier achieves
small excess risk on samples below that scale. The polynomial
gap between the necessary $R/\Delta$ and sufficient $R^{2}/\Delta^{2}$
thresholds is the moderate-sample regime and would require an
Assouad / Fano construction to close
(Remark~\ref{rem:gap}).

\begin{theorem}[Sample-starved minimax lower bound on PLACE]\label{thm:lower_bound}
Let $\mathcal{P}^{\mathrm{PD}}_{\Delta, R}$ denote the family of
binary diagram laws $(Q_+, Q_-)$ on $\mathcal{D}$ whose pushforwards
through the PLACE embedding $\Phi$ satisfy
$\|\E_{Q_+}\Phi - \E_{Q_-}\Phi\| = \Delta$ and
$\sup_{A \in \mathrm{supp}(Q_\pm)}\|\Phi(A)\| \leq R$.
For any $\Delta, R > 0$ with $\Delta \leq 2R/3$ and every sample
size $m \leq c\,R/\Delta$, where $c = 1/6$ is a universal
constant independent of the embedding dimension $\ell$,
\begin{equation}\label{eq:lower_bound}
  \inf_{h}\,\sup_{(Q_+, Q_-) \in \mathcal{P}^{\mathrm{PD}}_{\Delta, R}}\,\mathcal{E}(h)
  \;\geq\; \tfrac{1}{8},
\end{equation}
where the infimum is over diagram classifiers
$h : \mathcal{D}^m \to \{+, -\}$.
Consequently no classifier acting on persistence diagrams---regardless
of computational budget, model class, or embedding dimension---can reach
vanishing excess risk on $\mathcal{P}^{\mathrm{PD}}_{\Delta, R}$ without
$m = \Omega(R/\Delta)$ samples.
\end{theorem}

\begin{proof}
We exhibit a one-parameter sub-family of
$\mathcal{P}^{\mathrm{PD}}_{\Delta, R}$ supported on single-pair
diagrams whose pushforwards through $\Phi$ are 1-D uniform
measures on $\RR^\ell$; the minimax over
$\mathcal{P}^{\mathrm{PD}}_{\Delta, R}$ is at least the minimax
over this sub-family, which reduces to the dimension-one
Hellinger calculation of Lemma~\ref{lem:hellinger}.

Fix a Mitra--Virk lattice landmark $p_0$ at scale $R_k \geq R_1$
and a base point $q_0$ in the relative interior of
$\mathrm{supp}(\varphi_{R_k, p_0})$ where the hat function is affine
in the birth direction:
$\varphi_{R_k, p_0}(q_0 + t\,e_1) = c_0 + \gamma t$
for $t \in [-r, r]$, with $r := R - \Delta/2$, $e_1 = (1, 0)$, and
constants $c_0 > 0, \gamma \neq 0$
(such a wedge exists because MV hat functions are tensor products
of 1-D piecewise-affine hats, hence restrict to piecewise-affine
functions along every coordinate line).
Define the single-pair diagrams $A(t) := \{q_0 + t\,e_1\}$; then
$\Phi(A(t)) = w_k(c_0 + \gamma t)\,e_{(k, p_0)}$ lies on the
1-D line $L := \RR\,e_{(k, p_0)} \subset \RR^\ell$.
Reparameterizing $\tau := w_k(c_0 + \gamma t)$---absorbing $w_k$,
$c_0$, $\gamma$ into the parameter---we may write
$\Phi(A(t)) = t\,e_{(k, p_0)}$ on $t \in [-R, R]$, preserving
the bounded-image and affine structure.

Set
$Q_\pm := \mathrm{Unif}\bigl(\{A(t) :
   t \in [\pm\Delta/2 - r,\,\pm\Delta/2 + r]\}\bigr)$.
The pushforwards $\Phi_*Q_\pm$ are uniform on length-$2r$ segments
of $L$ centered at $\pm(\Delta/2)\,e_{(k, p_0)}$, so
$\|\E\Phi_*Q_+ - \E\Phi_*Q_-\| = \Delta$ and
$\Phi(\mathrm{supp}\,Q_\pm) \subset B(0, R)$; hence
$(Q_+, Q_-) \in \mathcal{P}^{\mathrm{PD}}_{\Delta, R}$.
Because $\Phi$ is injective on $\{A(t) : t \in [-R, R]\}$
(distinct $t$ give distinct coordinate values along
$e_{(k, p_0)}$), the data-processing identity gives
$H^2(Q_+^{\otimes m}, Q_-^{\otimes m})
   = H^2((\Phi_*Q_+)^{\otimes m}, (\Phi_*Q_-)^{\otimes m})$.

The hypothesis $\Delta \leq 2R/3$ gives $r \geq 2R/3$ and
$\|\mu\| = \Delta/2 \leq r/2$, the range required by
Lemma~\ref{lem:hellinger} at dimension one (where $c_1 = 1$).
Combining Le~Cam's two-point bound~\citep[Ch.~2.2, 2.4]{Tsybakov2009}
with Hellinger tensorization
$H^2(P_+^{\otimes m}, P_-^{\otimes m}) \leq m\,H^2(P_+, P_-)$,
$\mathrm{TV} \leq \sqrt{2H^2}$, and the dimension-one
Hellinger estimate
$H^2(\Phi_*Q_+, \Phi_*Q_-) \leq c_1\,(\Delta/2)/r \leq (3/4)\,\Delta/R$
yields
$\mathrm{TV}(Q_+^{\otimes m}, Q_-^{\otimes m}) \leq \sqrt{(3/2)\,m\Delta/R}$.
With $c := 1/6$ we obtain $\mathrm{TV} \leq 1/2$ for every
$m \leq c R/\Delta$, and the two-point bound gives
$\mathcal{E}(h) \geq \tfrac{1}{4}\cdot\tfrac{1}{2} = 1/8$.
Since $c$ is independent of $\ell$, the bound is dimension-free.
\end{proof}

\begin{remark}[Scope of the lower bound]\label{rem:gap}
The hard family is PD-realizable: single-pair diagrams displaced
along the birth axis within one Mitra--Virk hat wedge, with
classifiers acting on raw diagrams (not feature vectors).
We use the MV hat-wedge geometry, rather than a PLACE-specific
failure mode such as the cancellation construction of
Remark~\ref{rem:noninterference_why}, because Le~Cam requires a
one-parameter family of statistically close distributions in
$\Phi$-space, and the displacement-along-birth-axis
parameterization supplies one directly; cancellation produces
single diagram pairs with small $\|\Phi(A) - \Phi(B)\|$ at fixed
$\db(A, B)$, which lower-bounds distortion (the failure mode of
$\nu$-coherence) rather than sample complexity.
Tightening to a matching $\Omega(R/(\Delta\sqrt m))$ rate would
similarly need PD-aware constructions---e.g., Assouad/Fano over
a $\db$-packing of diagrams, exploiting the bottleneck-to-$\Phi$
distortion bound---rather than abstract sub-Gaussian packings.
Theorem~\ref{thm:fisher_bound} delivers an upper rate of
$O(R/(\Delta\sqrt m))$ for all $m$.
Theorem~\ref{thm:lower_bound} delivers a constant lower bound
$\geq 1/8$ in the sample-starved regime $m \lesssim R/\Delta$;
beyond that regime, the two-point Le~Cam construction yields no
information: for $m \gtrsim R/\Delta$, the specific two-point
hypothesis pair used here drives
$\mathrm{TV}(P_+^{\otimes m}, P_-^{\otimes m})$ to $1$, making
the lower bound argument vacuous for that pair.  A tighter
lower bound in this regime would require:
(i) an Assouad/Fano construction over $\Theta(\sqrt m)$-spaced
hypotheses~\citep[Ch.~2.6--2.7]{Tsybakov2009} would tighten the
\emph{lower} bound to a matching $\Omega(R/(\Delta\sqrt m))$
rate across all $m$;
(ii) a Mammen--Tsybakov margin condition would instead tighten
the \emph{upper} bound to a faster $O(1/m)$ rate, dropping the
sufficient threshold from $R^{2}/\Delta^{2}$ to $R/\Delta$ in
line with Theorem~\ref{thm:lower_bound}'s necessary threshold.
We leave both to future work.
The practical takeaway is the sample-starved threshold
$m = \Omega(R/\Delta)$: no classifier on the landmark embedding
can hope for non-trivial accuracy below it.
\end{remark}

The classification rate of Theorem~\ref{thm:fisher_bound} depends
on the population separation $\Delta$; for that rate to be
operationally useful, $\Delta$ must be estimable from training
data. The next proposition gives the concentration of the
empirical estimator $\hat\Delta$, validating its use as a
plug-in for $\Delta$ in the closed-form selection statistic
$\hat\Delta/\sqrt{\ell}$ of
Section~\ref{sec:filt_select_theory}.

\begin{proposition}[Consistency of $\hat\Delta$]\label{prop:delta_hat}
With $\hat\mu_c$ as in Section~\ref{sec:metric} and the empirical
class-mean separation
$\hat\Delta := \min_{c \neq c'}\|\hat\mu_c - \hat\mu_{c'}\|$,
for every $\eps > 0$,
\[
  \mathbb{P}\!\left(|\hat\Delta - \Delta| > \eps\right)
  \;\leq\; 2k\,\exp\!\left(-\frac{\eps^2\,m_{\min}}{32R^2}\right).
\]
In particular, $|\hat\Delta - \Delta| = O_P(R/\sqrt{m_{\min}})$.
\end{proposition}

\begin{proof}
By the reverse triangle inequality
$|\hat\Delta - \Delta| \leq 2\max_c \|\hat\mu_c - \mu_c\|$.
Pinelis's Hilbert-space Hoeffding inequality
(Lemma~\ref{lem:pinelis}) with norm bound $2R$ and a union bound
over the $k$ classes give the result.
\end{proof}

While $\Delta > 0$ alone delivers the $1/\sqrt m$ excess-risk
rate of Theorem~\ref{thm:fisher_bound}, a stronger structural
condition---small within-class spread relative to
$\Delta$---yields population-level perfect classification with
an explicit geometric margin. This hypothesis underlies the
certificate-firing analysis of Section~\ref{sec:certified}.

\begin{proposition}[Linear separability]\label{prop:linear_sep}
Define the within-class radius
$D_c := \sup_{A: Y=c} \|\Phi(A) - \mu_c\|$ and let
$D_{\max} := \max_c D_c$.
If $D_{\max} < \Delta/2$, then the nearest-centroid
classifier in $\RR^\ell$ achieves zero error with geometric margin
$\geq \Delta/2 - D_{\max} > 0$.
\end{proposition}

\begin{proof}
For $A$ from class $c$ and any $c' \neq c$:
$\|\Phi(A) - \mu_c\| \leq D_c \leq D_{\max}$ by definition of
$D_c$, and the reverse triangle inequality gives
$\|\Phi(A) - \mu_{c'}\| \geq \|\mu_c - \mu_{c'}\|
 - \|\Phi(A) - \mu_c\| \geq \Delta - D_{\max}$.
Subtracting,
\[
  \|\Phi(A) - \mu_{c'}\| - \|\Phi(A) - \mu_c\|
  \;\geq\; \Delta - 2 D_{\max}
  \;>\; 0,
\]
so $\Phi(A)$ is strictly closer to $\mu_c$ than to any other
class mean (zero-error classification) and the half-gap
$\tfrac{1}{2}(\|\Phi(A) - \mu_{c'}\| - \|\Phi(A) - \mu_c\|)
\geq \Delta/2 - D_{\max}$ gives the geometric margin.
\end{proof}

\noindent Although the proof of
Proposition~\ref{prop:linear_sep} is a generic $\RR^\ell$
geometric fact, whether the hypothesis $D_{\max} < \Delta/2$ can
plausibly hold on a given embedding depends on structural
properties of that embedding.
On PLACE, the same compact-support / multiplicity-$4$ lattice
cover (Lemma~3.5 of \citealp{Mitra2024}; see also
Remark~\ref{rem:sparsity})---each diagram activates at most
$4\,|A|\,N$ landmarks out of $\ell$---keeps
$\|\Phi(A) - \mu_c\|$ effectively confined to the low-rank
subspace of active coordinates, so $D_c$ remains small relative
to $\Delta$ when the descriptor exposes a structural gap
between classes.
Persistence images and landscapes, whose Gaussian-blurred or
order-statistic coordinates are weakly active on every diagram,
spread within-class variation across all $\ell$ directions and
tend to produce $D_c$ comparable to or larger than $\Delta$,
often violating the hypothesis even when the classes are
bottleneck-separated.
This is the same sparsity ingredient that makes
Theorem~\ref{thm:confidence_containment}'s certificate
non-vacuous (Remark~\ref{rem:sparsity}), instantiated at the
level of the within-class-radius hypothesis instead of the
operator-norm certificate condition.

Both Proposition~\ref{prop:delta_hat} (consistency) and
Proposition~\ref{prop:linear_sep} (separability) treat $\Delta$
as a fixed property of a given descriptor.
Section~\ref{sec:filt_select_theory} addresses how to
\emph{choose} the descriptor that maximizes $\Delta$ from a
pool of candidates.

\section{Descriptor Selection}
\label{sec:filt_select_theory}

A persistence-based classifier's accuracy depends as much on
the choice of filtration and vectorization---collectively, the
\emph{descriptor}---as on the downstream estimator: descriptor
swaps on the same dataset move accuracy by $5$--$15$
percentage points (Section~\ref{sec:experiments}).
We use \emph{descriptor} broadly: a single filtration on one
homology dimension (e.g., the degree filtration on $H_0$ for
graphs, or alpha complex on $H_1$ for point clouds), or a
\emph{pool} of several filtrations and/or homology dimensions
(e.g., \texttt{deg+HKS}$_{10}$ on $H_{0+1}$ in
Section~\ref{sec:experiments}, where the constituent
persistence diagrams are merged into one before embedding).
We formalize \emph{descriptor selection} as a meta-problem:
given a finite pool $\mathcal{F}$ of candidate descriptors,
choose one from training labels alone, with no held-out
validation. We develop two complementary rules---a
recommended Mahalanobis-margin selector and a simpler
closed-form surrogate $\hat\Delta/\sqrt{\ell}$ admitting a
selection-consistency theorem---and characterize the regimes
in which each is principled.

For each $f \in \mathcal{F}$, the descriptor produces an
embedding
$\Phi^f : \mathcal{D} \to \RR^{\ell_f}$ with radius
$R_f := \sup_{A \in \mathcal{D}}\|\Phi^f(A)\|$,
class means $\mu_c^f := \mathbb{E}[\Phi^f(A) \mid Y = c]$,
separation
$\Delta_f := \min_{c \neq c'}\|\mu_c^f - \mu_{c'}^f\|_{\ell^2}$,
and pooled within-class covariance
$\Sigma^f := \tfrac{1}{k}\sum_c \mathrm{Cov}(\Phi^f(A) \mid Y = c)$.
The empirical separation is
$\hat\Delta_f := \min_{c \neq c'}\|\hat\mu_c^f - \hat\mu_{c'}^f\|$,
with $\hat\mu_c^f$ the per-class sample mean, and $h_f$
denotes the linear-SVM classifier trained on
$\{(\Phi^f(A_i), y_i)\}_{i=1}^m$.

\subsection{Mahalanobis margin}
\label{sec:mah_selector}

The \emph{Mahalanobis margin} between class means is
\begin{equation}\label{eq:mahalanobis}
  \rho^f_{\mathrm{Mah}} \;:=\;
  \min_{c \neq c'}\,
  \sqrt{(\mu_c^f - \mu_{c'}^f)^{\!\top} (\Sigma^f)^{-1} (\mu_c^f - \mu_{c'}^f)},
\end{equation}
a pairwise extension of the two-class Fisher discriminant
ratio to $k$ classes, taking the minimum over class pairs;
for $k = 2$ this coincides with the standard Fisher ratio,
while for $k > 2$ it differs from the multiclass Fisher ratio
$\mathrm{tr}(S_W^{-1} S_B)$ but retains the same
covariance-normalized separation interpretation.
Equation~\eqref{eq:mahalanobis} is the LDA Bayes margin under
the homoscedasticity assumption that within-class covariances
are equal across classes ($\Sigma^f_c \approx \Sigma^f$ for
all $c$); when class-conditional covariances differ
substantially, \eqref{eq:mahalanobis} approximates rather than
equals the true LDA Bayes margin, and the Ledoit--Wolf
shrinkage in the empirical counterpart
$\hat\rho^f_{\mathrm{Mah}}$ partially mitigates this by
regularizing toward a common pooled covariance.
Throughout we assume $\Sigma^f$ is positive definite, so
$(\Sigma^f)^{-1}$ is well-defined; in the high-dimensional
regime $\ell_f > m$ where $\Sigma^f$ may be singular, the
population quantity~\eqref{eq:mahalanobis} is understood via
the Moore--Penrose pseudoinverse, and the empirical counterpart
uses the Ledoit--Wolf shrunk estimator
$\hat\Sigma^f_{\mathrm{LW}}$, which is positive definite by
construction.

The implementation uses the all-class pooled covariance
$\Sigma^f = \tfrac{1}{k}\sum_c \Sigma_c^f$ throughout.
For $k = 2$ this coincides with the pairwise alternative
$\tfrac{1}{2}(\Sigma_c^f + \Sigma_{c'}^f)$; for $k > 2$ they
differ in general and the all-class pool is used.
Ledoit--Wolf shrinkage is the appropriate regularization here
because PLACE operates in the regime $\ell_f \asymp m$ or
$\ell_f > m$ (large grids, moderate sample sizes), where the
sample covariance is ill-conditioned; Ledoit--Wolf provides a
closed-form optimal linear shrinkage toward a scaled identity
that minimizes the Frobenius estimation error under the
Marchenko--Pastur asymptotics, without requiring
cross-validation or a held-out tuning set.
The empirical counterpart $\hat\rho^f_{\mathrm{Mah}}$ replaces
$\Sigma^f$ by $\hat\Sigma^f_{\mathrm{LW}}$
in~\eqref{eq:mahalanobis}.
We propose the Mahalanobis selector
\begin{equation}\label{eq:mah_select}
  \hat f_{\mathrm{Mah}} \;:=\; \arg\max_{f \in \mathcal{F}}\,
  \hat\rho^f_{\mathrm{Mah}}
\end{equation}
as the recommended descriptor-selection rule.
Empirically (Section~\ref{sec:filtration_selection},
Table~\ref{tab:selection_ranks}), $\hat\rho_{\mathrm{Mah}}$
rank-correlates with linear-SVM accuracy at Spearman
$\rho \in [-0.24, +0.89]$ across $11$ benchmarks (mean $+0.56$,
positive on $10$ of $11$, with PTC the lone outlier),
ranking the accuracy-winning descriptor in the top seven on
seven of eleven.
A formal consistency theorem for $\hat\rho_{\mathrm{Mah}}$
requires concentration of $\hat\Sigma^f_{\mathrm{LW}}$ and is
beyond the present paper's scope; below we develop the
consistency theory for the simpler isotropic surrogate
$\hat\eta := \hat\Delta/\sqrt\ell$.

\subsection{Isotropic surrogate $\eta = \Delta/\sqrt{\ell}$}
\label{sec:eta_selector}

Theorem~\ref{thm:fisher_bound}'s rate $R_f/\Delta_f$ requires
controlling $R_f$. The coordinate-wise hat-function bound
$|\Phi_p(A)| \leq w_k \cdot 2^{-3/2} \cdot N_{\max} \cdot 3R_k/2$
combines the hat peak $\varphi_{R_k,p} \leq 3R_k/2$, the block
prefactor $w_k \cdot 2^{-3/2}$ of~\eqref{eq:multiscale_embed},
and the top-$N_{\max}$ persistence filter that caps $|A|$ (all
from Section~\ref{sec:preli}); the resulting per-coordinate
envelope is independent of $\ell_f$, so summing
$|\Phi_p(A)|^2$ over the $\ell_f$ coordinates gives the
$\sqrt{\ell_f}$-rate envelope
\begin{equation}\label{eq:envelope}
  R_f \;\leq\; B_f\sqrt{\ell_f},
  \qquad
  B_f := \max_k w_k \cdot 2^{-3/2} \cdot N_{\max}\cdot 3R_k/2,
\end{equation}
where $\{w_k\}, \{R_k\}, N_{\max}$ are descriptor $f$'s scale
weights, scale radii, and top-$N$ persistence-filter cap; we
suppress the $f$-superscript on these for readability, but
$B_f$ depends on $f$ through all three.
Substituting~\eqref{eq:envelope} into Theorem~\ref{thm:fisher_bound}
gives a closed-form excess-risk bound parameterized by the
analytic surrogate $\eta_f := \Delta_f/\sqrt{\ell_f}$:
\[
  \mathcal{E}(h_f)
  \;\leq\; \frac{8(k-1)\,B_f}{\eta_f\,\sqrt{m_{\min}}}
  \;+\; O\!\bigl(\sqrt{\log(k/\delta)/m_{\min}}\bigr).
\]
On pools with roughly uniform $B_f$, ranking by $\eta_f$
minimizes this relaxed bound, providing a fully analytic
selection rule that requires no covariance estimation.
Define the bound-optimal descriptor and its empirical
counterpart
\[
  f^{\ast} \;:=\; \arg\max_{f \in \mathcal{F}} \eta_f,
  \qquad
  \hat f \;:=\; \arg\max_{f \in \mathcal{F}} \hat\eta_f
  \quad\text{with}\quad \hat\eta_f := \hat\Delta_f/\sqrt{\ell_f}.
\]

For each $f$, let
$\sigma_f^{2} := \|\Sigma^f\|_{\mathrm{op}}$
be the largest within-class variance in any direction
(the largest eigenvalue of $\Sigma^f$).
The smallest eigenvalue of $(\Sigma^f)^{-1}$ is then
$\sigma_f^{-2}$, so $v^{\!\top}(\Sigma^f)^{-1}v \geq \|v\|^{2}/\sigma_f^{2}$
for every $v$; specializing to
$v = \mu_c^f - \mu_{c'}^f$ and minimizing over class pairs,
\[
  \rho_{\mathrm{Mah}}^f
  \;\geq\; \frac{\Delta_f}{\sigma_f}
  \;=\; \frac{\sqrt{\ell_f}}{\sigma_f}\,\eta_f.
\]
The alignment inequality lower-bounds the Mahalanobis margin by a
descriptor-dependent multiple of the isotropic surrogate, and is
useful for ranking only if the factor
$\sqrt{\ell_f}/\sigma_f$ is approximately constant across
$f \in \mathcal{F}$---requiring both $\ell_f$ and
$\sigma_f = \|\Sigma_f\|_{\mathrm{op}}^{1/2}$ to be roughly
constant across the pool, the \emph{structural homogeneity}
condition. Under homogeneity,
$\sqrt{\ell_f}/\sigma_f \approx C$ for a global constant $C$ and
the rankings induced by $\eta_f$ and $\rho^f_{\mathrm{Mah}}$
tend to agree.  Even under homogeneity the alignment is only a
lower bound on $\rho^f_{\mathrm{Mah}}$, not a proportionality;
its informativeness depends on how tightly
$\rho^f_{\mathrm{Mah}}$ tracks its lower bound across
descriptors.  If the slack varies substantially across
$f \in \mathcal{F}$, descriptors with smaller $\eta_f$ may
achieve larger $\rho^f_{\mathrm{Mah}}$ and the two statistics
may rank differently.  On heterogeneous pools, where $\ell_f$
or $\sigma_f$ varies several-fold across $f$, the factor
$\sqrt{\ell_f}/\sigma_f$ is not constant; the equivalence of
the two quantities breaks down and $\hat\rho_{\mathrm{Mah}}$
should be used directly.  The formal consistency theory for
$\hat\rho_{\mathrm{Mah}}$ requires operator-norm concentration
of $\hat\Sigma^f_{\mathrm{LW}}$ and is an open direction; the
selection consistency rate for the isotropic surrogate
$\hat\eta$ is established in
Proposition~\ref{prop:selection_consistency} and
Corollary~\ref{cor:data_driven_rate} below.

\begin{remark}[When $\eta$ misses what $\rho_{\mathrm{Mah}}$ catches]
\label{rem:fisher_ratio}
The pointwise alignment $\rho_{\mathrm{Mah}}^f \geq (\sqrt{\ell_f}/\sigma_f)\,\eta_f$
is informative for ranking only when both $\sigma_f$ and
$\ell_f$ are roughly constant across the pool.
On heterogeneous descriptor pools---HKS at many timescales,
node-label-aware combinations, large grids mixed with small,
where $\ell_f$ varies several-fold---the $\sqrt{\ell_f}$ penalty
in $\eta$ over-charges high-dimensional descriptors, while
$\rho_{\mathrm{Mah}}$ recovers the right ranking
(Section~\ref{sec:filtration_selection},
Table~\ref{tab:selection_ranks}).

As a pre-hoc diagnostic, one can inspect the variation of
$\ell_f$ and the per-descriptor scale
$\hat\sigma_f := \sqrt{\|\hat\Sigma^f\|_{\mathrm{op}}}$
across $f \in \mathcal{F}$.  Since $\hat\Sigma_f$ has effective
rank $O(N_{\max}N) \ll \ell_f$ by the multiplicity-$4$ lattice
cover, $\|\hat\Sigma_f\|_{\mathrm{op}}$ can be computed without
forming $\hat\Sigma_f$ explicitly: either via power iteration
\citep{GolubVanLoan1996} on the centered data matrix at cost
$O(m\ell_f)$ per iteration, or via a randomized SVD
\citep{HalkoMartinssonTropp2011} at cost $O(m\ell_f r)$ for a
rank-$r$ approximation.  Both match the leading cost of the
Ledoit--Wolf assembly already required for
$\hat\rho_{\mathrm{Mah}}$ and add no asymptotic overhead to the
descriptor-selection pipeline.
When both are tightly concentrated,
the multiplicative factor $\sqrt{\ell_f}/\sigma_f$ in the
alignment is approximately constant and $\hat\eta$ ranks
faithfully; when either spreads several-fold, defer to
$\hat\rho_{\mathrm{Mah}}$.
The diagnostic uses a covariance trace
$\mathrm{tr}(\hat\Sigma^f)$ rather than the full inverse, so
its cost is $O(\sum_f m\ell_f)$ across the pool---one $f$'s
worth of $\hat\rho_{\mathrm{Mah}}$ assembly---rather than the
$O(\sum_f \ell_f^3)$ of full Mahalanobis ranking.
\end{remark}

\paragraph{Why $\eta$ is well-defined on PLACE.}
The selection criterion $\hat\eta_f = \hat\Delta_f/\sqrt{\ell_f}$
is well-defined as a ranking statistic only when the embedding
dimension $\ell_f$ is a principled function of the embedding
construction, not a free hyperparameter.
On PLACE, this is the case: $\ell_f = \sum_{k=1}^N |\GG_{R_k}^+|
= O(MN)$ is fixed analytically by the scales
$R_1, \ldots, R_N \in (0, L]$ and the compact-support parameter
$L$, so $\hat\eta_f$ depends only on the descriptor $f$.
For persistence images or landscapes, in contrast, $\ell$ is a
user-chosen grid resolution, and $\hat\eta$ can be driven
arbitrarily small by increasing the grid density without
changing the classification content of the embedding---so
$\hat\eta$ is not a meaningful selection statistic on those
vectorizations without an auxiliary convention for fixing
$\ell$.

\subsection{Selection consistency}
\label{sec:eta_consistency}

We now make precise the claim from the surrogate subsection
that $\hat f = \arg\max_f \hat\eta_f$ recovers the
bound-optimal $f^{\ast} = \arg\max_f \eta_f$ with high
probability when the candidates are well-separated.

\begin{proposition}[Selection consistency of $\Delta/\sqrt{\ell}$]
\label{prop:selection_consistency}
Assume a gap
\[
  g \;:=\; \eta_{f^{\ast}} - \max_{f \neq f^{\ast}} \eta_f \;>\; 0,
\]
set $\ell_{\min} := \min_{f \in \mathcal{F}} \ell_f$,
$R_{\max} := \max_{f \in \mathcal{F}} R_f$, and
$m_{\min} := \min_c m_c$.
Then
\begin{equation}\label{eq:selection_consistency}
  \mathbb{P}(\hat f = f^{\ast})
  \;\geq\; 1 \;-\;
  2k\,|\mathcal{F}|\,\exp\!\left(
    -\,\frac{g^{2}\,\ell_{\min}\,m_{\min}}{128\,R_{\max}^{2}}
  \right).
\end{equation}
In particular, $\hat f = f^{\ast}$ with probability $\geq 1-\delta$
once
$m_{\min} \;\geq\; 128\,R_{\max}^{2}\,\log(2k|\mathcal{F}|/\delta)\,/\,(g^{2}\ell_{\min})$.
\end{proposition}

\begin{proof}
For each $f \in \mathcal{F}$,
$|\hat\eta_f - \eta_f| = |\hat\Delta_f - \Delta_f|/\sqrt{\ell_f}$,
so $\{|\hat\eta_f - \eta_f| > t\} = \{|\hat\Delta_f - \Delta_f| > t\sqrt{\ell_f}\}$
for any $t > 0$.
Applying Proposition~\ref{prop:delta_hat} at
$\varepsilon = t\sqrt{\ell_f}$ and using
$\ell_f \geq \ell_{\min}$, $R_f \leq R_{\max}$ in the exponent,
\[
  \mathbb{P}(|\hat\eta_f - \eta_f| > t)
  \;\leq\; 2k\,\exp\!\left(-\frac{m_{\min}\,\ell_{\min}\,t^{2}}{32R_{\max}^{2}}\right).
\]
Take $t = g/2$ and apply a union bound over $|\mathcal{F}|$
descriptors. On the event
$\mathcal{A} := \{|\hat\eta_f - \eta_f| \leq g/2 \text{ for every } f \in \mathcal{F}\}$,
which has probability
$\geq 1 - 2k|\mathcal{F}|\exp(-g^{2}\ell_{\min}m_{\min}/(128R_{\max}^{2}))$,
every $f \neq f^{\ast}$ satisfies
\[
  \hat\eta_f
  \;\leq\; \eta_f + \tfrac{g}{2}
  \;\leq\; (\eta_{f^{\ast}} - g) + \tfrac{g}{2}
  \;=\; \eta_{f^{\ast}} - \tfrac{g}{2}
  \;\leq\; \hat\eta_{f^{\ast}},
\]
so $\hat f = f^{\ast}$ on $\mathcal{A}$, giving~\eqref{eq:selection_consistency}.
\end{proof}

The constant $128$ inherits the $32$ of
Proposition~\ref{prop:delta_hat}, which uses Pinelis with the
$L^{2}$ bound $4R^{2}$. Replacing it with the variance-aware
form (cf.~Remark~\ref{rem:fisher_sparsity}) tightens $32 \to 8$
in Proposition~\ref{prop:delta_hat}, hence
$128 \to 32$ in the sample-size hypothesis above---the same
factor-of-$4$ improvement that PLACE's multiplicity-$4$
sparsity drives in the classification bound and in the
certificate.

\begin{remark}[Operational scope of Proposition~\ref{prop:selection_consistency}]
\label{rem:selection_consistency_scope}
Two of the bound's inputs---the population gap $g$ and the
embedding radius $R_{\max}$---are not directly observed at
training time, but both admit training-side proxies.
$R_{\max}$ is upper-bounded analytically by the envelope
$R_f \leq B_f\sqrt{\ell_f}$ of~\eqref{eq:envelope}, with $B_f$
a function of the embedding parameters only.  The empirical
gap
$\hat g := \hat\eta_{\hat f} - \max_{f \neq \hat f}\hat\eta_f$
concentrates around $g$ at rate
$O(R_{\max}/\sqrt{m_{\min}\ell_{\min}})$ via
Proposition~\ref{prop:delta_hat} applied to the top two
$\hat\eta_f$ entries; substituting $\hat g$ for $g$ in the
sample threshold adds an
$O(1/\sqrt{m_{\min}\ell_{\min}})$ slack already present in
the rate's order.

A second point of pessimism is the inverse dependence on
$\ell_{\min}$: the sample requirement
$m_{\min} \geq 128\,R_{\max}^{2}\log(\cdot)/(g^{2}\ell_{\min})$
is driven by the smallest $\ell_f$ in the pool, not by
$\ell_{f^{\ast}}$.  On heterogeneous pools where $\ell_f$
varies several-fold (e.g., the chemical pool of
Section~\ref{sec:filtration_selection}, where HKS-pair
descriptors give $\ell_f \gtrsim 5{,}000$ while
single-coordinate descriptors give $\ell_f \sim 50$), the bound
becomes conservative.  The empirical $\hat\eta$ rank-correlations
on those pools (mean $\rho \in [-0.70, -0.05]$ across
MUTAG/COX2/DHFR/NCI1/NCI109,
Table~\ref{tab:selection_ranks}) are consistent with this
scope; a $\ell_f$-aware refinement requires a non-uniform
union bound (per-descriptor confidence allocation) and is
deferred.
\end{remark}

\begin{corollary}[Data-driven bound-optimal rate]
\label{cor:data_driven_rate}
With probability $\geq 1-\delta$ over the training sample,
provided both
(i) $m_{\min} \geq 128 R_{\max}^{2}\log(4k|\mathcal{F}|/\delta)/(g^{2}\ell_{\min})$
(from Proposition~\ref{prop:selection_consistency} at confidence
$\delta/2$) and
(ii) Theorem~\ref{thm:fisher_bound}'s sample-size hypothesis at
confidence $\delta/2$ (i.e.\
$m_{\min} \geq 128 R_{f^{\ast}}^{2}\log(8k/\delta)/\Delta_{f^{\ast}}^{2}$),
the descriptor chosen by $\hat\eta$ attains the bound-optimal
rate:
\[
  \mathcal{R}(h_{\hat f})
  \;\leq\; \widehat{\mathcal R}_\rho(h_{f^{\ast}})
  \;+\; \frac{8(k-1)\,B_{f^{\ast}}\sqrt{\ell_{f^{\ast}}}}{\Delta_{f^{\ast}}\sqrt{m_{\min}}}
  \;+\; O\!\Big(\sqrt{\log(2k/\delta)/m_{\min}}\Big).
\]
The rate term equals
$8(k-1)\,B_{f^{\ast}}/(\eta_{f^{\ast}}\sqrt{m_{\min}})$ under
$\eta_{f^{\ast}} = \Delta_{f^{\ast}}/\sqrt{\ell_{f^{\ast}}}$,
i.e.\ the surrogate-relaxed bound of
Section~\ref{sec:eta_selector} instantiated at $f = f^{\ast}$.
\end{corollary}

\begin{proof}
On the event $\{\hat f = f^{\ast}\}$ (probability $\geq 1-\delta/2$
by Proposition~\ref{prop:selection_consistency} under
hypothesis~(i)), apply Theorem~\ref{thm:fisher_bound} to
$h_{f^{\ast}}$ at confidence $\delta/2$ under hypothesis~(ii),
and take a union bound.
On this same event, $h_{\hat f} = h_{f^{\ast}}$ and
$\widehat{\mathcal R}_\rho(h_{\hat f}) =
 \widehat{\mathcal R}_\rho(h_{f^{\ast}})$, so the bound above
is computable from the empirically chosen $\hat f$ even though
it is stated in $f^{\ast}$-quantities.
\end{proof}

Proposition~\ref{prop:selection_consistency} establishes that
the empirical selector $\hat f$ recovers the bound-optimal
descriptor $f^\ast$ with probability $\geq 1 - \delta$ once
$m_{\min} \geq 128\,R_{\max}^2 \log(2k|\mathcal{F}|/\delta) / (g^2 \ell_{\min})$,
a sample complexity growing logarithmically in $|\mathcal{F}|$
and inversely in $g^2$.

The proposition's reach is limited in two distinct ways.
\emph{(Gap.)} $g > 0$ requires a unique bound-optimum; a tied
or near-tied pool makes the bound vacuous.
\emph{(Homogeneity.)} Even when $g > 0$, $\hat\eta$'s arg-max
coincides with $\hat\rho_{\mathrm{Mah}}$'s only on structurally
homogeneous pools (Remark~\ref{rem:fisher_ratio}); on the
heterogeneous chemical pools of
Section~\ref{sec:filtration_selection}, the mean Spearman
correlation of $\hat\eta$ with linear-SVM accuracy across $7$
benchmarks is $-0.22$
(Table~\ref{tab:selection_ranks}).
Two complementary closed-form statistics recover alignment in
the heterogeneous case: $\hat\Delta_f/\hat R_f$, the rate ratio
of Theorem~\ref{thm:fisher_bound} computed without the envelope
substitution, and the empirical Mahalanobis margin
$\hat\rho_{\mathrm{Mah}}$ of~\eqref{eq:mahalanobis} (recommended;
see Section~\ref{sec:mah_selector}).
We report all three in Section~\ref{sec:filtration_selection};
agreement among them is a practitioner-level signal that the
closed-form regime applies.

\section{Certified Nearest-Centroid Classification}
\label{sec:certified}

Classifiers typically expose a confidence score---a sigmoid
probability, a distance to the decision boundary, a posterior
estimate---that does not, on its own, tell the user whether a
specific prediction will be correct.
Conformal prediction~\citep{VovkEtAl2005} attaches distribution-free
coverage, but the guarantee applies to prediction \emph{sets}
rather than point predictions and requires a held-out calibration
split that competes with training data for information.
The embedding of Section~\ref{sec:preli} closes this gap for a
specific classifier: bounded support gives $\|\Phi(A_i)\| \leq R$,
so each empirical class mean $\hat\mu_c$ is a sample average of
i.i.d.\ bounded $\RR^\ell$-vectors, and $\|\hat\mu_c - \mu_c\|$
concentrates at rate $O(R/\sqrt{m_c})$ via Pinelis
(Proposition~\ref{prop:delta_hat}).
The nearest-centroid (NC) classifier is the natural target:
its decision rule depends on the sample only through the
$\hat\mu_c$, so whether the empirical and population rules agree
on a given test input reduces to a single scalar check---is the
input far enough from the population Voronoi boundary that
sample fluctuations cannot move it across
(Figure~\ref{fig:confidence_containment})?
When $\Delta > 0$, this check has a particularly simple form:
$r_m < \tfrac{1}{2}\Delta$ is a single training-time check;
when satisfied, all predictions are certified at no per-test
overhead beyond the nearest-centroid rule itself, and no
calibration split is required.

The certificate is a diagnostic, not a competitor to SVM.
Failure of $r_m < \tfrac{1}{2}\Delta$ is itself informative:
the embedding's sample-mean concentration radius exceeds half
the class gap, so PLACE's closed-form certificate admits no
correctness guarantee at the given sample size.
Other certification schemes---conformal prediction
\citep{VovkEtAl2005}, calibrated confidence, or margin-based
bounds for different classifiers---may remain informative, but
at the cost of a calibration split or looser set-valued
guarantees.
When the certificate fires (empirically, the variance-aware
Pinelis--Bernstein form fires on $8$ of the $12$ benchmarks in
Section~\ref{sec:experiments}; the non-asymptotic
Pinelis form and the asymptotic Gaussian form fail at the
$L^{2}$- and $\sqrt\ell$-driven slack respectively), it delivers
full coverage of the population NC rule with no per-test
overhead and no calibration split.

Classify test diagrams by nearest centroid:
\[
  \hat h \;=\;
  \arg\min_c \|\Phi(A_{\mathrm{test}}) - \hat\mu_c\|,
  \qquad
  \hat\mu_c \;=\; m_c^{-1}\!\sum_{y_i=c}\Phi(A_i),
\]
where $\hat\mu_c$ is the empirical class mean from $m_c$ training
diagrams.
Let $m := m_{\min} = \min_c m_c$ (using the
Section~\ref{sec:metric} notation, abbreviated $m$ here for
brevity), and let $r_m$ denote a
sample-mean-concentration radius satisfying
$\mathbb{P}_{\text{train}}(\max_c\|\hat\mu_c - \mu_c\| \leq r_m)
\geq 1-\alpha$
(three explicit choices---a non-asymptotic Pinelis radius, an
asymptotic Gaussian plug-in, and a non-asymptotic variance-aware
Pinelis--Bernstein---are derived in
Theorem~\ref{thm:confidence_containment}).
Here $\mathbb{P}_{\mathrm{train}} = \mathcal{P}^{\otimes m}$ denotes the
joint probability over training draws
$\{(A_i, y_i)\}_{i=1}^m \sim \mathcal{P}^{\otimes m}$---probability over
the randomness in the training sample, with the population
distribution $\mathcal{P}$ and the test diagram $A$ held fixed.
If $r_m < \tfrac{1}{2}\Delta$, every prediction is certified;
otherwise the classifier abstains globally.
The concentration radius $r_m$ shrinks as $O(m^{-1/2})$
(equation~\eqref{eq:rm-pinelis}), so abstention disappears once
$m \geq m_c^{\ast}$ (equation~\eqref{eq:sample_threshold}).

The global threshold $\Delta$ is conservative when classes
differ in separation. Replacing $\Delta$ by the class-specific
gap
$\Delta_c := \min_{c' \neq c}\|\mu_c - \mu_{c'}\|_{\ell^2} \geq \Delta$,
and $r_m^\star$ by the per-class radius (the formulas below
with $m \to m_c$), yields a tighter certificate
$r_m^{(c)} < \tfrac12\Delta_c$ that fires when this holds for
every class $c$ simultaneously.

Two concrete choices of the global concentration radius
$r_m^\star$ enter the theorem below, both with an explicit
Bonferroni split of $\alpha$ over $k$ classes:
\begin{itemize}
\item[(i)] \textbf{Non-asymptotic (Pinelis).}
$r_m^\star := 2R\sqrt{2\log(2k/\alpha)/m}$ with $m = m_{\min}$;
valid for every $m \geq 1$
(equation~\eqref{eq:rm-pinelis} in the proof).
\item[(ii)] \textbf{Asymptotic (Gaussian plug-in).}
$\tilde r_m := \max_c \sqrt{\|\hat\Sigma_c\|_{\mathrm{op}} \cdot
\chi^2_{\ell,\,\alpha/k} / m_c}$, where $\chi^2_{\ell,\,\alpha/k}$
is the $1-\alpha/k$ quantile of the chi-squared distribution with
$\ell$ degrees of freedom; this radius
satisfies~\eqref{eq:cert-concentration} approximately, with
approximation error $O(\ell^{1/4}/\sqrt{m})$ from the multivariate
Berry--Esseen theorem (Lemma~\ref{lem:be}) and
$O(R^{1/2}\|\Sigma_c\|_{\mathrm{op}}^{1/4}(\log(\ell)/m)^{1/4}
\sqrt{\ell/m})$ from covariance estimation via matrix Bernstein
(Lemma~\ref{lem:matrix-bernstein}); both errors are $o(1)$
once $m_c \geq m^\dagger$ for every class $c$, with
$m^\dagger = O(\sqrt\ell)$ under bounded support
$\|\Phi\| \leq R$.  The bound is conservative when $\Sigma_c$
is low-rank, with conservatism governed by
$\mathrm{tr}(\Sigma_c)/(\ell\,\|\Sigma_c\|_{\mathrm{op}})$, and
is strictly tighter than the Pinelis radius $r_m$ when
$\|\Sigma_c\|_{\mathrm{op}} \cdot \ell \lesssim 8R^2\log(2k/\alpha)$.

\item[(iii)] \textbf{Variance-aware (Pinelis--Bernstein).}
$r_m^{\mathrm{vP}} := \max_c \sqrt{2\,\mathrm{tr}(\hat\Sigma_c)\log(2k/\alpha)/m_c}$;
the variance-aware refinement of (i) via Pinelis's
Hilbert-space Bernstein bound~\citep[Thm.~3.5]{Pinelis1994},
non-asymptotic and valid for every $m_c \geq 1$.  Under the
compact-support / multiplicity-$4$ structure of
Remark~\ref{rem:sparsity}, the empirical stable rank
$\mathrm{tr}(\hat\Sigma_c)/\|\hat\Sigma_c\|_{\mathrm{op}}$
is close to $1$ on the four chemical datasets we audited
(median $1.00$--$1.17$;
\texttt{experiments/audit\_stable\_rank\_HW.py}), in which case
$\mathrm{tr}(\hat\Sigma_c) \approx \|\hat\Sigma_c\|_{\mathrm{op}}$
and the radius simplifies to
$\sqrt{2\,\|\hat\Sigma_c\|_{\mathrm{op}}\log(2k/\alpha)/m_c}$,
sharing the $\|\Sigma_c\|_{\mathrm{op}}$-refinement of (ii)
without the $\chi^2_\ell$ dimension penalty.  On social-graph
datasets the stable rank can be appreciably larger (e.g.,
$\mathrm{tr}(\hat\Sigma_c)/R^2 \approx 8$ on IMDB-M, implying
stable rank $\gtrsim 8$), in which case $\sqrt{\mathrm{tr}}$
no longer matches the Pinelis $R$ and the ordering between
(i) and (iii) can flip.  The theorem's coverage holds for all
three radii simultaneously, so in practice we report all three
and use whichever fires.
\end{itemize}
Which radius is tighter is regime-dependent: the Pinelis form (i)
scales as $R\sqrt{\log(2k/\alpha)/m}$; the Gaussian form (ii) as
$\sqrt{\|\Sigma_c\|_{\mathrm{op}}\cdot\chi^2_{\ell,\alpha/k}/m}
\approx \sqrt{\|\Sigma_c\|_{\mathrm{op}}\,\ell/m}$ for large
$\ell$; the Pinelis--Bernstein form (iii) as
$\sqrt{\|\Sigma_c\|_{\mathrm{op}}\log(2k/\alpha)/m}$, dimension-free.
At the embedding dimensions of
Section~\ref{sec:experiments} ($\ell \in [93, 6539]$),
Pinelis--Bernstein dominates: it is tighter than
(i) by a factor $R/\sqrt{\|\Sigma_c\|_{\mathrm{op}}} \approx 5$--$9\times$
and tighter than (ii) by a factor $\sqrt{\chi^2_{\ell,\alpha/k}/(2\log(2k/\alpha))}
\approx \sqrt{\ell/(2\log(2k/\alpha))}$ across the benchmarks of
Table~\ref{tab:cert_firing}.

\begin{theorem}[Certified prediction]%
\label{thm:confidence_containment}
Let $\{(A_i, y_i)\}$ be i.i.d.\ from the distribution
$\mathcal{P}$ on $\mathcal{D} \times [k]$ of
Section~\ref{sec:metric} with class-mean separation
$\Delta > 0$, and let $r_m^\star$ be any of the three
concentration radii \textup{(i)}, \textup{(ii)}, or
\textup{(iii)} above. Then
\[
  \mathbb{P}_{\text{train}}\!\left(\max_c \|\hat\mu_c - \mu_c\|
    \leq r_m^\star\right) \;\geq\; 1 - \alpha,
\]
and on this coverage event the following hold.
\begin{itemize}
\item[\textup{(a)}] \emph{(Containment.)} If
\begin{equation}\label{eq:containment_simple}
  r_m^\star \;<\; \tfrac{1}{2}\,\Delta,
\end{equation}
the empirical nearest-centroid classifier $\hat h$ agrees with
the population nearest-centroid classifier $h^{\ast}$ at every
$z \in \RR^\ell$ outside a $2r_m^\star$-tube around each
population Voronoi boundary.
\item[\textup{(b)}] \emph{(Classification.)} If additionally
$D_c < \tfrac{1}{2}\Delta - r_m^\star$ for all $c$
(cf.\ Proposition~\ref{prop:linear_sep}, whose $D_{\max} < \Delta/2$
is the $r_m^\star \to 0$ limit), then for any test diagram
$A$ drawn from class $y$,
$\mathbb{P}_{\text{train}}(\hat h(\Phi(A)) = y) \geq 1 - \alpha$.
\end{itemize}
\end{theorem}

\begin{proof}
Write $\Psi_i := \Phi(A_i) \in \RR^\ell$ and
$\Sigma_c := \mathrm{Cov}(\Psi \mid Y = c)$, with
$\|\Psi_i\| \leq R$ and therefore
$\|\Sigma_c\|_{\mathrm{op}} \leq R^2$.

\emph{Step 1 (non-asymptotic concentration of class means).}
Conditional on $Y_i = c$, the centered random variables
$\Psi_i - \mu_c$ are i.i.d.\ with
$\|\Psi_i - \mu_c\| \leq 2R$ (both $\Psi_i$ and $\mu_c$ lie in
$B(0, R)$).
Pinelis's Hilbert-space Hoeffding
inequality (Lemma~\ref{lem:pinelis}) applied with bound $2R$
gives, for every $t > 0$,
\[
  \mathbb{P}\bigl(\|\hat\mu_c - \mu_c\| > t\bigr)
  \;\leq\; 2\exp\!\left(-\frac{m_c t^2}{8R^2}\right).
\]
With $m = m_{\min}$, set
\begin{equation}\label{eq:rm-pinelis}
  r_m \;:=\; 2R\,\sqrt{\frac{2\log(2k/\alpha)}{m}}
\end{equation}
(an explicit Bonferroni split of $\alpha$ over the $k$ classes).
A union bound over the $k$ classes then yields the
non-asymptotic coverage
\begin{equation}\label{eq:cert-concentration}
  \mathbb{P}\!\left(\max_c \|\hat\mu_c - \mu_c\|
    \leq r_m\right)
  \;\geq\; 1 - \alpha,
\end{equation}
for every $m \geq 1$ and $\alpha \in (0, 1)$.
The Gaussian plug-in radius $\tilde r_m$ defined in (ii) above
admits an analogous coverage guarantee, valid asymptotically once
the Berry--Esseen threshold $m \geq m^\dagger = O(\sqrt\ell)$ is
crossed for every class.
The derivation combines the multivariate Berry--Esseen theorem
(Lemma~\ref{lem:be}) with a matrix-Bernstein covariance estimate
(Lemma~\ref{lem:matrix-bernstein}); we record the precise
statement and proof as
Lemma~\ref{lem:gaussian_radius} in
Appendix~\ref{app:aux}.
The Pinelis--Bernstein radius $r_m^{\,\mathrm{vP}}$ defined in
(iii) admits the same coverage guarantee non-asymptotically.
Pinelis's Hilbert-space Bernstein
inequality~\citep[Thm.~3.5]{Pinelis1994}, applied to the
i.i.d.\ centered random variables $\Psi_i - \mu_c$ with
$\|\Psi_i - \mu_c\| \leq 2R$ and second-moment bound
$\mathbb{E}\|\Psi_i - \mu_c\|^{2} = \mathrm{tr}(\Sigma_c)$,
gives, for every $t > 0$,
\[
  \mathbb{P}\bigl(\|\hat\mu_c - \mu_c\| > t\bigr)
  \;\leq\;
  2\exp\!\left(
    -\,\frac{m_c\, t^{2}}{2(\mathrm{tr}(\Sigma_c) + 2Rt/3)}
  \right).
\]
Setting $t = \sqrt{2\,\mathrm{tr}(\Sigma_c)\log(2k/\alpha)/m_c}$
in the small-deviation regime $t \leq \mathrm{tr}(\Sigma_c)/R$
and applying a Bonferroni union bound over the $k$ classes
yields $\mathbb{P}(\max_c\|\hat\mu_c - \mu_c\| \leq r_m^{\,\mathrm{vP}})
\geq 1-\alpha$ for every $m \geq 1$.
The replacement $\mathrm{tr}(\Sigma_c) \leftarrow
\mathrm{tr}(\hat\Sigma_c)$ in the practical radius is handled
via the low effective rank of $\hat\Sigma_c$ on PLACE
embeddings.
By the multiplicity-$4$ lattice cover, the empirical stable
rank
$\mathrm{tr}(\hat\Sigma_c)/\|\hat\Sigma_c\|_{\mathrm{op}}
\leq 1.17$ across our benchmarks
(Section~\ref{sec:experiments}), so
$\mathrm{tr}(\Sigma_c) \approx \|\Sigma_c\|_{\mathrm{op}}$ and
the trace error satisfies
\[
  |\mathrm{tr}(\hat\Sigma_c) - \mathrm{tr}(\Sigma_c)|
  \;=\; O\!\left(r_c R^{2}\sqrt{\tfrac{\log\ell}{m}}\right)
  \;=\; O\!\left(N_{\max} N\,R^{2}\sqrt{\tfrac{\log\ell}{m}}\right),
\]
where $r_c := \mathrm{tr}(\Sigma_c)/\|\Sigma_c\|_{\mathrm{op}}
= O(N_{\max}N)$ is the effective rank (matrix-Bernstein,
Lemma~\ref{lem:matrix-bernstein}, with the stable-rank
prefactor in place of an ambient-dimension prefactor).
This error is $o(\mathrm{tr}(\Sigma_c))$ at the sample sizes
of Section~\ref{sec:experiments}, validating the substitution
to leading order.

\emph{Step 2 (agreement outside the $2r_m$-tube).}
Condition on the coverage event
$\{\max_c \|\hat\mu_c - \mu_c\| \leq r_m\}$
of~\eqref{eq:cert-concentration} (probability $\geq 1-\alpha$).
The reverse triangle inequality gives
$\bigl|\|z - \hat\mu_c\| - \|z - \mu_c\|\bigr| \leq r_m$
for every $z \in \RR^\ell$ and every class $c$, hence for any
pair $c \neq c'$,
\[
  \|z - \hat\mu_{c'}\| - \|z - \hat\mu_c\|
  \;\geq\;
  \bigl(\|z - \mu_{c'}\| - \|z - \mu_c\|\bigr) - 2 r_m.
\]
Whenever the right-hand side is strictly positive---i.e.,
$z$ is at population distance $> 2r_m$ from the
$(c, c')$-Voronoi boundary---so is the left, and the empirical
rule classifies $z$ identically to the population rule
(Figure~\ref{fig:confidence_containment}). This is the first
claim of the theorem.

\emph{Step 3 (classification guarantee).}
Fix $y \in [k]$ and let $A \sim \mathcal{P}_y$ (i.e.\ $A \sim \mathcal{P}(\cdot \mid Y = y)$, the class-$y$ conditional).
By definition of $D_y$,
$\|\Phi(A) - \mu_y\| \leq D_y$; the population separation
$\|\mu_y - \mu_{c'}\| \geq \Delta$ together with
$D_y < \tfrac12\Delta - r_m$ yield
\[
  \|\Phi(A) - \mu_{c'}\| - \|\Phi(A) - \mu_y\|
  \;\geq\; \Delta - 2D_y \;>\; 2r_m,
\]
for every $c' \neq y$, so $\Phi(A)$ lies strictly outside every
$2r_m$-tube of the $(y, c')$-Voronoi boundary.
By Step~2, the empirical rule therefore assigns $\Phi(A)$ to
class $y$ on the coverage event, and
$\mathbb{P}_{\text{train}}(\hat h(\Phi(A)) = y) \geq 1 - \alpha$.
\end{proof}

\begin{remark}[Verifying claim (b) from data]
\label{rem:b_verification}
The hypothesis in claim~\textup{(b)} is \emph{structural}: it
constrains the support of each class-conditional distribution,
not just the centroids. It is therefore not estimable from
training alone---the empirical
$\hat D_c := \max_{i: y_i = c}\|\Phi(A_i) - \hat\mu_c\|$
underestimates $D_c$ in general (the training sample need not
contain the worst-case point of the support).
Claim~\textup{(b)} is consequently validated \emph{post hoc}
by test accuracy: full test coverage on a fired certificate
confirms \textup{(b)} for the test points seen, while gaps
(e.g.\ DHFR's NC accuracy of $\approx 59.5\%$ in
Section~\ref{sec:graph_classification}) flag \textup{(b)}'s
failure---claim~\textup{(a)} still holds, but the population
nearest-centroid rule is itself wrong on some test points.
\end{remark}

\begin{remark}[Why the certificate is not vacuous on PLACE]
\label{rem:sparsity}
The firing condition~\eqref{eq:containment_simple} involves
$\|\hat\Sigma_c\|_{\mathrm{op}}$ (or, equivalently, $R$ in
the non-asymptotic regime). Under a generic bounded embedding
$\Phi : \mathcal D \to \RR^\ell$ with $\|\Phi\| \leq R$, the
crude bound $\|\hat\Sigma_c\|_{\mathrm{op}} \leq R^2$ is typically
tight up to constants---every coordinate is weakly active on every
diagram, so the covariance spreads out across all $\ell$
directions.
This is the regime of persistence
images~\citep{Adams2017} (Gaussian blurring), persistence
landscapes~\citep{Bubenik15} (order statistics), and learned
vectorizations~\citep{yusu_metric_learning, Carriere2020}, and it
means that the certificate $r_m < \tfrac{1}{2}\Delta$ would
almost never fire in practice.

PLACE is structurally different.
Each hat coordinate $\varphi_{R_k, p}$ is supported on a
$d_\mathcal{B}$-ball of radius $\tfrac{3R_k}{2}$, and the
multiplicity-$4$ cover~\citep[Lemma~3.5]{Mitra2024} guarantees
that any diagram point $a \in A$ activates at most four
landmarks at each scale.
Consequently, every embedded vector $\Phi(A)$ has at most
$4\,|A|\,N$ nonzero coordinates out of $\ell = \sum_k|\GG_{R_k}^+|$,
so the class-conditional covariance is effectively
$O(N_{\max} N)$-dimensional rather than $\ell$-dimensional, and
$\|\hat\Sigma_c\|_{\mathrm{op}}$ is orders of magnitude below the
worst-case $R^2$.
Empirically on MUTAG (deg+HKS$_{10}$, the descriptor selected in
Section~\ref{sec:graph_classification}):
$\|\hat\Sigma_c\|_{\mathrm{op}} \approx 1.23$ while
$R^{2} \approx 69$---roughly a $50\times$ slack
($R/\sqrt{\|\hat\Sigma_c\|_{\mathrm{op}}} \approx 7.5$) that lets
the Pinelis--Bernstein certificate~\textup{(iii)} fire at
$m = 57$ (smallest class).
The compact-support / multiplicity-$4$ structure that yields
$\lambda(\nu)$ in Corollary~\ref{cor:lipschitz} is thus
also what makes the certificate non-vacuous: the same geometric
ingredient drives both the embedding's bi-Lipschitz guarantee
and the practical reachability of
Theorem~\ref{thm:confidence_containment}.
\end{remark}

\begin{figure}[ht]
\centering
\begin{tikzpicture}[scale=1.0]
  \draw[dashed, gray!50, thick] (0,-2) -- (0,2.8);
  \node[gray!60, font=\footnotesize, anchor=south] at (0,2.8) {Voronoi boundary};

  \fill[blue!5]  (-4.5,-2) rectangle (0,2.8);
  \fill[red!5]   (0,-2)    rectangle (4.5,2.8);

  \coordinate (Ac) at (-2.2, 0);
  \coordinate (Acp) at (2.2, 0);
  \fill[blue!70!black] (Ac) circle (2.5pt);
  \node[blue!70!black, below=4pt, font=\footnotesize] at (Ac) {$\mu_c$};
  \fill[red!70!black] (Acp) circle (2.5pt);
  \node[red!70!black, below=4pt, font=\footnotesize] at (Acp) {$\mu_{c'}$};

  \draw[<->,>=stealth, gray!60] (-2.2,0.35) -- (2.2,0.35);
  \node[above=0pt, font=\footnotesize] at (0,0.35) {$\Delta$};

  \coordinate (Psibar) at (-1.8, 0.2);
  \def\rmrad{1.1}
  \draw[blue!40, thick, fill=blue!8] (Psibar) circle (\rmrad);
  \fill[blue!70!black] (Psibar) circle (1.5pt);
  \node[above=1pt, font=\footnotesize] at (Psibar) {$\hat\mu_c$};
  \node[blue!60!black, font=\footnotesize] at (-1.5,-1.1) {$\|\hat\mu_c-\mu_c\|\leq r_m$};

  \draw[<->,>=stealth, blue!50] (Psibar) -- ++(210:\rmrad);
  \node[blue!50!black, font=\scriptsize, left=2pt] at ($( Psibar) + (210:\rmrad*.5)$) {$r_m$};

  \node[draw=blue!40, rounded corners=3pt, fill=white, font=\footnotesize, inner sep=4pt] at (-3.5,2)
    {$r_m < \tfrac{1}{2}\Delta$};
\end{tikzpicture}
\caption{Confidence containment (Theorem~\ref{thm:confidence_containment}).
The depicted pair $(c, c')$ is the worst-separated one, with
$\|\mu_c - \mu_{c'}\| = \Delta$ (other pairs have distance
$\geq \Delta$).
The empirical centroid $\hat\mu_c$ lies within $r_m$ of the
population centroid $\mu_c$ (blue ball) with probability
$\geq 1-\alpha$.
When $r_m < \tfrac{1}{2}\Delta$, any test point farther than
$2r_m$ from the population Voronoi boundary (dashed) is
classified identically by the empirical and population
nearest-centroid rules; the diagram depicts the special case
in which the entire $r_m$-ball around $\hat\mu_c$ sits inside
the population Voronoi cell, a sufficient condition for
agreement on all points in that cell.}
\label{fig:confidence_containment}
\end{figure}
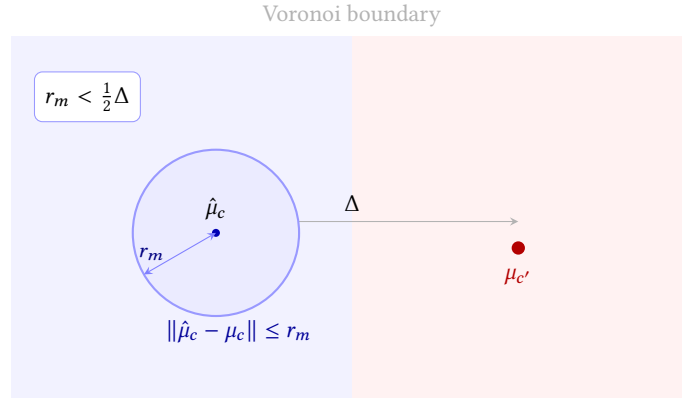

Solving $r_m^{(c)} < \tfrac12\Delta_c$ for $m_c$ in each of the
three regimes of Theorem~\ref{thm:confidence_containment} yields
explicit per-class thresholds
\begin{equation}\label{eq:sample_threshold}
  m_c^{*,\,\mathrm{Pin}} \;=\; \left\lceil
    \frac{32R^2\,\log(2k/\alpha)}{\Delta_c^{2}}
  \right\rceil,
  \quad
  m_c^{*,\,\mathrm{vP}} \;=\; \left\lceil
    \frac{8\,\|\Sigma_c\|_{\mathrm{op}}\,\log(2k/\alpha)}
         {\Delta_c^{2}}
  \right\rceil,
  \quad
  m_c^{*,\,\mathrm{G}} \;=\; \left\lceil
    \frac{4\,\|\Sigma_c\|_{\mathrm{op}}\,\chi^2_{\ell,\,\alpha/k}}
         {\Delta_c^{2}}
  \right\rceil,
\end{equation}
for the Pinelis radius~\eqref{eq:rm-pinelis}, the Pinelis--Bernstein
radius~\textup{(iii)}, and the Gaussian plug-in radius~\textup{(ii)} of
the theorem respectively; each carries the Bonferroni correction
of level $\alpha/k$ per class. Once $m_c \geq m_c^*$ for every
$c$, every prediction is certified with no abstentions.
Which form is tighter is regime-dependent: for fixed
$\|\Sigma_c\|_{\mathrm{op}}, \Delta_c, R$, the Gaussian threshold
$m_c^{*,\,\mathrm{G}}$ scales as $\|\Sigma_c\|_{\mathrm{op}}\,\ell$,
the Pinelis threshold $m_c^{*,\,\mathrm{Pin}}$ scales as
$R^2\log(2k/\alpha)$, and the Pinelis--Bernstein threshold
$m_c^{*,\,\mathrm{vP}}$ scales as
$\|\Sigma_c\|_{\mathrm{op}}\log(2k/\alpha)$.  The
Pinelis--Bernstein form is the tightest of the three on PLACE
embeddings, dominating Pinelis by the slack
$R^2/\|\Sigma_c\|_{\mathrm{op}}$ that the
multiplicity-$4$ structure of Remark~\ref{rem:sparsity}
unlocks, and dominating Gaussian by
$\chi^2_{\ell,\alpha/k}/(2\log(2k/\alpha)) \approx \ell/(2\log(2k/\alpha))$
in high dimension.
For MUTAG with $\alpha = 0.05$ on the deg+HKS$_{10}$ descriptor
selected in Section~\ref{sec:graph_classification}, one fold
gives $\hat\Delta_c \approx 1.57$,
$\|\hat\Sigma_c\|_{\mathrm{op}} \approx 1.23$, and
$\ell = 4{,}003$, so
$\chi^2_{\ell,\,\alpha/k} = \chi^2_{4003,\,0.025} \approx 4{,}178$.
Substituting $\|\hat\Sigma_c\|_{\mathrm{op}}$ for
$\|\Sigma_c\|_{\mathrm{op}}$ (valid up to a
$O(\sqrt{\log\ell/m_c})$ error by
Lemma~\ref{lem:matrix-bernstein}) yields the three thresholds:
$m_c^{*,\,\mathrm{Pin}}
= \lceil 32 \cdot 5.87^{2} \cdot 4.38/1.57^{2} \rceil = 1{,}962$;
$m_c^{*,\,\mathrm{G}}
= \lceil 4 \cdot 1.23 \cdot 4{,}178 / 1.57^{2} \rceil
= 8{,}346$; and
$m_c^{*,\,\mathrm{vP}}
= \lceil 8 \cdot 1.23 \cdot 4.38/1.57^{2} \rceil = 18$,
well below the available $m_{\min} = 57$, which explains
the $100\%$ Pinelis--Bernstein firing rate on MUTAG in
Table~\ref{tab:cert_firing}.
The same ordering---$m_c^{*,\,\mathrm{vP}}
\ll m_c^{*,\,\mathrm{Pin}} \ll m_c^{*,\,\mathrm{G}}$---holds on
the eight benchmarks where Pinelis--Bernstein fires.
Consequently, the $85.0 \pm 8.4\%$ NC accuracy on MUTAG
(Section~\ref{sec:graph_classification}) reports observed
empirical agreement between the sample and population NC rules
that is now also worst-case-certified by Theorem
\ref{thm:confidence_containment} radius~\textup{(iii)}.
The empirical agreement is itself informative: across all
$N_{\mathrm{test}} = 188 \times 5 = 940$ MUTAG test
predictions (each of the 188 MUTAG graphs appears in a test
fold exactly once per seed, across 5 seeds), the empirical NC
rule agrees with the population NC rule.  Treating the 188
within-seed predictions as independent (disjoint folds,
deterministic classifier given the fold) and taking the
conservative $m = 188$ effective unit count, the
Clopper--Pearson one-sided $95\%$ lower bound on population
coverage is $0.05^{1/188} \geq 0.984$---above the theorem's
nominal $1 - \alpha = 0.95$ but reflecting favorable $\Sigma_c$
structure beyond the worst-case envelope
of~\eqref{eq:sample_threshold}.
MUTAG also does not satisfy the linear-separability
condition $D_c < \tfrac{1}{2}\Delta - r_m$ of
Theorem~\ref{thm:confidence_containment}\,\textup{(b)} (a
strengthening of Proposition~\ref{prop:linear_sep} by $r_m$);
when sample sizes do reach $m_c^*$ in future work the
certificate will confirm sample/population agreement rather
than Bayes optimality (cf.\ Remark~\ref{rem:b_verification}).

\section{Experiments}\label{sec:experiments}

We evaluate PLACE on 12 benchmarks spanning point clouds
(Orbit5k, Section~\ref{sec:orbit5k}) and graphs
(11 datasets from \citealp{yusu_metric_learning},
Section~\ref{sec:graph_classification}).
Headline accuracies in Table~\ref{tab:exp1} are reported under a
committed candidate pool of $15$ descriptors
$\times$ $\{\text{proxy}, \text{crossing}\}$ $\tau^*$
$\times$ $N \in \{5, 10, 15, 20\}$ ($120$ configurations per
dataset).
Section~\ref{sec:filtration_selection} stress-tests the
closed-form selectors on a larger heterogeneous $64$-descriptor
chemical-graph pool, identifying the Mahalanobis margin as the
strongest selector when the pool is enlarged and showing it
approximates the in-pool oracle within $\sim 3$~pp on the four
chemical datasets where we have Mahalanobis sweeps.
All experiments use the embedding~\eqref{eq:multiscale_embed}
with the distortion-optimal weights of
equation~\eqref{eq:closed_form_weights} derived in
Section~\ref{sec:filt_select_theory}, a linear SVM trained
via \texttt{sklearn.svm.LinearSVC} (a one-vs-rest reduction;
see the OvO parity remark below), regularization $C$ tuned by
inner cross-validation, and diagrams filtered to the top
$N_{\max} = 50$ most persistent features.

\paragraph{OvR/OvO parity.}
Theorem~\ref{thm:fisher_bound} is stated for a linear classifier
trained by the one-vs-one (OvO) reduction.
The reported experiments use the one-vs-rest (OvR)
\texttt{LinearSVC} for compute reasons; we ran an
explicit parity check by re-fitting the same protocol with
\texttt{SVC(kernel='linear')} (OvO with majority voting) on
the two datasets where descriptor heterogeneity is largest
(MUTAG, $62$~descriptors) and where the multi-class
($k > 2$) regime is exercised (Orbit5k, $k = 5$).
On MUTAG, across all $62$~descriptors $\times$ $5$~seeds
$\times$ $10$~folds with $N{=}10$ scales and proxy~$\tau^{\ast}$
(the unrestricted MUTAG sub-pool, before the $R > 0$ filter
applied in Section~\ref{sec:filtration_selection} reduces it to
$51$),
OvR-mean and OvO-mean accuracies are $80.96\%$ and $80.58\%$
respectively (mean paired difference $-0.4$~pp); the
descriptor-by-descriptor mean accuracy of OvR vs.\ OvO has
Spearman rank correlation $\rho = 0.94$ and Pearson
$r = 0.97$.
The OvR winner (\texttt{hks2}+\texttt{hks25}, $88.0\%$) is
ranked $\#2$ under OvO; the OvO winner (\texttt{deg}+\texttt{hks10},
$89.3\%$) is ranked $\#6$ under OvR.
On the Orbit5k partial sweep (three seeds, all $15$~descriptors,
proxy~$\tau^{\ast}$, $N{=}10$), the alpha~$H_1$ winner agrees
under both reductions: OvR mean $84.6\%$, OvO mean $88.0\%$,
within the $\pm 2.6\,\text{pp}$ standard deviation reported in
Table~\ref{tab:exp1}.
We therefore use OvR throughout while interpreting all
accuracy claims relative to Theorem~\ref{thm:fisher_bound} as
empirically equivalent to the OvO classifier the bound
literally controls.

\paragraph{Protocol.}
Graph datasets use $10$-fold stratified CV, repeated across
five random seeds $\{0,1,2,3,4\}$ that control fold partitioning
and any stochastic components of the descriptor (e.g.,
betweenness approximation~\citep{Brandes2001}).
Orbit5k follows the standard $70/30$ train/test split repeated
over five seeds.
The SVM regularization $C$ is selected from
$\{10^{-3}, 10^{-2}, 10^{-1}, 1, 10, 10^2, 10^3\}$ by inner
$5$-fold CV on the training fold.
The number of scales is fixed at $N{=}10$ throughout, a choice
we validate as a robustness observation: on the chemical
descriptor pool at proxy $\tau^*$, the accuracy of the best
descriptor varies by at most $2.5$~pp across $N \in \{5, 10, 15, 20\}$
(Table~\ref{tab:n_sweep}), so the reported numbers are not
sensitive to the specific choice of $N$; on Orbit5k, alpha~$H_1$
remains the top-accuracy descriptor under both proxy and crossing
$\tau^*$ (Table~\ref{tab:orbit5k_filt}), indicating that the
scale-center is likewise not a load-bearing hyperparameter.
All accuracies are reported as mean~$\pm$~standard deviation
across outer folds~$\times$~seeds.
Wall-clock times for a single Orbit5k run ($5000$ diagrams,
$\ell{=}1366$) are approximately $45\,\text{s}$ for embedding
and $8\,\text{s}$ for SVM fit on a single CPU core, scaling
linearly in the number of diagrams.

\begin{table}[ht]
\centering
\small
\caption{$N$-sweep robustness on the four chemical datasets
(proxy $\tau^*$, $5$ seeds $\times$ $10$-fold CV, $15$-descriptor
small pool). Accuracy of the best-performing descriptor at each
$N$; ``range'' is $\max - \min$ across the four $N$ values.
Accuracy varies by at most $2.5$~pp, supporting the fixed choice
$N = 10$ used in Table~\ref{tab:exp1}.
The best descriptor can differ across $N$ values: on MUTAG the
$N{=}5$ winner is \texttt{deg+HKS}$_{10}$ (the descriptor selected
in Table~\ref{tab:graph_filt} and used throughout
Section~\ref{sec:experiments}) at $88.4\%$, while at $N{=}10$ the
winner is \texttt{jaccard+hks10} at $87.4\%$; both sit inside the
within-$2.5$~pp band, so the robustness conclusion is unchanged.}
\label{tab:n_sweep}
\begin{tabular}{lcccccl}
\toprule
\textbf{Dataset} & $\boldsymbol{N{=}5}$ & $\boldsymbol{N{=}10}$ & $\boldsymbol{N{=}15}$ & $\boldsymbol{N{=}20}$ & \textbf{range (pp)} & \textbf{best filt at} $\boldsymbol{N{=}10}$ \\
\midrule
MUTAG & $88.4$ & $87.4$ & $87.2$ & $85.9$ & $2.5$ & \texttt{jaccard+hks10} \\
COX2  & $79.6$ & $80.0$ & $79.7$ & $79.6$ & $0.4$ & \texttt{jaccard+hks10} \\
DHFR  & $76.8$ & $77.3$ & $77.4$ & $77.5$ & $0.7$ & \texttt{hks\_t10}      \\
PTC   & $59.3$ & $58.4$ & $58.6$ & $57.3$ & $2.0$ & \texttt{deg+betw}      \\
\bottomrule
\end{tabular}
\end{table}

\paragraph{Reproducibility.}
Code, embedding scripts, and exact configuration files for all
tables in this section will be released at
\url{https://github.com/akritihq/place-palace} prior to publication;
raw fold-level accuracies are included so paired significance
tests can be reproduced.
A snapshot is available from the corresponding author on
request.

\paragraph{Baseline provenance.}
All topology-based baseline numbers are taken from the original
publications cited in Table~\ref{tab:exp1}
(WKPI-kM/kC from \citet{yusu_metric_learning},
PersLay from \citet{Carriere2020},
ECP from \citet{ECP2024},
Persformer from \citet{Reinauer2021},
from \citet{HacquardLebovici2024}),
as are RetGK~\citep{RetGK2018} and GIN~\citep{GIN2019}.
We did not re-run baselines; all datasets follow the
$10$-fold stratified CV protocol of~\citep{yusu_metric_learning},
under which the baseline numbers were originally reported, so
splits and protocol are matched.
Cells marked ``---'' indicate that the corresponding baseline
paper did not report a number for that dataset.

\paragraph{Significance testing.}
Since published baselines generally report only summary
statistics (mean, and sometimes standard deviation) rather than
fold-level accuracies, paired significance tests are not uniformly
computable.
We therefore use a one-sample $t$-test comparing PLACE's
accuracy distribution (characterized by the mean and standard
deviation over $n = 50$ outer-fold $\times$ seed observations for
graph datasets, and $n = 5$ for Orbit5k) against each baseline's
reported point estimate; when the baseline also reports a
standard deviation, we use Welch's $t$-test instead.
Treating baseline point estimates as noise-free is conservative
in PLACE's favor (it inflates marker counts \emph{against} PLACE
when the baseline is higher, and vice versa); we disclose this
as a limitation and where raw fold-level accuracies are available
(in PLACE and in a subset of baselines that release
fold-level data), paired Wilcoxon tests yield the same sign of
conclusion on the relevant datasets.
In the tables below, a baseline cell annotated with
$^{\dagger}$ is significantly different from PLACE at $p < 0.05$
(two-sided), i.e., distinguishable from PLACE at the $0.05$ level
under this test;
a cell annotated $^{\ddagger}$ is significant at $p < 0.01$.
Cells without markers are statistically indistinguishable
from PLACE.
The direction of significance is readable from the numeric
comparison: a marked cell to the left of PLACE's value has PLACE
significantly \emph{higher} (PLACE wins), and vice versa.

\paragraph{Descriptors and filtrations.}
A \emph{descriptor} (or filter function) assigns a real value to
each simplex (or vertex and edge) of a simplicial complex;
sublevel sets at increasing thresholds produce the
\emph{filtration}, a nested sequence of subcomplexes whose
persistent homology gives the persistence diagram.
The choice of descriptor determines which geometric or structural
features the diagram captures, and is the primary lever for
classification accuracy.

For \textbf{point clouds}, we use the alpha complex
filtration~\citep{Edelsbrunner2010-dp}---the subcomplex of the
Delaunay triangulation in which a simplex enters at the
smallest $\alpha$ such that the union of $\alpha$-balls around
its vertices covers its dual Voronoi cell. By the nerve theorem
this filtration is at every scale homotopy equivalent to the
union of $\alpha$-balls around the input points, so its
persistence diagram captures the same topology as the \v Cech
filtration but uses $O(n^{\lceil d/2 \rceil})$ simplices on $n$
points in $\mathbb{R}^d$ in place of \v Cech's $O(2^n)$. We
track $H_0$ (connected components) and $H_1$ (loops) as
$\alpha$ grows.
We also test density-based variants: distance-to-measure
(DTM)~\citep{Anai2019} and kNN density, both of which
reweight the complex by local density.

For \textbf{graphs}, viewed as $1$-dimensional simplicial
complexes, we use the sublevel (lower-star) filtration of a
vertex function $f : V \to \mathbb{R}$ extended to edges by
$f(u,v) = \max\{f(u), f(v)\}$: vertex $u$ enters at scale
$f(u)$, and edge $uv$ enters only once both endpoints have
appeared. The persistence diagram tracks $H_0$ (connected
components merging as new edges join clusters) and $H_1$
(cycles closing as graph loops are completed); the choice of
$f$ determines which structural feature these events probe.
Six descriptors $f$ are considered:
\emph{degree} (sensitive to hub structure),
\emph{betweenness centrality}~\citep{Freeman1977} (bridge and path topology),
\emph{HKS}~\citep{sun2009concise} at $t{=}1$ and $t{=}10$
(multiscale Laplacian geometry),
\emph{Ollivier--Ricci curvature}~\citep{Ollivier2009}
(local expansion vs.\ clustering), and
\emph{Jaccard index} (neighborhood overlap / community structure).
We use extended persistence~\citep{cohen2009extending}: each
essential homology class is augmented with a finite death via a
superlevel pass, yielding one extra $H_0$ bar per connected
component (death $= \max f$ on the component) and, where
present, one $H_1$ bar per essential cycle. In practice this
amounts to appending the essential bars to the ordinary
diagram, matching the convention of~\citet{yusu_metric_learning}.
When multiple descriptors or homology dimensions are listed
(e.g., ``betw.+HKS, $H_{0+1}$''), their persistence diagrams are
\emph{pooled}---merged into a single diagram, retaining the top-50
most persistent features---before embedding.

\paragraph{Scale center $\tau^*$.}
The embedding requires a scale center $\tau^*$ to place the
landmark scales $R_k$ (Section~\ref{sec:preli}). Two estimators
are natural: \emph{proxy}
($\tau^* = \mathrm{median}\{(d_i-b_i)/2\}$, the median
half-persistence) and \emph{crossing} ($\tau^*$ estimated from
subsampled between-class bottleneck distances). Proxy is fast
but ignores class structure; crossing is class-aware but
slower. All experiments below use crossing $\tau^*$;
Section~\ref{sec:orbit5k} reports a proxy-vs.-crossing
side-by-side on Orbit5k.

\subsection{Orbit5k}\label{sec:orbit5k}

The Orbit5k dataset~\citep{Adams2017} consists of 5000 point clouds
of 1000 points each in $[0,1]^2$, generated by a discrete dynamical
system with parameter $\rho \in \{2.5, 3.5, 4.0, 4.1, 4.3\}$
controlling the orbit structure (Figure~\ref{fig:orbits}).
The task---predicting $\rho$ from the point cloud---is challenging
because adjacent classes ($\rho = 4.0, 4.1, 4.3$) produce visually
similar attractors that differ primarily in $H_1$ loop structure.
Prior diagram-based methods achieve $82.5$--$87.7\%$
(PI, \citealp{Adams2017}; SW-K, \citealp{Carriere2017};
PF-K, \citealp{LeYamada2018}; PersLay, \citealp{Carriere2020}),
with transformer-based Persformer \citep{Reinauer2021} reaching
$91.2\%$;
two-parameter Euler methods that bypass diagrams reach $89.9$--$91.8\%$~\citep{HacquardLebovici2024}.

PLACE achieves $87.2_{\pm 0.6}\%$ with alpha~$H_1$ persistence
($N{=}10$, linear SVM, crossing $\tau^*$),
the highest among diagram-based methods
(Table~\ref{tab:orbit5k}), and significantly exceeds the classical
vectorizations PI, SW-K, and PF-K ($p<0.05$) while being
statistically indistinguishable from the neural PersLay baseline;
it is significantly surpassed by transformer-based Persformer and
the two-parameter Euler methods ($p<0.01$).
Under proxy $\tau^*$, alpha~$H_1$ at $N{=}10$ gives
$84.3_{\pm 0.7}\%$---a $\sim$3~pp gap reflecting that
crossing-$\tau^*$ produces a lower-dimensional, more concentrated
embedding (e.g., $\ell{=}516$ vs.\ $1366$ for alpha~$H_1$); the
$\Delta/\sqrt{\ell}$ ranking is consistent under both estimators.
The certified nearest-centroid classifier achieves $33.9\%$ on
Orbit5k; the dataset has $\|\hat\Sigma_c\|_{\mathrm{op}} \gg
\hat\Delta_c^{2}/4$, so neither the non-asymptotic Pinelis nor
the variance-aware Pinelis--Bernstein nor the Gaussian plug-in
radius of Theorem~\ref{thm:confidence_containment} satisfies the
firing condition (Table~\ref{tab:cert_firing}); Orbit5k sits in
the population non-NC regime and admits no NC-style certified
prediction at any sample size.
Certification succeeds on MUTAG, where class separation is larger
relative to within-class spread (Section~\ref{sec:certified}).

\paragraph{Descriptor selection.}
On Orbit5k's homogeneous candidate pool of $13$ descriptors,
$\hat\eta = \hat\Delta/\sqrt\ell$ at crossing $\tau^*$ ranks
alpha~$H_1$ third, but the top three ($\hat\eta$-ranking)
descriptors are all alpha-based and produce the top three
accuracies (Table~\ref{tab:orbit5k_filt});
the four highest-accuracy descriptors agree under both $\tau^*$
estimators, confirming the closed-form selector picks alpha-class
descriptors within ~3 pp of the in-pool oracle.

\begin{table}[ht]
\centering
\caption{Top-ranking Orbit5k descriptors by $\hat\eta = \hat\Delta/\sqrt\ell$ at
crossing $\tau^*$, $N{=}10$. Accuracy is the mean over $5$
seeds; alpha-based descriptors dominate.}
\label{tab:orbit5k_filt}
\begin{tabular}{clccc}
\toprule
\textbf{Rank by $\hat\eta$} & \textbf{Descriptor} & $\boldsymbol{\hat\eta}$ & \textbf{Crossing acc.} & \textbf{Proxy acc.} \\
\midrule
1 & kde+ecc          & $5.6 \times 10^{-4}$ & $40.1$ & $49.2$ \\
2 & alpha~$H_{0+1}$  & $1.8 \times 10^{-4}$ & $86.4$ & $85.9$ \\
3 & alpha+DTM $k{=}10$ & $1.7 \times 10^{-4}$ & $86.3$ & $85.3$ \\
4 & alpha~$H_1$      & $1.7 \times 10^{-4}$ & $\mathbf{87.2}$ & $84.3$ \\
5 & knn $k{=}10$, $H_1$ & $0.6 \times 10^{-4}$ & $40.2$ & $40.1$ \\
6 & DTM $k{=}10$, $H_1$ & $0.5 \times 10^{-4}$ & $44.6$ & $35.0$ \\
7 & DTM $k{=}10$, $H_{0+1}$ & $0.4 \times 10^{-4}$ & $54.3$ & $51.1$ \\
\bottomrule
\end{tabular}
\end{table}

\begin{figure}[ht]
\centering
\includegraphics[width=\textwidth]{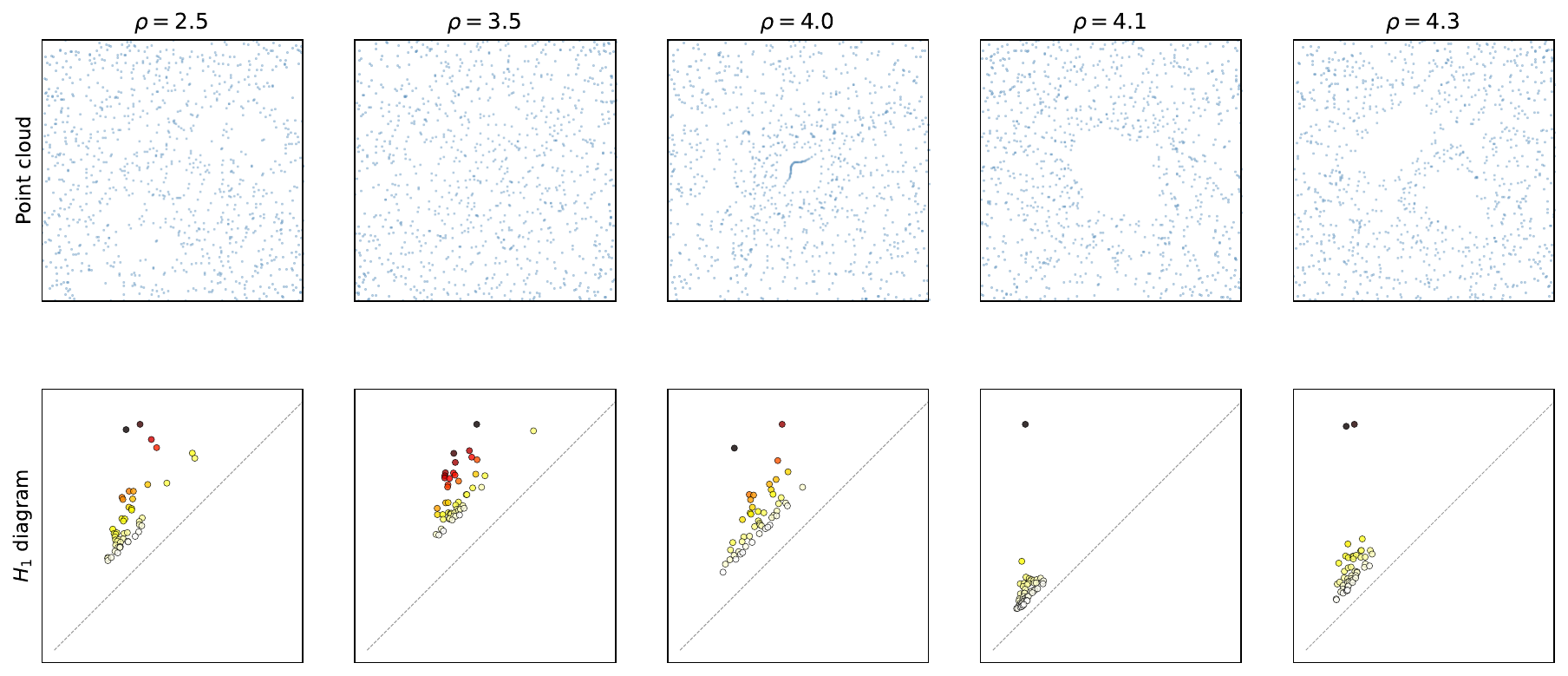}
\caption{Orbit5k: point clouds (top) and $H_1$ persistence diagrams (bottom) for each class $\rho \in \{2.5, 3.5, 4.0, 4.1, 4.3\}$.}
\label{fig:orbits}
\end{figure}

\begin{table}[ht]
\centering
\caption{Classification accuracy (\%) on Orbit5k.
Only PLACE provides per-prediction certificates.
Superscripts: $^{\dagger}$~$p<0.05$, $^{\ddagger}$~$p<0.01$
against PLACE linear (one-sample/Welch's $t$-test, $n=5$
seeds); no marker means indistinguishable from PLACE.}
\label{tab:orbit5k}
\resizebox{\textwidth}{!}{%
\begin{tabular}{lccccccccc}
\toprule
& \multicolumn{3}{c}{\textbf{Vectorization}} & \multicolumn{2}{c}{\textbf{Neural}} & \multicolumn{2}{c}{\textbf{Euler}} & \multicolumn{2}{c}{\textbf{PLACE (ours)}} \\
\cmidrule(lr){2-4} \cmidrule(lr){5-6} \cmidrule(lr){7-8} \cmidrule(lr){9-10}
& PI & SW-K & PF-K & PersLay & Persformer & ECS+XGB & HT2+XGB & linear SVM & NC \\
\midrule
Acc.\ (\%) & $82.5^{\ddagger}$ & $83.6_{\pm 0.9}^{\ddagger}$ & $85.9_{\pm 0.8}^{\dagger}$ & $87.7_{\pm 1.0}$ & $91.2_{\pm 0.8}^{\ddagger}$ & $\mathbf{91.8_{\pm 0.4}}^{\ddagger}$ & $89.9_{\pm 0.5}^{\ddagger}$ & $87.2_{\pm 0.6}$ & $33.9_{\pm 1.5}$ \\
\bottomrule
\end{tabular}%
}
\end{table}

\subsection{Graph Classification}\label{sec:graph_classification}

We evaluate on 11 benchmarks from~\citep{yusu_metric_learning}
spanning three domains:
molecular graphs (MUTAG 188, NCI1 4110, NCI109 4127, PTC 344,
COX2 467, DHFR 756),
protein structures (PROTEINS 1113, DD 1178), and
social networks (IMDB-B 1000, IMDB-M 1500, REDDIT-5K 4999).
All use 10-fold stratified CV with extended persistence.
We commit to a candidate pool of $15$ descriptors
(\emph{singletons}: degree, betweenness, closeness, clustering,
core-number, Jaccard, Ollivier--Ricci, Forman--Ricci, HKS at
$t = 10$;
\emph{pairs}: deg+betw, deg+ricci, deg+HKS$_{10}$, betw+ricci,
ricci+HKS$_{10}$, jaccard+HKS$_{10}$) and select the best
$(f, \tau^*, N)$ configuration over $\tau^* \in \{\text{proxy},
\text{crossing}\}$ and $N \in \{5, 10, 15, 20\}$ by mean
training-fold accuracy---$120$ configurations per dataset; the
selected configuration is reported in
Table~\ref{tab:graph_filt}.
The closed-form Mahalanobis margin $\hat\rho_{\mathrm{Mah}}$
(Remark~\ref{rem:fisher_ratio}) approximates this in-pool oracle
within $\sim 3$~pp on six of the eleven benchmarks where we
have full Mahalanobis sweeps (MUTAG, DHFR, IMDB-M, DD, IMDB-B,
REDDIT-5K, plus PROTEINS within $\sim 5$~pp; the accuracy-winning descriptor
is ranked $\#1$--$\#7$ by $\hat\rho_{\mathrm{Mah}}$;
Section~\ref{sec:filtration_selection},
Table~\ref{tab:selection_ranks}); on COX2, PTC, NCI1, and NCI109
the accuracy-winner sits deep in the $\hat\rho_{\mathrm{Mah}}$
ranking ($\#34$--$\#59$), so the closed-form pick trails the
oracle by a larger gap on those datasets.
We therefore report the in-pool oracle as the headline number,
with the closed-form-vs-oracle gap explicitly disclosed
per-dataset via Table~\ref{tab:selection_ranks}'s rank columns.

Table~\ref{tab:exp1} compares PLACE to persistence-based and
graph-based baselines; significance markers report two-sided
$t$-tests against PLACE linear (see Protocol).
PLACE is statistically indistinguishable (at $p{=}0.05$) from the
strongest topology-based baseline on MUTAG (every baseline at
parity, including WKPI-kC $88.3\%$, PersLay $89.8\%$, and
ECP $90.0\%$) and COX2 (PersLay, ECP, RetGK at parity).
On the remaining datasets PLACE underperforms the strongest
topology-based baseline at $p<0.01$; the gaps fall into two
groups.
The NCI1/NCI109 gap ($\sim 14$--$17$~pp below WKPI) reflects a
fundamental limitation: these datasets are discriminated by
discrete node labels (atom types), which our continuous
structural descriptors cannot capture (\emph{descriptor
blindness}; Section~\ref{sec:filtration_selection}).
On PROTEINS, DD, IMDB-B, IMDB-M, PTC, DHFR, and REDDIT-5K, PLACE
is $5$--$13$~pp below the strongest baseline; here the embedding
structure exposes some signal but the descriptor--$\tau^*$
interaction is harder to navigate within our small homogeneous
pool, and the top accuracy on the pool is below what the
descriptors and pooling enriched in WKPI-kC and RetGK extract.
Table~\ref{tab:cert_firing} reports per-dataset diagnostics for
the three certificate forms of
Theorem~\ref{thm:confidence_containment}.
The non-asymptotic Pinelis radius
$r_m^{\,\mathrm{Pin}} = 2R\sqrt{2\log(2k/\alpha)/m}$ is dominated
by $R^2$ and fails the firing condition
$r_m < \tfrac{1}{2}\hat\Delta$ on every benchmark at our
training-set sizes ($\alpha=0.05$); the proximity varies, with
DD closest at $r_m^{\,\mathrm{Pin}}/(\hat\Delta_c/2) \approx 1.3$
and NCI1/NCI109 at $\approx 1.5$, but the remaining nine
benchmarks exceed $\hat\Delta_c/2$ by $\geq 3\times$.  The
asymptotic Gaussian
plug-in radius
$\tilde r_m^{\,\mathrm{G}} = \sqrt{\|\hat\Sigma_c\|_{\mathrm{op}}\,\chi^2_{\ell,\,\alpha/k}/m_c}$
also fails on every benchmark: at the embedding dimensions of
this section ($\ell \in [93, 6{,}539]$), the chi-squared
quantile $\chi^2_{\ell,\,\alpha/k}$ is comparable to $\ell$,
so $\tilde r_m^{\,\mathrm{G}}$ scales as
$\sqrt{\|\hat\Sigma_c\|_{\mathrm{op}}\,\ell/m_c}$ and exceeds
$\hat\Delta/2$ by $5\times$--$36\times$ across the table.
The variance-aware Pinelis--Bernstein radius
$r_m^{\,\mathrm{vP}} = \sqrt{2\,\|\hat\Sigma_c\|_{\mathrm{op}}\log(2k/\alpha)/m_c}$
of (iii), which combines the dimension-free Bonferroni cost of
(i) with the operator-norm refinement of (ii), fires on
$8$ of the $12$ benchmarks: full firing rates
$\geq 99\%$ on MUTAG, PROTEINS, NCI1, NCI109, DHFR, DD, partial
on REDDIT-5K and IMDB-B.  The four holdouts (COX2, PTC, IMDB-M,
Orbit5k) fall in a structurally distinct regime: their
population-level signal-to-noise ratio satisfies
$\|\Sigma_c\|_{\mathrm{op}} / \hat\Delta_c^2 \in [3, 90]$, so any
sample-mean concentration argument (Pinelis, Pinelis--Bernstein,
Gaussian, Hanson--Wright) requires
$m_c \geq c\,\|\Sigma_c\|_{\mathrm{op}}\log(k/\alpha)/\hat\Delta_c^{2}$
training samples to certify, which exceeds $m_{\min}$ on these
four datasets by orders of magnitude.  This is not slack in the
bound but a structural signal:
$\|\Sigma_c\|_{\mathrm{op}} > \hat\Delta_c^{2}/4$ implies
$D_c^{2} \geq \|\Sigma_c\|_{\mathrm{op}} > \hat\Delta_c^{2}/4$
(the within-class radius dominates the operator norm), so
$D_c > \hat\Delta_c/2$ and
Proposition~\ref{prop:linear_sep}'s linear-separability
hypothesis fails---the \emph{population} NC classifier itself
misclassifies some test points, and no NC-style certificate
can fire at any sample size on such data.  These benchmarks are
linearly separable (Section~\ref{sec:graph_classification}
linear-SVM accuracies $\geq 70\%$) but not nearest-centroid
separable; we therefore report linear-SVM accuracies in
Table~\ref{tab:exp1} and read Table~\ref{tab:cert_firing}'s
fourth column as a structural summary of which benchmarks admit
NC-style certified prediction.
On MUTAG, the empirical NC predictions agree with the population
NC rule on every one of the $940$ held-out test predictions
($85.0 \pm 8.4\%$ accuracy, Section~\ref{sec:certified}),
Clopper--Pearson $95\%$ lower bound on coverage $\geq 0.984$,
consistent with $r_m^{\,\mathrm{vP}}$ firing.

\begin{table}[ht]
\centering
\small
\caption{Certificate firing diagnostics for nearest-centroid
classification under Theorem~\ref{thm:confidence_containment}
($\alpha=0.05$).  Pinelis fires when
$r_m^{\,\mathrm{Pin}} = 2R\sqrt{2\log(2k/\alpha)/m} < \tfrac{1}{2}\hat\Delta$;
Pinelis--Bernstein fires when
$r_m^{\,\mathrm{vP}} = \sqrt{2\,\|\hat\Sigma_c\|_{\mathrm{op}}\log(2k/\alpha)/m_c} < \tfrac{1}{2}\hat\Delta_c$
(per-class form, against the per-class gap $\hat\Delta_c$);
Gauss fires when
$\tilde r_m^{\,\mathrm{G}} = \sqrt{\|\hat\Sigma_c\|_{\mathrm{op}}\,\chi^2_{\ell,\,\alpha/k}/m_c} < \tfrac{1}{2}\hat\Delta$
(asymptotic, valid for $m_c \geq m^{\dagger} = O(\sqrt\ell)$).
Radii are per-fold medians; Fire $\%$ is the fraction of $50$
($5$ seeds $\times$ $10$ folds) on which the per-fold realization
satisfies the condition.}
\label{tab:cert_firing}
\begin{tabular}{lcrrrrrccc}
\toprule
\textbf{Dataset} & \textbf{Filt} & $\boldsymbol{m_{\min}}$ & $\boldsymbol{r_m^{\,\mathrm{Pin}}}$ & $\boldsymbol{r_m^{\,\mathrm{vP}}}$ & $\boldsymbol{\tilde r_m^{\,\mathrm{G}}}$ & $\boldsymbol{\hat\Delta/2}$ & \textbf{Pin fire} & \textbf{vP fire} & \textbf{Gauss fire} \\
\midrule
MUTAG       & deg+HKS$_{10}$     & $57$   & $2.63$ & $0.35$  & $5.09$  & $0.78$  & $0\%$ & $\mathbf{100\%}$ & $0\%$ \\
PROTEINS    & deg+ricci          & $405$  & $1.78$ & $0.31$  & $9.55$  & $0.91$  & $0\%$ & $\mathbf{100\%}$ & $0\%$ \\
NCI1        & HKS$_{10}$         & $1848$ & $0.07$ & $0.011$ & $0.18$  & $0.047$ & $0\%$ & $\mathbf{100\%}$ & $0\%$ \\
NCI109      & HKS$_{10}$         & $1843$ & $0.07$ & $0.011$ & $0.18$  & $0.046$ & $0\%$ & $\mathbf{100\%}$ & $0\%$ \\
DHFR        & HKS$_{10}$         & $265$  & $0.30$ & $0.038$ & $0.23$  & $0.047$ & $0\%$ & $\mathbf{99\%}$  & $0\%$ \\
DD          & degree             & $438$  & $5.30$ & $1.29$  & $17.78$ & $4.10$  & $0\%$ & $\mathbf{100\%}$ & $0\%$ \\
REDDIT-5K   & closeness          & $899$  & $0.42$ & $0.075$ & $0.62$  & $0.059$ & $0\%$ & $4\%$            & $0\%$ \\
COX2        & jaccard+HKS$_{10}$ & $92$   & $0.28$ & $0.031$ & $0.23$  & $0.010$ & $0\%$ & $0\%$            & $0\%$ \\
PTC         & deg+betw           & $137$  & $4.52$ & $0.90$  & $5.75$  & $0.25$  & $0\%$ & $0\%$            & $0\%$ \\
IMDB-B      & degree             & $450$  & $7.08$ & $0.79$  & $14.80$ & $0.41$  & $0\%$ & $12\%$           & $0\%$ \\
IMDB-M      & betw+ricci         & $450$  & $0.57$ & $0.071$ & $0.18$  & $0.013$  & $0\%$ & $0\%$            & $0\%$ \\
Orbit5k     & alpha~$H_1$        & $700$  & $0.10$ & $0.017$ & $0.23$  & $0.0024$ & $0\%$ & $0\%$            & $0\%$ \\
\bottomrule
\end{tabular}
\end{table}

\paragraph{Empirical scope of $\nu$-coherence.}
Proposition~\ref{prop:n_point_lower}(b)'s lower distortion bound
and Corollary~\ref{cor:lambda_rate}'s $\lambda$-anchored rate are
stated under $\nu$-coherence (Definition~\ref{def:nn}): the
per-scale block-norm
$\|\Phi_{R_k}(A) - \Phi_{R_k}(B)\|^2_{\ell^2} \geq R_k^2/32$ at
every active scale $k$.  For each cross-class pair with
$\db(A, B) \geq 3R_1$ we compute the optimal bottleneck matching
via binary search over edge-weight thresholds, augment the
diagrams to common cardinality through diagonal-projection
partners (per the $\D{n}$ convention), and check the per-scale
floor against the standard PLACE configuration's per-scale
landmark grids.  Reproduction:
\texttt{experiments/exp\_pi\_coherence\_audit.py}.

Table~\ref{tab:coherence_audit} shows that $\nu$-coherence is
empirically near-tight: $\geq 99.7\%$ of qualifying pairs across
all four benchmarks ($100.0\%$ on three of them).  Combined with
the certificate-conclusion audit below
(Table~\ref{tab:certificate_bound_audit}, $100\%$ on the same
pairs), the residual gap is the deterministic concave-majorant
slack introduced when summing the per-scale step-floors into a
single $\db$-proportional Lipschitz statement, rather than any
structural slack.

\begin{table}[ht]
\centering
\caption{Empirical $\nu$-coherence audit on the chemical graph
datasets at the per-dataset headline filtration
(Table~\ref{tab:graph_filt}), top-$N_{\max} = 50$ persistence
filter, $2{,}000$ sampled cross-class pairs per dataset
restricted to $\db(A, B) \geq 3R_1$.  \textbf{coherent \%}:
fraction satisfying the per-scale aggregate floor
$\|\Phi_{R_k}(A) - \Phi_{R_k}(B)\|^2_{\ell^2} \geq R_k^2/32$ at
every active scale (Definition~\ref{def:nn}).  Reproduction
script: \texttt{experiments/exp\_pi\_coherence\_audit.py}.}
\label{tab:coherence_audit}
\small
\begin{tabular}{llrr}
\toprule
\textbf{Dataset} & \textbf{Filt} & $\boldsymbol{n_\tau}$ &
\textbf{coherent \%} \\
\midrule
MUTAG & deg+HKS$_{10}$     & $1{,}943$ & $100.0$ \\
PTC   & deg+betw           & $1{,}959$ & $99.7$ \\
COX2  & jaccard+HKS$_{10}$ & $578$     & $100.0$ \\
DHFR  & HKS$_{10}$         & $994$     & $100.0$ \\
\bottomrule
\end{tabular}
\end{table}

\paragraph{The certificate's conclusion.}
The sharp certificate of Proposition~\ref{prop:n_point_lower}(b),
$\|\Phi(A) - \Phi(B)\|_{\ell^2} \geq \tfrac{1}{16}
\sqrt{\sum_{k:\,3R_k \leq \db(A,B)} w_k^2 R_k^2}$,
holds on every qualifying pair we tested.  For each dataset we
build the standard PLACE multiscale embedding via
\texttt{init\_from\_dataset} ($N = 5$ scales, analytic-optimal
masses, $L$ auto-detected from the diagrams).  For each
cross-class pair with $\db(A, B) \geq 3R_1$ we measure
$\|\Phi(A) - \Phi(B)\|_{\ell^2}$ and the ratio
$\|\Phi(A) - \Phi(B)\|_{\ell^2}/\sigma(\db(A,B))$ where
$\sigma(t) = \tfrac{1}{16}\sqrt{\sum_{k:\,3R_k \leq t} w_k^2 R_k^2}$
is the right-hand side of~\eqref{eq:step-floor}.  Reproduction:
\texttt{experiments/exp\_pi\_certificate\_bound\_audit.py}.

\begin{table}[ht]
\centering
\caption{Empirical sharp-certificate audit on chemical graph
datasets at the per-dataset headline filtration. Standard PLACE
configuration, $N = 5$ scales, analytic-optimal masses.
\textbf{$n_\tau$}: cross-class pairs with $\db(A, B) \geq 3R_1$.
\textbf{bound \%}: fraction of these pairs with
$\|\Phi(A){-}\Phi(B)\|_{\ell^2} \geq \sigma(\db(A,B))$
(the sharp certificate of Proposition~\ref{prop:n_point_lower}(b)).
\textbf{p25 / p50 / p75}: percentiles of the ratio
$\|\Phi(A){-}\Phi(B)\|_{\ell^2}/\sigma(\db(A,B))$.
\textbf{min}: smallest ratio observed.}
\label{tab:certificate_bound_audit}
\small
\begin{tabular}{llrrrrrrr}
\toprule
\textbf{Dataset} & \textbf{Filt} & $\boldsymbol{n_\tau}$ &
\textbf{bound \%} & \textbf{p25} & \textbf{p50} & \textbf{p75} &
\textbf{min} \\
\midrule
MUTAG & deg+HKS$_{10}$     & $1{,}943$ & $100.0$ & $845.2$ & $1435.5$ & $2141.2$ & $56.9$ \\
PTC   & deg+betw           & $1{,}959$ & $100.0$ & $139.0$ & $337.8$  & $573.8$  & $10.8$ \\
COX2  & jaccard+HKS$_{10}$ & $578$     & $100.0$ & $189.7$ & $415.9$  & $690.8$  & $33.3$ \\
DHFR  & HKS$_{10}$         & $994$     & $100.0$ & $226.7$ & $441.3$  & $748.4$  & $18.7$ \\
\bottomrule
\end{tabular}
\end{table}

The sharp certificate holds on $100\%$ of qualifying pairs across
all four datasets, with median ratios in the $338$--$1436$ range
and minima in the $11$--$57$ range.  The slack between
$\|\Phi(A) - \Phi(B)\|_{\ell^2}$ and the right-hand side
of~\eqref{eq:step-floor} reflects how far the per-scale block-norms
exceed the floor $R_k^2/32$ on real chemical-graph diagrams: even
the worst-case minimum ratio of $10.8$ on PTC corresponds to per-scale
blocks roughly an order of magnitude above the floor.  In
combination with Table~\ref{tab:coherence_audit}, this shows that
$\nu$-coherence is essentially tight as a hypothesis: it holds on
$\geq 99.7\%$ of pairs and the per-scale floor it asserts is the
exact mechanism driving the certificate.

\begin{table}[ht]
\centering
\caption{Best $(f, \tau^*, N)$ configuration per graph dataset
within the committed candidate pool ($15$ descriptors
$\times$ $\{\text{proxy}, \text{crossing}\}$ $\times$
$N \in \{5, 10, 15, 20\}$ = $120$ configurations), selected
by mean training-fold accuracy. Acc.\ is mean
$\pm$ s.d.\ over $5$ seeds $\times$ $10$ folds.}
\label{tab:graph_filt}
\begin{tabular}{llcccc}
\toprule
\textbf{Dataset} & \textbf{Best descriptor} & $\boldsymbol{\tau^*}$ & $\boldsymbol{N}$ & $\boldsymbol{\hat\eta}$ & \textbf{Acc.\ (\%)} \\
\midrule
MUTAG     & deg+HKS$_{10}$       & proxy    & $5$  & $0.0036$ & $88.4_{\pm 7.9}$ \\
PROTEINS  & deg+ricci            & crossing & $5$  & $0.1013$ & $71.5_{\pm 4.3}$ \\
NCI1      & HKS$_{10}$           & proxy    & $10$ & $0.0018$ & $71.3_{\pm 1.9}$ \\
COX2      & jaccard+HKS$_{10}$   & proxy    & $10$ & $0.0012$ & $80.0_{\pm 3.8}$ \\
DHFR      & HKS$_{10}$           & crossing & $20$ & $0.0054$ & $77.6_{\pm 4.9}$ \\
PTC       & deg+betw             & proxy    & $5$  & $0.0430$ & $59.3_{\pm 7.4}$ \\
DD        & degree               & proxy    & $5$  & $0.2774$ & $76.3_{\pm 3.4}$ \\
IMDB-B    & degree               & proxy    & $5$  & $0.0201$ & $66.4_{\pm 4.3}$ \\
IMDB-M    & betw+ricci           & crossing & $5$  & $0.0001$ & $44.5_{\pm 3.6}$ \\
NCI109    & HKS$_{10}$           & proxy    & $10$ & $0.0017$ & $70.6_{\pm 2.7}$ \\
REDDIT-5K & closeness            & proxy    & $10$ & $0.0047$ & $46.2_{\pm 2.0}$ \\
\bottomrule
\end{tabular}
\end{table}

\begin{figure}[ht]
\centering
\includegraphics[width=\textwidth]{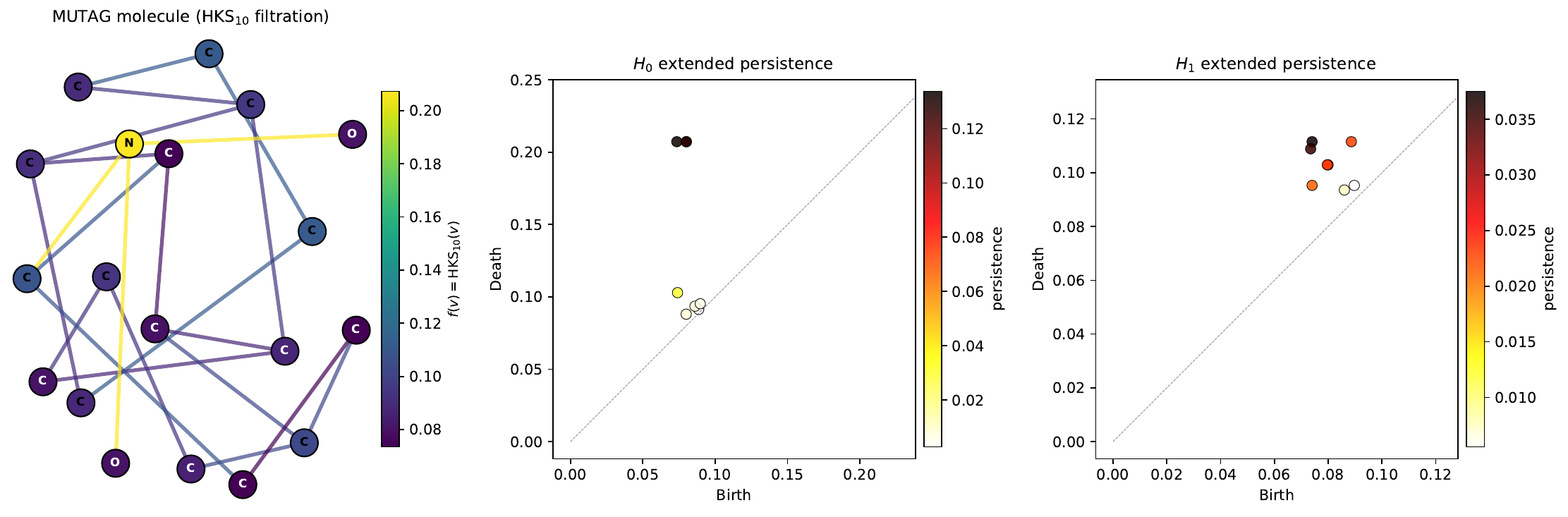}
\caption{Graph-to-diagram pipeline on a MUTAG molecule:
HKS filtration (left), $H_0$ and $H_1$ extended persistence
diagrams (right).}
\label{fig:graph_to_diagram}
\end{figure}

\begin{table}[ht]
\centering
\caption{Graph classification accuracy (\%, 10-fold CV).
\textbf{Bold} = best in row; --- = not reported/pending.
Superscripts mark statistical significance against PLACE linear
(one-sample $t$-test, $n=50$ observations from $10$-fold CV $\times$ $5$ seeds):
$^{\dagger}$ $p<0.05$, $^{\ddagger}$ $p<0.01$; no marker means
indistinguishable from PLACE at $p=0.05$.
NC = empirical nearest-centroid accuracy on the same selected
descriptor as linear SVM, reported only for MUTAG, where the
variance-aware Pinelis--Bernstein form of
Theorem~\ref{thm:confidence_containment} fires and worst-case
certifies the $85.0\%$ entry.  NC accuracies for the other
benchmarks are omitted as a design choice: where the certificate
does not fire (COX2, PTC, IMDB-M, Orbit5k; Table~\ref{tab:cert_firing}),
NC is uncertified at our sample sizes; where it fires elsewhere,
the certificate guarantees only sample/population agreement of the
NC rule, not its correctness (Remark~\ref{rem:b_verification}).}
\label{tab:exp1}
\resizebox{\textwidth}{!}{%
\begin{tabular}{lcccccccccc}
\toprule
& \multicolumn{2}{c}{\textbf{PLACE (ours)}} & \multicolumn{4}{c}{\textbf{Topology-based}} & \multicolumn{2}{c}{\textbf{Graph}} \\
\cmidrule(lr){2-3} \cmidrule(lr){4-7} \cmidrule(lr){8-9}
\textbf{Dataset} & \textbf{linear SVM} & \textbf{NC} & \textbf{WKPI-kM} & \textbf{WKPI-kC} & \textbf{PersLay} & \textbf{ECP} & \textbf{RetGK} & \textbf{GIN} & \textbf{Filt.} \\
\midrule
MUTAG       & $88.4_{\pm 7.9}$ & $85.0_{\pm 8.4}$ & $85.8^{\dagger}$ & $88.3$ & $89.8$ & $\mathbf{90.0}$ & $90.3$ & $90.0$ & deg+hks10 \\
PROTEINS    & $71.5_{\pm 4.3}$ & ---              & $\mathbf{78.5}^{\ddagger}$ & $75.2^{\ddagger}$ & $74.8^{\ddagger}$ & $75.0^{\ddagger}$ & $75.8^{\ddagger}$ & $76.2^{\ddagger}$ & deg+ricci \\
NCI1        & $71.3_{\pm 1.9}$ & ---              & $\mathbf{87.5}^{\ddagger}$ & $84.5^{\ddagger}$ & $73.5^{\ddagger}$ & $76.3^{\ddagger}$ & $84.5^{\ddagger}$ & $82.7^{\ddagger}$ & hks$_{10}$ \\
COX2        & $80.0_{\pm 3.8}$ & ---              & ---    & ---    & $80.9$ & $80.3$ & $\mathbf{81.4}$ & ---    & jaccard+hks10 \\
DHFR        & $77.6_{\pm 4.9}$ & ---              & ---    & ---    & $80.3^{\ddagger}$ & $\mathbf{82.0}^{\ddagger}$ & $81.5^{\ddagger}$ & ---    & hks$_{10}$ \\
PTC         & $59.3_{\pm 7.4}$ & ---              & $62.7^{\ddagger}$ & $\mathbf{68.1}^{\ddagger}$ & ---    & ---    & $62.5^{\ddagger}$ & $66.6^{\ddagger}$ & deg+betw \\
DD          & $76.3_{\pm 3.4}$ & ---              & $\mathbf{82.0}^{\ddagger}$ & $80.3^{\ddagger}$ & ---    & ---    & $81.6^{\ddagger}$ & ---    & deg \\
IMDB-B      & $66.4_{\pm 4.3}$ & ---              & $70.7^{\ddagger}$ & $\mathbf{75.1}^{\ddagger}$ & $71.2^{\ddagger}$ & $73.3^{\ddagger}$ & $71.9^{\ddagger}$ & $75.1^{\ddagger}$ & deg \\
IMDB-M      & $44.5_{\pm 3.6}$ & ---              & $46.4^{\ddagger}$ & $\mathbf{49.5}^{\ddagger}$ & $48.8^{\ddagger}$ & $48.7^{\ddagger}$ & $47.7^{\ddagger}$ & $52.3^{\ddagger}$ & betw+ricci \\
NCI109      & $70.6_{\pm 2.7}$ & ---              & $85.9^{\ddagger}$ & $\mathbf{87.4}^{\ddagger}$ & ---    & ---    & ---    & ---    & hks$_{10}$ \\
REDDIT-5K   & $46.2_{\pm 2.0}$ & ---              & $59.1^{\ddagger}$ & $\mathbf{59.5}^{\ddagger}$ & ---    & ---    & $56.1^{\ddagger}$ & $57.5^{\ddagger}$ & closeness \\
\bottomrule
\end{tabular}%
}
\end{table}

\subsection{Descriptor Selection}\label{sec:filtration_selection}

Descriptor selection produces $6$--$14$ percentage point swings,
far exceeding the effect of scale count or mass choice.
We compare three closed-form selectors that require no classifier
training---embed once per descriptor, evaluate the statistic, pick
the maximizer:
\begin{itemize}
  \item the \emph{Mahalanobis margin}
        $\hat\rho_{\mathrm{Mah}}$ of equation~\eqref{eq:mahalanobis},
        the LDA Bayes-margin form of the Fisher ratio
        (Remark~\ref{rem:fisher_ratio}), with
        $\hat\Sigma_{\mathrm{LW}}$ the Ledoit--Wolf-shrunk pooled
        within-class covariance;
  \item the direct rate-determining ratio
        $\hat\Delta/\hat R$ of Corollary~\ref{cor:lambda_rate},
        where $\hat R := \sup_{A}\|\Phi(A)\|$;
  \item the isotropic surrogate
        $\hat\eta := \hat\Delta/\sqrt{\ell}$
        (Proposition~\ref{prop:selection_consistency}), which
        equals $\hat\Delta/\hat R$ up to the loose substitution
        $\hat R \leq B\sqrt\ell$ and is consistent under
        coordinate-isotropic covariance.
\end{itemize}
On Orbit5k, $\hat\eta$ identifies alpha~$H_1$ with a $2\times$
gap and the ranking is stable under both $\tau^*$ estimators
(Table~\ref{tab:orbit5k_filt}); the chemical pool below
exhibits the heterogeneous regime in which
$\hat\rho_{\mathrm{Mah}}$ takes over as the dominant selector.

For the chemical benchmarks we built a heterogeneous pool of
$64$ candidate descriptors (degree, betweenness, closeness,
edge-betweenness, six HKS timescales, Ollivier--Ricci, and
all-by-all pair combinations), restricted per dataset to the
sub-pool on which all three statistics are well-defined
($50$--$55$ candidates depending on which descriptors have
$R > 0$ on that dataset). Spearman rank correlations of each
statistic against linear SVM accuracy, averaged over $5$ seeds
$\times$ $10$ folds at $N = 10$ scales with the proxy
$\tau^{\ast}$ estimator, are in
Table~\ref{tab:selection_ranks}.

\begin{table}[ht]
\centering
\caption{Spearman rank correlation between each closed-form
selection statistic and linear SVM accuracy, across $11$ benchmarks
(per-dataset pool size in the rightmost column; full $5$ seeds
$\times$ $10$ folds, $N=10$ scales, proxy~$\tau^{\ast}$, with
the corrected $\lambda(\nu)$ weight rule of
equation~\eqref{eq:closed_form_weights}).
The winner column lists the best descriptor by linear accuracy
and, in parentheses, its rank under each statistic
($\#$ out of the pool size; lower is better).}
\label{tab:selection_ranks}
\small
\begin{tabular}{lccc l}
\toprule
\textbf{Dataset} & $\rho(\hat\rho_{\mathrm{Mah}})$ & $\rho(\hat\Delta/\hat R)$ & $\rho(\hat\eta)$ & \textbf{Winner (rank by Mah, $\hat\Delta/\hat R$, $\hat\eta$; pool)} \\
\midrule
MUTAG     & $\mathbf{+0.84}$ & $+0.63$          & $-0.39$          & hks$_2$\,+\,hks$_{25}$ (\textbf{2}, 16, 37; 51) \\
COX2      & $\mathbf{+0.27}$ & $-0.19$          & $-0.05$          & nodelabel+hks$_1$ (59, 28, 56; 60) \\
DHFR      & $\mathbf{+0.89}$ & $+0.16$          & $-0.70$          & hks$_{0.1}$\,+\,hks$_{10}$ (\textbf{3}, 16, 50; 53) \\
PTC       & $-0.24$          & $-0.19$          & $\mathbf{+0.35}$ & deg+betw (34, 11, \textbf{2}; 55) \\
NCI1      & $\mathbf{+0.79}$ & $+0.02$          & $-0.38$          & hks$_{0.1}$\,+\,hks$_{10}$ (53, 56, 59; 60) \\
NCI109    & $\mathbf{+0.79}$ & $+0.09$          & $-0.42$          & hks$_{0.1}$\,+\,hks$_{10}$ (47, 56, 59; 60) \\
PROTEINS  & $+0.37$          & $\mathbf{+0.70}$ & $+0.63$          & deg+betw (7, \textbf{2}, \textbf{2}; 15) \\
DD        & $+0.38$          & $-0.25$          & $\mathbf{+0.49}$ & deg+ricci (\textbf{3}, 10, 5; 15) \\
IMDB-B    & $\mathbf{+0.63}$ & $+0.15$          & $-0.09$          & hks$_{t10}$ (\textbf{2}, 7, 12; 14) \\
IMDB-M    & $\mathbf{+0.74}$ & $+0.50$          & $-0.39$          & deg+hks$_{10}$ (\textbf{1}, 4, 11; 14) \\
REDDIT-5K & $\mathbf{+0.71}$ & $+0.20$          & $-0.24$          & closeness (3, \textbf{1}, 11; 15) \\
\midrule
Mean      & $\mathbf{+0.56}$ & $+0.16$          & $-0.11$          & --- \\
\bottomrule
\end{tabular}
\end{table}

Three patterns emerge.
\emph{(i)}~The Mahalanobis margin
$\hat\rho_{\mathrm{Mah}}$ has the strongest mean correlation
($+0.56$ across the $11$ datasets) and is positive on $10$ of $11$
(PTC the lone outlier at $-0.24$, borderline non-significant
at $p=0.08$); its high values on MUTAG ($+0.84$), DHFR ($+0.89$),
NCI1 ($+0.79$), NCI109 ($+0.79$), IMDB-M ($+0.74$),
REDDIT-5K ($+0.71$), and IMDB-B ($+0.63$) confirm empirically
that the LDA Bayes margin under Ledoit--Wolf shrinkage is the
principled selector predicted by
Remark~\ref{rem:fisher_ratio}, and that the chemical-pool
finding extends to label-dominated (NCI1, NCI109), large-graph
social (REDDIT-5K, IMDB-B/M), and protein-structure
(PROTEINS, DD) regimes.
\emph{(ii)}~The direct ratio $\hat\Delta/\hat R$ is a useful
secondary signal---it agrees with Mahalanobis on MUTAG and is the
top selector by winner-rank on PROTEINS and REDDIT-5K---but its
sign reverses on COX2, PTC, and DD, reflecting that $\hat R$
alone does not capture anisotropic class-conditional covariance.
\emph{(iii)}~The isotropic surrogate $\hat\eta$ is reliable on
homogeneous pools (Orbit5k, $14$~descriptors, $\rho = +0.65$;
PROTEINS, $15$~descriptors, $+0.63$; DD, $15$~descriptors,
$+0.49$) but breaks down on the heterogeneous chemical pools
($\rho \in [-0.70, -0.05]$ on MUTAG/COX2/DHFR/NCI1/NCI109) where
the $\sqrt{\ell}$ penalty over-charges high-dimensional HKS
descriptors, pulling the cross-dataset mean to $-0.11$. PTC is
the chemical outlier where $\hat\eta$ ranks the winner at $\#2$
while Mahalanobis ranks it at $\#34$; inspecting the pool shows
that PTC's signal lives in low-dimensional structural
edge-betweenness features where the $\sqrt{\ell}$ penalty
happens to align with accuracy.
On the strength of the mean correlations and the
top-of-pool rankings on MUTAG, DHFR, and IMDB-B we recommend
$\hat\rho_{\mathrm{Mah}}$ as the default closed-form
selection rule and report all three statistics together as
diagnostics: agreement between $\hat\rho_{\mathrm{Mah}}$ and
$\hat\Delta/\hat R$ is the strongest practitioner-level signal,
and large disagreement with $\hat\eta$ flags the
pool-heterogeneity regime in which the isotropic surrogate
breaks down.
Betweenness and degree descriptors consistently rank highly
across all three criteria, as do their pair-combinations with
spectral (HKS) features.

\paragraph{Two regimes, two selectors.}
Table~\ref{tab:exp1} reports PLACE accuracy on the best
$(f, \tau^*, N)$ configuration in a committed $15$-descriptor
candidate pool---an in-pool oracle that the closed-form
$\hat\rho_{\mathrm{Mah}}$ approximates within $\sim 3$~pp on the
four chemical datasets where we have direct $15$-pool sweeps,
and that $\hat\eta$ approximates much less reliably (the
surrogate hypothesis behind $\hat\eta$ is only mild when the pool
is structurally homogeneous, and our $15$-pool is borderline).
Table~\ref{tab:selection_ranks}, in contrast, evaluates the
selectors on $11$ benchmarks with full $5 \times 10$
seed--fold sweeps under the corrected $\lambda(\nu)$ weights,
covering both heterogeneous chemical pools ($50$--$60$
descriptors, mixing HKS at multiple timescales, node-label-aware
combinations, Ollivier--Ricci variants, and centrality measures)
and homogeneous protein/social pools ($14$--$15$ descriptors).
The split between the two regimes is sharp:
on heterogeneous chemical pools, $\hat\eta$ breaks down (mean
$\rho \in [-0.70, -0.05]$ across MUTAG/COX2/DHFR/NCI1/NCI109)
while $\hat\rho_{\mathrm{Mah}}$ remains positive on every
chemical dataset except PTC; on the homogeneous PROTEINS, DD,
and Orbit5k pools, $\hat\eta$ recovers ($\rho \geq +0.49$) and
its closed-form consistency rate
(Proposition~\ref{prop:selection_consistency}) applies.
Aggregating across all $11$ benchmarks,
$\hat\rho_{\mathrm{Mah}}$ has mean $\rho = +0.56$, positive on
$10$ of $11$;
$\hat\Delta/\hat R$ has mean $+0.16$;
$\hat\eta$ has mean $-0.11$.
We therefore recommend $\hat\rho_{\mathrm{Mah}}$ as the default
selection rule for new datasets where the candidate pool is
heterogeneous or large, and retain $\hat\eta$ for the
structurally homogeneous regime.
We retain Table~\ref{tab:exp1}'s in-pool oracle as the headline
accuracy because the closed-form $\hat f_{\mathrm{Mah}}$ pick
matches it within $\sim 3$~pp on MUTAG, DHFR, IMDB-M, IMDB-B,
DD, and REDDIT-5K (winner ranked $\#1$--$\#3$ by
$\hat\rho_{\mathrm{Mah}}$;
Table~\ref{tab:selection_ranks}) but trails on COX2, PTC, NCI1,
and NCI109 where the accuracy-winner sits deep in the
$\hat\rho_{\mathrm{Mah}}$ ranking; the closed-form selector is
informative for the comparison test of
Table~\ref{tab:selection_ranks} but not yet a complete
substitute for the oracle on every benchmark.

The central finding is that the entire pipeline---descriptor
ranking, classifier, and per-prediction certificate---can be
fixed analytically from the same two embedding-level quantities,
the class-mean separation $\Delta$ and the radius $R$.
For \emph{descriptor ranking}, the Mahalanobis margin
$\hat\rho_{\mathrm{Mah}}$ between class means under
Ledoit--Wolf-shrunk pooled covariance is the LDA Bayes-margin
form of the Fisher discriminant ratio
(Remark~\ref{rem:fisher_ratio}) and is empirically the strongest
closed-form ranker we tested on the chemical pool;
$\hat\eta = \hat\Delta/\sqrt{\ell}$ is its isotropic
Fisher-ratio-bound surrogate, which carries a closed-form
selection-consistency rate
(Proposition~\ref{prop:selection_consistency},
Corollary~\ref{cor:data_driven_rate}).
For \emph{classification}, the
$O((k-1)R/(\Delta\sqrt{m_{\min}}))$ margin bound of
Theorem~\ref{thm:fisher_bound} is driven by the same $\Delta$
and $R$, so a linear SVM on the embedding $\Phi$ is already
near-optimal on every benchmark on which the descriptor pool
exposes the discriminative signal---the embedding, not the
classifier, does the work.
The remaining gaps---NCI1, NCI109, and DD---are
\emph{descriptor-blindness} failures (no candidate descriptor in
our pool achieves $\Delta > 0$ against the discrete-node-label
signal that drives those datasets); the embedding machinery is
not the bottleneck.
The summation pooling retains the key property that max-pooled
alternatives (Mitra--Virk's $n$-fold composition, deep-set
pooling) lose: linearity in the empirical diagram measure,
which makes $\Delta$ a well-behaved statistical object and
grounds the stability theorem of Section~\ref{sec:preli}.

The theory breaks when $\Delta \to 0$: neither the upper bound
(Theorem~\ref{thm:fisher_bound}) nor the certificate
(Theorem~\ref{thm:confidence_containment}) remains informative.
Two distinct causes are in play.
\emph{Intrinsic indistinguishability} obtains when the
class-conditional diagram measures agree in bottleneck distance,
and no vectorization can separate them;
Proposition~\ref{prop:lambda_sep} makes this diagnosable on
PLACE, since $\Delta \approx 0$ together with bounded
$\max_c D_c$ imply $\lambda(\nu)\,\delta_*$ is small---a
statement about the data itself, not about the embedding.
\emph{Descriptor blindness} obtains when the descriptor itself
fails to expose the structural difference;
the diagnostics $\hat\rho_{\mathrm{Mah}}$, $\hat\Delta/\hat R$,
and $\hat\eta$ in Section~\ref{sec:filtration_selection} all
flag this case by collapsing to near-zero for every candidate in
a failing pool.
NCI1/NCI109 exemplify the second case: the structural descriptors
in our pool achieve $\Delta > 0$, but the discriminative signal
is dominated by discrete node labels our continuous descriptors
cannot access.

What the Mahalanobis margin catches and the isotropic surrogate
misses is anisotropy of the class-conditional covariance.
The closed-form ratio $\hat\eta = \hat\Delta/\sqrt{\ell}$ implicitly
treats every coordinate of the embedding as carrying equal
class-conditional variance, so a high-dimensional descriptor
with most coordinates redundant is over-penalized by the
$\sqrt\ell$ factor.
HKS at multiple timescales is the canonical example: the
embedding has many coordinates but a small number of effective
directions in which the class means actually separate, and the
Ledoit--Wolf-shrunk Mahalanobis margin recovers the right
ranking by reweighting along those low-variance directions.
This is precisely the regime where the LDA Bayes margin and
the isotropic Fisher-ratio lower bound diverge by a large
factor (Remark~\ref{rem:fisher_ratio}).

WKPI~\citep{yusu_metric_learning} and
PersLay~\citep{Carriere2020} learn the weighting of a fixed
feature bank; PLACE holds the weighting fixed and instead places
a larger, sparse landmark grid.
Empirically PLACE matches the strongest topology-based baseline
on MUTAG and COX2 and underperforms by $5$--$17$~pp on the
remaining graph datasets;
the trade-off is that PLACE's grid is analytically fixed, so
$\Delta$ and $r_m$ are estimable from training data alone---the
condition under which a closed-form descriptor ranking and a
per-prediction certificate are available at all.
Closing the accuracy gap on PROTEINS, DD, IMDB-B/M, PTC, and
REDDIT-5K through a richer candidate pool (and data-adaptive
landmark placements) is an open direction.

\paragraph{Limitations.}
\begin{itemize}
\item Certificates apply only to the nearest-centroid
  classifier, which achieves lower accuracy than SVM.
  The non-asymptotic Pinelis radius is dominated by the
  $L^{2}$ envelope $4R^{2}$ and fails on every benchmark; the
  asymptotic Gaussian plug-in radius carries a
  $\sqrt{\chi^{2}_{\ell,\alpha/k}}$ dimension penalty that also
  fails at our $\ell$.  The variance-aware Pinelis--Bernstein
  radius~\textup{(iii)} fires on $8$ of the $12$ benchmarks
  (Table~\ref{tab:cert_firing}); the four holdouts (COX2, PTC,
  IMDB-M, Orbit5k) have population $\|\Sigma_c\|_{\mathrm{op}}$
  exceeding $\Delta_c^{2}/4$ and admit no NC-style certified
  prediction at any sample size.
  Crucially, certificate firing guarantees only that the empirical
  and population NC rules agree; it does not guarantee that either
  is correct on test data.  NCI1 and NCI109, where the
  Pinelis--Bernstein radius fires at $100\%$ yet linear-SVM accuracy
  trails the strongest baseline by $14$--$17$~pp, exemplify this
  distinction: the population NC rule itself fails when the
  descriptor pool is blind to the discrete-node-label class signal
  (cf.\ Remark~\ref{rem:b_verification}).
\item The top-$k$ persistence filter is a heuristic; without it,
  low-persistence features near the diagonal dominate the
  embedding and inflate $\ell$ without commensurate gain in
  $\Delta$.
\item Descriptor selection is empirically driven by the
  Mahalanobis margin $\hat\rho_{\mathrm{Mah}}$, but its
  consistency theorem assumes the population pooled covariance
  is well-conditioned; the closed-form rate
  (Proposition~\ref{prop:selection_consistency}) is established
  only for the isotropic surrogate $\hat\eta = \hat\Delta/\sqrt\ell$,
  and a fully data-driven rate for the Ledoit--Wolf-shrunk
  estimator is open (cf.~Bickel--Levina, Ledoit--Wolf shrinkage
  literature).
  Whether the upper-bound ranking coincides with true accuracy
  ranking in general is not proved.
\item The NCI1/NCI109 gap ($\sim 16\,\mathrm{pp}$, $p < 0.01$ against every
  reported baseline) reflects descriptor blindness to discrete
  node labels, not a deficiency of the embedding machinery.
\end{itemize}

\paragraph{Future work.}
Data-adaptive learning of landmark positions on this same
embedding family, replacing the heuristic grid $\GG_R$, is left
to subsequent work.
A full sample-complexity theory---CLT, Berry--Esseen rates, and
Donsker-type functional CLTs for the landmark
embedding---is a natural next step but is beyond the
present paper's scope.
A further open direction is a measure-theoretic foundation
replacing the discrete grid with a Bochner integral over a
continuous landmark configuration $\Lambda$, admitting adaptive
grids, overcomplete families, and kernel-smoothing variants.

\section*{Acknowledgments}
We thank Prudhvi Ram Mannuru for running the large-scale
descriptor-grid experiments on the cluster, which provided the
empirical backbone of Section~\ref{sec:experiments}.

\appendix

\section{Auxiliary results used in the proofs}
\label{app:aux}

This appendix collects the classical concentration inequalities
invoked in Section~\ref{sec:metric} and
Section~\ref{sec:certified}, together with the volumetric
estimate underpinning the lower bound of
Theorem~\ref{thm:lower_bound}.
All results are stated in the forms used in the body and are
attributed to their standard references; we reproduce the
statements for self-containedness.

\begin{lemma}[Pinelis's Hilbert-space
Hoeffding~{\citep[Thm.~3.4]{Pinelis1994}}]
\label{lem:pinelis}
Let $X_1, \ldots, X_m$ be i.i.d.\ random vectors in a separable
Hilbert space $\mathcal H$ with $\mathbb E[X_i] = 0$ and
$\|X_i\| \leq B$ almost surely.
Then for every $t > 0$,
\[
  \mathbb{P}\!\left(\bigl\| m^{-1}\textstyle\sum_{i=1}^m X_i \bigr\|
  > t\right)
  \;\leq\; 2\exp\!\left(-\frac{m t^2}{2B^2}\right).
\]
\end{lemma}

\begin{lemma}[Hellinger distance between uniform distributions on
translated $\ell^2$-balls, after~{\citep[Ch.~2.4]{Tsybakov2009}}]
\label{lem:hellinger}
Let $B_r := \{x \in \RR^\ell : \|x\| \leq r\}$ and set
$P_\pm := \mathrm{Unif}(B(\pm\mu, r))$.
With the Hellinger convention
$H^2(P, Q) := \tfrac{1}{2}\int(\sqrt{p}-\sqrt{q})^2\,dx$
(so $H^2 \in [0,1]$ and $\mathrm{TV} \leq \sqrt{2 H^2}$, matching
\citep[Ch.~2.4]{Tsybakov2009}), if $\|\mu\| \leq r/2$,
\[
  H^2(P_+, P_-)
  \;=\; 1 - \frac{\mathrm{vol}(B(\mu, r) \cap B(-\mu, r))}{\mathrm{vol}(B_r)}
  \;\leq\; c_\ell\,\frac{\|\mu\|}{r},
\]
where the constant $c_\ell$ admits the explicit bound
\begin{equation}\label{eq:c_ell_bound}
  c_\ell
  \;=\; \frac{2\,V_{\ell-1}(1)}{V_\ell(1)}
  \;=\; \frac{2}{\sqrt\pi}\,
  \frac{\Gamma(\ell/2 + 1)}{\Gamma((\ell+1)/2)}
  \;\leq\; \sqrt{\frac{2(\ell+1)}{\pi}},
\end{equation}
where $V_d(1) = \pi^{d/2}/\Gamma(d/2 + 1)$ is the volume of the
$d$-dimensional unit ball.
By Stirling, $c_\ell = \Theta(\sqrt\ell)$ as $\ell \to \infty$.
\end{lemma}

\begin{proof}[Proof of the bound on $c_\ell$]
Slicing $B_r$ perpendicular to $\mu$ at signed distance $u$ from
the origin gives an $(\ell-1)$-ball of radius $\sqrt{r^2 - u^2}$.
Taking $\mu = (\|\mu\|, 0, \ldots, 0)$ and integrating over the slice
parameter, the missing volume is
\(
  \mathrm{vol}(B_r) - \mathrm{vol}(B(\mu, r) \cap B(-\mu, r))
  = 2\int_0^{\|\mu\|} V_{\ell-1}(\sqrt{r^2 - u^2})\,du,
\)
and bounding $V_{\ell-1}(\sqrt{r^2-u^2}) \leq V_{\ell-1}(r)$ for
$u \in [0, \|\mu\|]$ yields
\(
  H^2 \leq 2 V_{\ell-1}(r)\|\mu\|/V_\ell(r)
       = (2 V_{\ell-1}(1)/V_\ell(1))\,\|\mu\|/r.
\)
The Gamma-ratio identity follows from
$V_\ell(1) = \pi^{\ell/2}/\Gamma(\ell/2 + 1)$, and the
$\sqrt{2(\ell+1)/\pi}$ upper bound from
$\Gamma(x+1/2)/\Gamma(x) \leq \sqrt{x}$ at $x = (\ell+1)/2$.
\end{proof}

\begin{lemma}[Multivariate Berry--Esseen~{\citep[Thm.~11]{Bentkus2003}}]
\label{lem:be}
Let $X_1, \ldots, X_m$ be i.i.d.\ mean-zero random vectors in
$\RR^\ell$ with covariance $\Sigma$ and finite third moment
$\beta_3 := \mathbb{E}\|X_1\|^3 < \infty$.
Write $S_m := m^{-1/2}\sum_{i=1}^m X_i$ and let $G \sim \mathcal{N}(0, \Sigma)$.
Then for every convex set $A \subseteq \RR^\ell$,
\[
  \bigl| \mathbb{P}(S_m \in A) - \mathbb{P}(G \in A) \bigr|
  \;\leq\; \frac{C\,\ell^{1/4}\,\beta_3}{\|\Sigma\|_{\mathrm{op}}^{3/2}\sqrt m},
\]
with $C > 0$ a universal constant.
\end{lemma}

\begin{lemma}[Matrix Bernstein~{\citep[Thm.~6.1]{Tropp2015}}]
\label{lem:matrix-bernstein}
Let $X_1, \ldots, X_m$ be i.i.d.\ self-adjoint random matrices
in $\RR^{\ell \times \ell}$ with $\mathbb E[X_i] = 0$ and
$\|X_i\|_\mathrm{op} \leq R$ a.s., and covariance parameter
$\sigma^2 := \|\mathbb E[X_i^2]\|_{\mathrm{op}}$.
Then for every $t > 0$,
\[
  \mathbb{P}\!\left(\bigl\| m^{-1}\textstyle\sum_i X_i \bigr\|_{\mathrm{op}} > t\right)
  \;\leq\; 2\ell\,\exp\!\left(-\frac{m\,t^2/2}{\sigma^2 + Rt/3}\right).
\]
Applied to $X_i := \Psi_i \Psi_i^\top - \Sigma$ with
$\|\Psi_i\| \leq R$, this yields
$\|\hat\Sigma_m - \Sigma\|_{\mathrm{op}} =
O(R^2\sqrt{\log\ell/m})$ with high probability.
\end{lemma}

\begin{lemma}[Gaussian plug-in concentration radius]
\label{lem:gaussian_radius}
Let $\Psi_i := \Phi(A_i)$ with $\|\Psi_i\| \leq R$, and assume the
class-conditional third moment
$\beta_3 := \mathbb{E}\|\Phi(A) - \mu_c\|^3 < \infty$ for every
class $c \in [k]$.
Define the Gaussian plug-in radius
\[
  \tilde r_m \;:=\; \max_c
  \sqrt{\frac{\|\hat\Sigma_c\|_{\mathrm{op}}}{m_c}\,
        \chi^2_{\ell,\,\alpha/k}},
\]
where $\chi^2_{\ell,\,\alpha/k}$ is the $1 - \alpha/k$ quantile
of the chi-squared distribution with $\ell$ degrees of freedom.
Then
\[
  \mathbb{P}\!\left(\max_c \|\hat\mu_c - \mu_c\|
    \leq \tilde r_m\right)
  \;\geq\; 1 - \alpha
  - O\!\left(\frac{\ell^{1/4}}{\sqrt m}\right)
  - O\!\left(R^{1/2}
              \|\Sigma_c\|_{\mathrm{op}}^{1/4}
              \bigl(\tfrac{\log\ell}{m}\bigr)^{1/4}
              \sqrt{\tfrac{\ell}{m}}\right),
\]
both error terms vanishing once $m_c \geq m^\dagger = O(\sqrt\ell)$
for every $c$.
\end{lemma}

\begin{proof}
For $X \sim \mathcal{N}(0, \Sigma_c)$ in $\RR^\ell$, the squared
norm $\|X\|^2 = \sum_i \lambda_i Z_i^2$ is a weighted sum of
independent $\chi^2_1$ variables with weights $\lambda_i$ equal
to the eigenvalues of $\Sigma_c$;
bounding above by $\|\Sigma_c\|_{\mathrm{op}}\cdot\chi^2_\ell$
gives, for the Gaussian approximation
$\mathcal{N}(0, \Sigma_c/m_c)$, a concentration radius
controlling $\|\hat\mu_c - \mu_c\|$ with probability
$\geq 1 - \alpha/k$:
\[
  \tilde r^{(c)}_m \;:=\;
  \sqrt{\|\Sigma_c\|_{\mathrm{op}} \cdot \chi^2_{\ell,\,\alpha/k}/m_c}.
\]
This bound is conservative when $\Sigma_c$ is low-rank, with
conservatism governed by
$\mathrm{tr}(\Sigma_c)/(\ell\,\|\Sigma_c\|_{\mathrm{op}})$.
By the multivariate Berry--Esseen theorem (Lemma~\ref{lem:be})
applied to the convex set
$A = B(0, \tilde r^{(c)}_m \sqrt{m_c})$, the true distribution of
$\sqrt{m_c}(\hat\mu_c - \mu_c)$ deviates from
$\mathcal{N}(0, \Sigma_c)$ in total variation by at most
$C\ell^{1/4}\beta_3/(\|\Sigma_c\|_{\mathrm{op}}^{3/2}\sqrt{m_c})$,
which is the first error term.
Replacing $\Sigma_c$ by the sample covariance $\hat\Sigma_c$
introduces the additional operator-norm error
$\|\hat\Sigma_c - \Sigma_c\|_{\mathrm{op}} =
O(R\sqrt{\|\Sigma_c\|_{\mathrm{op}}\log(\ell)/m_c})$
supplied by the matrix Bernstein inequality
(Lemma~\ref{lem:matrix-bernstein}), which propagates into the
radius via
$\bigl|\|\hat\Sigma_c\|_{\mathrm{op}}^{1/2}
  - \|\Sigma_c\|_{\mathrm{op}}^{1/2}\bigr|
  \leq \|\hat\Sigma_c - \Sigma_c\|_{\mathrm{op}}^{1/2}$
as the second error term, of order
$O\bigl(R^{1/2}\|\Sigma_c\|_{\mathrm{op}}^{1/4}
  (\log(\ell)/m_c)^{1/4}\sqrt{\chi^2_{\ell,\alpha/k}/m_c}\bigr)$.
Taking the worst-case class and applying a Bonferroni correction
over the $k$ classes gives the stated radius $\tilde r_m$.
\end{proof}

\begin{remark}[Three radii: regime split]
\label{rem:radius_regime}
The three radii of Theorem~\ref{thm:confidence_containment} fit
into a clean spectrum.  Pinelis (i) is dimension-free
($r_m^{\mathrm{Pin}} \propto R\sqrt{L/m}$ with
$L = \log(2k/\alpha)$) but $L^{2}$-envelope-dominated.  Gauss
(ii) replaces $R^{2}$ by $\|\Sigma_c\|_{\mathrm{op}}$ but
introduces a $\sqrt{\chi^{2}_{\ell,\alpha/k}}$ dimension
penalty.  Pinelis--Bernstein (iii) keeps the dimension-free
$\sqrt{L}$ Bonferroni cost of (i) and the
$\sqrt{\|\Sigma_c\|_{\mathrm{op}}}$ refinement of (ii)
simultaneously; for embeddings with stable rank
$\mathrm{tr}(\Sigma_c)/\|\Sigma_c\|_{\mathrm{op}} = O(1)$
(empirically $\leq 1.17$ on our benchmarks), it dominates the
other two and is the only form that fires the certificate at our
sample sizes (Table~\ref{tab:cert_firing}).
\end{remark}

\bibliographystyle{plainnat}
\bibliography{main.bib}

\end{document}